\documentclass{article}
\usepackage[accepted]{icml2022}

\usepackage{microtype}
\usepackage{graphicx}
\usepackage{booktabs} %

\usepackage{microtype}
\usepackage{graphicx}
\usepackage[space]{grffile}
\usepackage{booktabs} %
\usepackage{stmaryrd}

\usepackage[colorlinks=true,citecolor=blue,urlcolor=red,backref=page]{hyperref}
 
\usepackage{amsmath}
\usepackage{amssymb}
\usepackage{amsthm}
\usepackage{amsfonts}
\usepackage{mathtools}
\usepackage{nicefrac}
\usepackage{lipsum}
\usepackage{ulem}
\usepackage[labelfont=bf,font=small]{caption}
\usepackage{subcaption}
\usepackage{wrapfig}
\usepackage{xcolor}
\usepackage{comment}
\usepackage{multibib}

\usepackage[export]{adjustbox}

\graphicspath{ {images/} }

\def\argmin{\mathop{\rm arg\,min}}

\def\minop{\mathop{\rm min}\limits}
\def\maxop{\mathop{\rm max}\limits}

\def\R{\mathbb{R}}

\newcommand{\norm}[1]{\left\|#1\right\|}

\newcommand{\inner}[1]{\left\langle#1\right\rangle}

\DeclareMathOperator{\E}{\mathbb{E}}
\let\l\relax
\DeclareMathOperator{\l}{\mathbf{\ell}}

\newcommand{\wv}{\boldsymbol{w}}
\newcommand{\xv}{\boldsymbol{x}}

\newcommand{\deltav}{\boldsymbol{\delta}}

\newcommand{\Scal}{\mathcal{S}}

\usepackage{tcolorbox}

\usepackage{amsmath,amsfonts,bm}

\def\eqref#1{equation~\ref{#1}}

\def\1{\bm{1}}

\DeclareMathAlphabet{\mathsfit}{\encodingdefault}{\sfdefault}{m}{sl}
\SetMathAlphabet{\mathsfit}{bold}{\encodingdefault}{\sfdefault}{bx}{n}

\newcommand{\myparagraph}{\textbf}

\newtheorem{theorem}{Theorem}
\newtheorem*{theorem*}{Theorem}
\newtheorem{lemma}[theorem]{Lemma}

\renewcommand{\wv}{w}
\renewcommand{\xv}{x}

\renewcommand{\deltav}{\delta}

\usepackage{enumitem}
\setitemize{topsep=3.5pt,parsep=0pt,partopsep=0pt,leftmargin=25pt}
\usepackage{amssymb}

\definecolor{ForestGreen}{RGB}{34,139,34}
\newcommand{\hl}[1]{#1}


\DeclareMathOperator\arcsinh{arcsinh}

\icmltitlerunning{Towards Understanding Sharpness-Aware Minimization}

\begin{document}

\twocolumn[
\icmltitle{Towards Understanding Sharpness-Aware Minimization}

\begin{icmlauthorlist}
\icmlauthor{Maksym Andriushchenko}{epfl}
\icmlauthor{Nicolas Flammarion}{epfl}
\end{icmlauthorlist}

\icmlaffiliation{epfl}{EPFL, Switzerland}

\icmlcorrespondingauthor{Maksym Andriushchenko}{maksym.andriushchenko@epfl.ch}

\icmlkeywords{Machine Learning, ICML}

\vskip 0.3in
]

\printAffiliationsAndNotice{}

\begin{abstract}
    Sharpness-Aware Minimization (SAM) is a recent training method that relies on worst-case weight perturbations which significantly improves generalization in various settings. We argue that the existing justifications for the success of SAM which are based on a PAC-Bayes generalization bound and the idea of convergence to flat minima are incomplete. Moreover, there are no explanations for the success of using $m$-sharpness in SAM which has been shown as essential for generalization. To better understand this aspect of SAM, we theoretically analyze its implicit bias for diagonal linear networks. We prove that SAM always chooses a solution that enjoys better generalization properties than standard gradient descent for a certain class of problems, and this effect is amplified by using $m$-sharpness. We further study the properties of the implicit bias on non-linear networks empirically, where we show that fine-tuning a standard model with SAM can lead to significant generalization improvements. Finally, we provide convergence results of SAM for non-convex objectives when used with stochastic gradients. We illustrate these results empirically for deep networks and discuss their relation to the generalization behavior of SAM. The code of our experiments is available at {\small \url{https://github.com/tml-epfl/understanding-sam}}.
\end{abstract}

\section{Introduction}
\label{sec:intro}
Understanding generalization of overparametrized deep neural networks is a central topic of machine learning. 
Training objective has many global optima where the training data are perfectly fitted \citep{zhang2016understanding}, but different global optima lead to dramatically different generalization performance \citep{liu2019bad}.
However, it has been observed that stochastic gradient descent (SGD) tends to converge to well-generalizing solutions, even \textit{without} any explicit regularization methods \citep{zhang2016understanding}.
This suggests that the leading role is played by the \textit{implicit} bias of the optimization algorithms used~\citep{neyshabur2014search}: when the training objective is minimized using a particular algorithm and initialization method, it converges to a specific solution with favorable generalization properties. 
However, even though SGD has a very beneficial implicit bias, significant overfitting  can still occur, particularly in the presence of label noise \citep{nakkiran2019deep} and adversarial perturbations \citep{rice2020overfitting}. %

Recently it has been observed that the \textit{sharpness} of the training loss, i.e., how quickly it changes in some neighborhood around the parameters of the model, correlates well with the generalization error~\citep{keskar2016large, jiang2019fantastic}, and generalization bounds related to the sharpness have been derived \citep{dziugaite2018entropy}. The idea of minimizing the sharpness to improve generalization has motivated recent works of \citet{foret2021sharpnessaware}, \citet{zheng2020regularizing}, and \citet{wu2020adversarial} which propose to use worst-case perturbations of the weights on every iteration of training in order to improve generalization. 
We refer to this method as \textit{Sharpness-Aware Minimization} (SAM) and focus mainly on the version proposed in \citet{foret2021sharpnessaware} that performs only one step of gradient ascent to approximately solve the weight perturbation problem before updating the weights.

Despite the fact that SAM significantly improves generalization in various settings, the existing justifications based on the generalization bounds provided by \citet{foret2021sharpnessaware} and \citet{wu2020adversarial} do not seem conclusive. 
The main reason is that their generalization bounds do not distinguish the robustness to \textit{worst-case} weight perturbation from \textit{average-case} robustness to Gaussian noise. However the latter does not sufficiently improve generalization as both \citet{foret2021sharpnessaware} and \citet{wu2020adversarial} report. 
Furthermore, their analysis does not distinguish whether the worst-case weight perturbation is computed based on some or on all training examples. As we will discuss, this feature has a crucial impact on generalization. 

In our paper, we aim to further investigate the reasons for SAM's success and make the following contributions: %
\begin{itemize}
    \item We discuss why the current understanding of the success of SAM which is based on a PAC-Bayesian generalization bound and on convergence to a flatter minimum is incomplete.
    \item We test hypotheses regarding why maximization in SAM taken over \textit{fewer} training points can lead to better generalization and conclude that the benefit is likely to come from the better objective.
    \item We study the implicit bias of this objective theoretically for diagonal linear networks. For non-linear networks, we study the implicit bias empirically and relate it to the theoretical model.
    \item We prove convergence of SAM for non-convex objectives in the stochastic setting. We check convergence empirically for deep networks and relate it to the generalization behavior of SAM.
\end{itemize}

\section{Background on SAM}
\label{sec:background_on_sam}

\myparagraph{Related work.}
Here we discuss relevant works on robustness in the \textit{weight space} and its relation to generalization.
Works on weight-space robustness of neural networks date back at least to the 1990s \citep{murray1993synaptic, hochreiter1995simplifying}.
Random perturbations of the weights are used extensively in deep learning \citep{jim1996analysis, graves2013speech}, and most prominently in approaches such as dropout \citep{srivastava2014dropout}. %
Many practitioners have observed that using SGD with larger batches for training leads to worse generalization \citep{lecun2012efficient}, and \citet{keskar2016large} have shown that this degradation of performance is \textit{correlated} with the sharpness of the found parameters. 
This observation has motivated many further works which focus on closing the generalization gap between small-batch and large-batch SGD~\citep{wen2018smoothout, haruki2019gradient, lin2020extrapolation}.
More recently, \citet{jiang2019fantastic} have shown a strong correlation between the sharpness and the generalization error on a large set of models under a variety of different settings hyperparameters, beyond the batch size.
This has motivated the idea of minimizing the sharpness during training to improve standard generalization, leading to Sharpness-Aware Minimization (SAM)~\citep{foret2021sharpnessaware}.
SAM modifies SGD such that on every iteration of training, the gradient is taken not at the current iterate but rather at a worst-case point in its vicinity. 
\hl{\citet{zheng2020regularizing} concurrently propose a similar weight perturbation method which also successfully improves standard generalization on multiple deep learning benchmarks.}
\citet{wu2020adversarial} have also proposed an almost identical algorithm with the same motivation, but with
the focus on improving robust generalization of adversarial training. 
On the theoretical side, \citet{mulayoff2020unique} study the sharpness properties of minima of deep linear network, and \citet{neu2021information, wang2022on} study generalization bounds based on average-case sharpness and quantities related to the optimization trajectory of SGD.

\myparagraph{Sharpness.}
Let $\Scal_{train} = \{x_i, y_i\}_{i=1}^n$ be the training data and $\l_i(\wv)$ be the loss of a classifier parametrized by weights $w \in \R^{|w|}$ and evaluated at point $(x_i, y_i)$. Then the \textit{sharpness} on a set of points $\Scal \subseteq \Scal_{train}$ is defined as:
\begin{align} 
    s(\wv, \Scal) \triangleq \maxop_{\norm{\deltav}_2 \leq \rho} \frac{1}{|\Scal|}\sum_{i: (x_i, y_i) \in \Scal} \l_i(\wv + \deltav) - \l_i(\wv).
    \label{eq:sharpness}
\end{align}
In most of the past literature, sharpness is defined for $\Scal = \Scal_{train}$ \citep{keskar2016large, neyshabur2017exploring, jiang2019fantastic}. However, \citet{foret2021sharpnessaware} recently introduced the notion of \textit{$m$-sharpness} which is the average of the sharpness computed over all the batches $\Scal$ of size $m$ from the training set $\Scal_{train}$.

Lower sharpness is correlated with lower test error \citep{keskar2016large}, however, the correlation is not always perfect \citep{neyshabur2017exploring, jiang2019fantastic}. Moreover, the sharpness definition itself can be problematic since rescaling of incoming and outcoming weights of a node that leads to the same function can lead to very different sharpness values \citep{dinh2017sharp}.
\citet{kwon2021asam} suggest a sharpness definition that fixes this rescaling problem but other problems still exist such as the sensitivity of classification losses to the scale of the parameters \citep{neyshabur2017exploring}.

\begin{figure*}[t!]
    \centering
    \begin{subfigure}[t]{.35\textwidth}
        \caption{\hspace{8mm}\textbf{ResNet-18 on CIFAR-10}}
        \vspace{-2mm}
        \includegraphics[width=1.0\columnwidth]{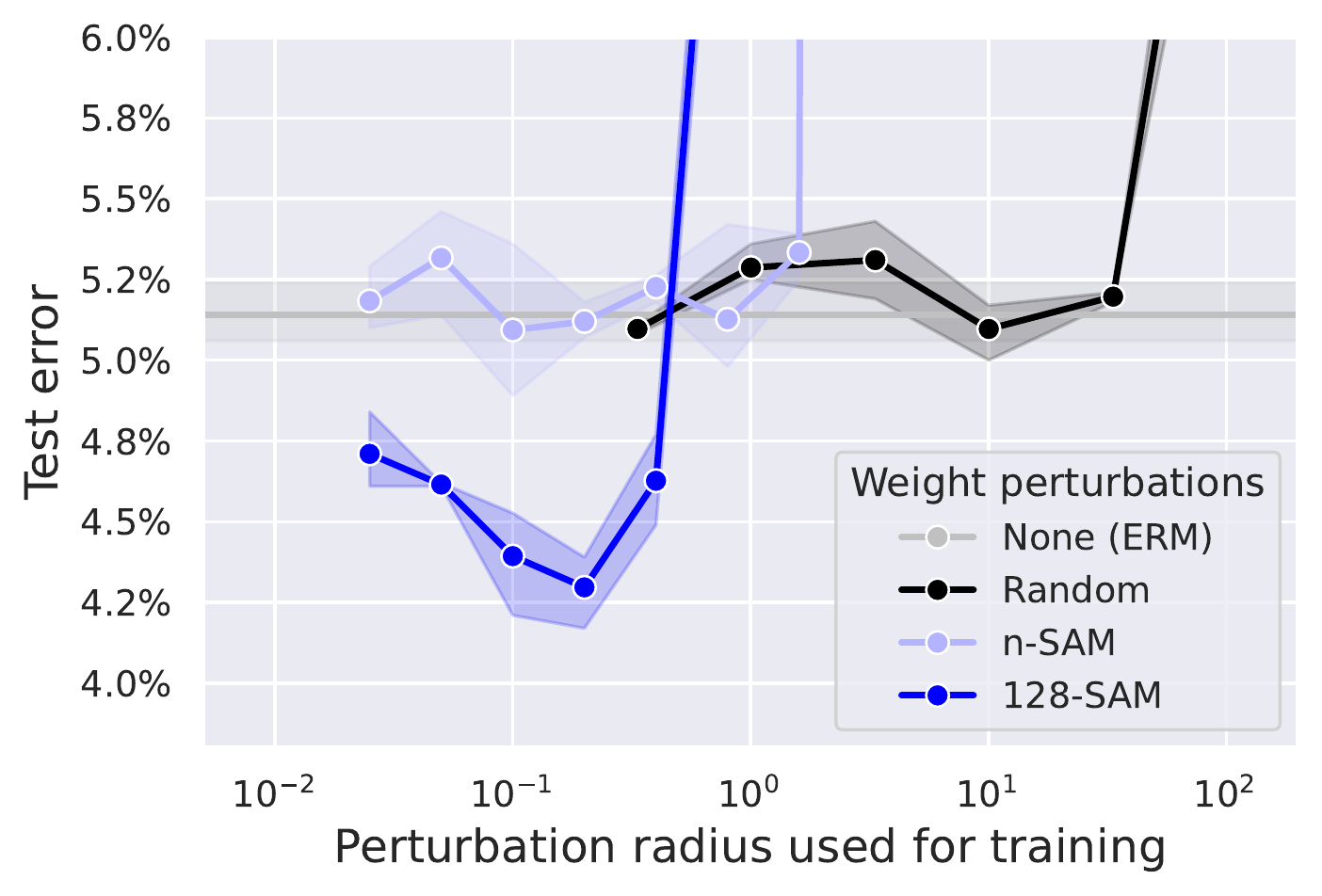}
    \end{subfigure}
    \quad \quad
    \begin{subfigure}[t]{.35\textwidth}
        \caption{\hspace{9mm}\textbf{ResNet-34 on CIFAR-100}}
        \vspace{-2mm}
        \includegraphics[width=1.0\columnwidth]{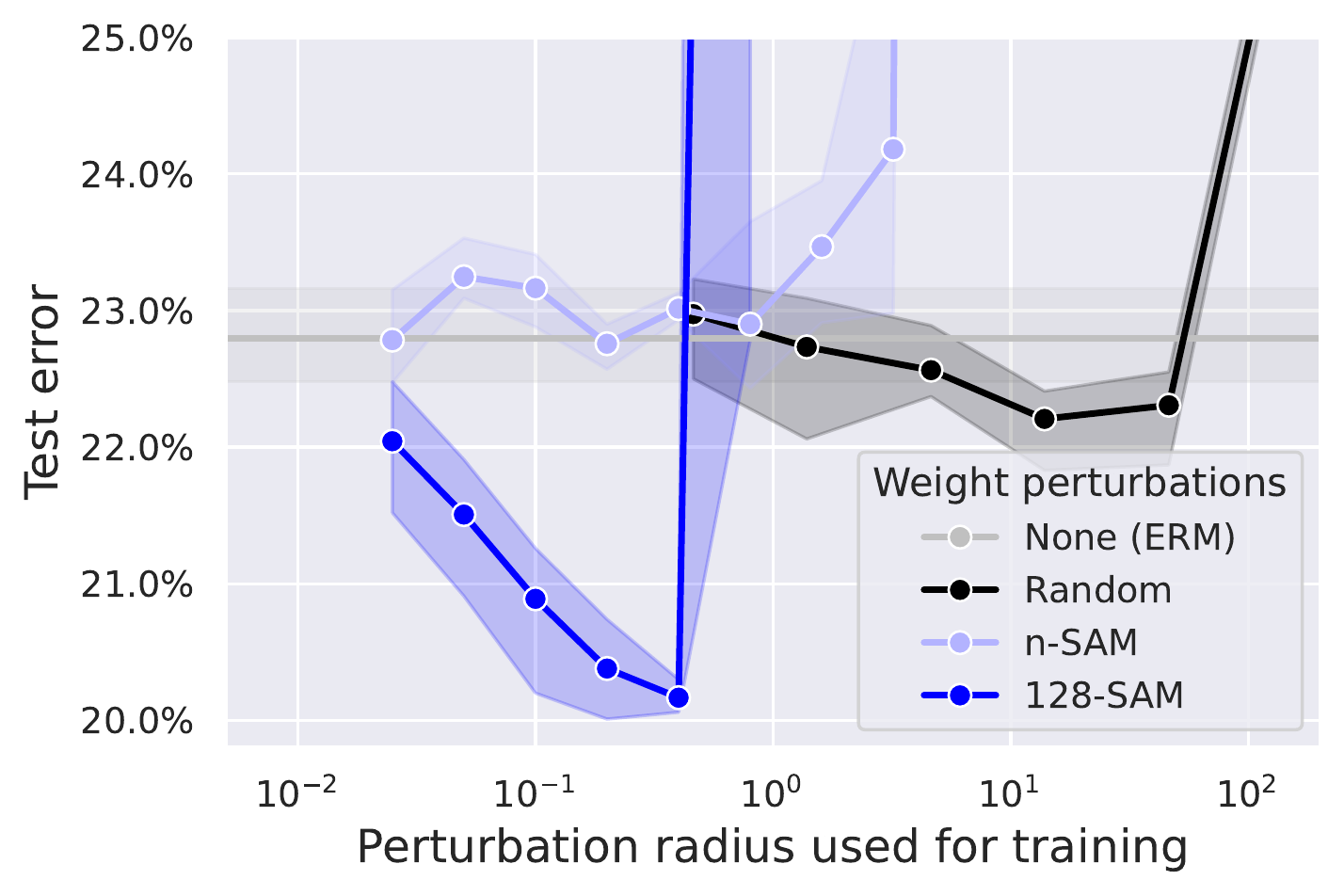}
    \end{subfigure}
    \vspace{-1.5mm}
    \caption{Comparison of different weight perturbation methods: no perturbations (ERM), random perturbations prior to taking the gradient on each iteration, $n$-SAM, and $128$-SAM (see Sec.~\ref{sec:background_on_sam} for the notation). All models are trained with standard data augmentation and small batch sizes ($128$). We observe that among these methods only $m$-SAM with a low $m$ (i.e., $128$-SAM) substantially improves generalization.
    }
    \label{fig:gen_bound_is_not_predictive}
\end{figure*}
\myparagraph{Sharpness-aware minimization.}
\citet{foret2021sharpnessaware} theoretically base the SAM algorithm on the following objective: %
\begin{align}
    \text{\textbf{$\bm n$-SAM}:} \quad \minop_{\wv \in \R^{|w|}} \maxop_{\norm{\deltav}_2 \leq \rho} \sum_{i=1}^{n} \l_i(\wv + \deltav),
    \label{eq:n_sam}
\end{align} 
which we denote as \textbf{$\bm n$-SAM} since it is based on maximization of the sum of the losses over the $n$ training points.
They justify this objective via a PAC-Bayesian generalization bound, although they show empirically (see Fig.~3 therein) that the following objective leads to better generalization:
\begin{align}
    \text{\textbf{$\bm m$-SAM}:}  \minop_{\wv \in \R^{|w|}} %
    \sum_{\substack{ \Scal\subset \Scal_{train},\\
                             |\Scal|=m}}
                               \maxop_{\norm{\deltav}_2 \leq \rho} \sum_{i\in \Scal} \l_i(\wv + \deltav),
    \label{eq:m_sam}
\end{align} 
which we denote as \textbf{$\bm m$-SAM} since it is based on maximization of the sum of the losses over batches of $m$ training points and therefore related to the $m$-sharpness.

To make SAM practical, \citet{foret2021sharpnessaware} propose to minimize the $m$-SAM objective with stochastic gradients.
Denoting the batch indices at time $t$ by $I_t$ ($|I_t|=m$), this leads to the following update rule on each iteration of training:
\begin{equation}\label{eq:sambatch}
    w_{t+1} = w_t -\frac{\gamma_t}{|I_t|}\sum_{i\in I_t} \nabla \ell_i \big(  w_t +\frac{\rho_t}{|I_t|} \sum_{j\in I_t} \nabla \ell_j(w_t) \big).
\end{equation}
Importantly, the \textit{same} batch $I_t$ is used for the inner and outer gradient steps. 
We note that $\rho_t$ can optionally include the gradient normalization suggested in \citet{foret2021sharpnessaware}, i.e., $\rho_t := \rho / \| \frac{1}{|I_t|} \sum_{j\in I_t} \nabla \ell_j(w_t)\|_2$. However, we show in Sec.~\ref{sec:convergence_sam} that its usage is not necessary for improving generalization, so we will omit it from our theoretical analysis.

\myparagraph{Importance of low-$m$, worst-case perturbations.}
In order to improve upon ERM, \citet{foret2021sharpnessaware} use SAM with low-$m$ and worst-case perturbations. To clearly illustrate the importance of these two choices, we show the performance of 
the following weight perturbation methods: no perturbations (ERM), random perturbations (prior to taking the gradient on each iteration), $n$-SAM, and $128$-SAM.
We use ResNet-18 on CIFAR-10 and ResNet-34 on CIFAR-100 \citep{krizhevsky2009learning} with standard data augmentation and batch size $128$ and refer to App.~\ref{sec:app_exp_details} for full experimental details, %
including our implementation of $n$-SAM. 
Fig.~\ref{fig:gen_bound_is_not_predictive} clearly suggests that (1) the improvement from random perturbations is marginal, and (2) the only method that substantially improves generalization is low-$m$ SAM (i.e., $128$-SAM). Thus, worst-case perturbations and the use of $m$-sharpness in SAM are essential for the generalization improvement (which depends continuously on $m$ as noted by \citet{foret2021sharpnessaware}, see Fig.~\ref{fig:test_err_sharpness_vs_rho_many_diff_m} in App.~\ref{subsec:app_effect_of_m}).
We also note that using too low $m$ is inefficient in practice since it does not fully utilize the computational accelerators such as GPUs. 
Thus, using higher $m$ values (such as $128$) helps to balance the generalization improvement with the computational efficiency.
Finally, we note that using SAM with large batch sizes without using a smaller $m$ leads to suboptimal generalization (see Fig.~\ref{fig:test_err_sharpness_vs_rho_many_diff_bs} in App.~\ref{subsec:app_effect_of_batch_size}).

\section{Challenging the Existing Understanding of SAM}

In this section, we show the limitations of the current understanding of SAM. %
In particular, we discuss that the generalization bounds on which its only \textit{formal} justification relies on (such as those presented in \citet{foret2021sharpnessaware, wu2020adversarial, kwon2021asam}) cannot explain its success.
Second, we argue that contrary to a common belief, convergence of SAM to flatter minima measured in terms of $m$-sharpness does not always translate to better generalization.

\myparagraph{The existing generalization bound does not explain the success of SAM.} 
The main theoretical justification for SAM comes from the PAC-Bayesian generalization bound presented, e.g., in Theorem~2 of \citet{foret2021sharpnessaware}. However, the bound is derived for \textit{random} perturbations of the parameters, i.e. the leading term of the bound is %
$\E_{\delta \sim \mathcal{N}(0, \sigma)}\sum_{i=1}^n \l_i(w+\delta)$. The extension to \textit{worst-case} perturbations, i.e. %
$\max_{\norm{\delta}_2 \leq \rho} \sum_{i=1}^n \l_i(w+\delta)$, is done post hoc and only makes the bound less tight. Moreover, we can see empirically (Fig.~\ref{fig:gen_bound_is_not_predictive}) that \textit{both} training methods suggested by the derivation of this bound (random perturbations and $n$-SAM) do not substantially improve generalization. 
This generalization bound can be similarly extended to $m$-SAM by upper bounding the leading term via the maximum taken over mini-batches. However, this bound would incorrectly suggest that $128$-SAM should have the worst generalization among all the three weight-perturbation methods while it is the only method that successfully improves generalization. %

We note that coming up with tight generalization bounds %
even for well-established ERM for overparametrized models is an open research question \citep{nagarajan2019uniform}. One could expect, however, that at least the \textit{relative} tightness of the bounds could reflect the correct ranking between the three methods, but it is not the case. Thus, we conclude that the existing generalization bound cannot explain the generalization improvement of low-$m$ SAM. %

\myparagraph{A flatter minimum does not always lead to better generalization.}
One could assume that although the generalization bound that relies on $m$-sharpness is loose, $m$-sharpness can still be an important quantity for generalization. This is suggested by its better correlation with the test error compared to the sharpness computed on the whole training set \citep{foret2021sharpnessaware}. In particular, we could expect that convergence of SAM to better-generalizing minima can be explained by a lower $m$-sharpness of these minima.
To check this hypothesis, we select multiple models trained with group normalization\footnote{We consider networks with group normalization \citep{wu2018group} instead of the more common batch normalization \citep{ioffe2015batch} since we observed a large discrepancy between $m$-sharpness computed with the training-time vs. test-time batch normalization (see the experiment in Fig.~\ref{fig:bn_max_loss} in App.~\ref{subsec:app_sharpness_bn_modes}).} that achieve zero training error and measure their $m$-sharpness for $m=128$ and different perturbation radii $\rho$ in Fig.~\ref{fig:sharpness_erm_sam_bs}. %
We note that the considered networks are not reparametrized in an adversarial way \citep{dinh2017sharp} and they all use the same weight decay parameters which makes them more comparable to each other.
First of all, we observe that \textit{none} of the radii $\rho$ gives the correct ranking between the methods according to their test error, although $m$-sharpness ranks correctly SAM and ERM for the same batch size. %
In particular, we see that the minimum found by SAM with a large batch size (1024) is flatter than the minimum found by ERM with a small batch size (128) although the ERM model leads to a better test error: 6.17\% vs. 6.80\% on CIFAR-10 and 25.06\% vs. 28.31\% on CIFAR-100. This shows that it is easy to find counterexamples where flatter minima generalize worse.
\begin{figure}[t]
    \centering
    \footnotesize
    \begin{subfigure}[t]{.235\textwidth}
        \caption{\hspace{4mm}\textbf{ResNet-18 on CIFAR-10}}
        \vspace{-2mm}
        \includegraphics[width=1.0\columnwidth]{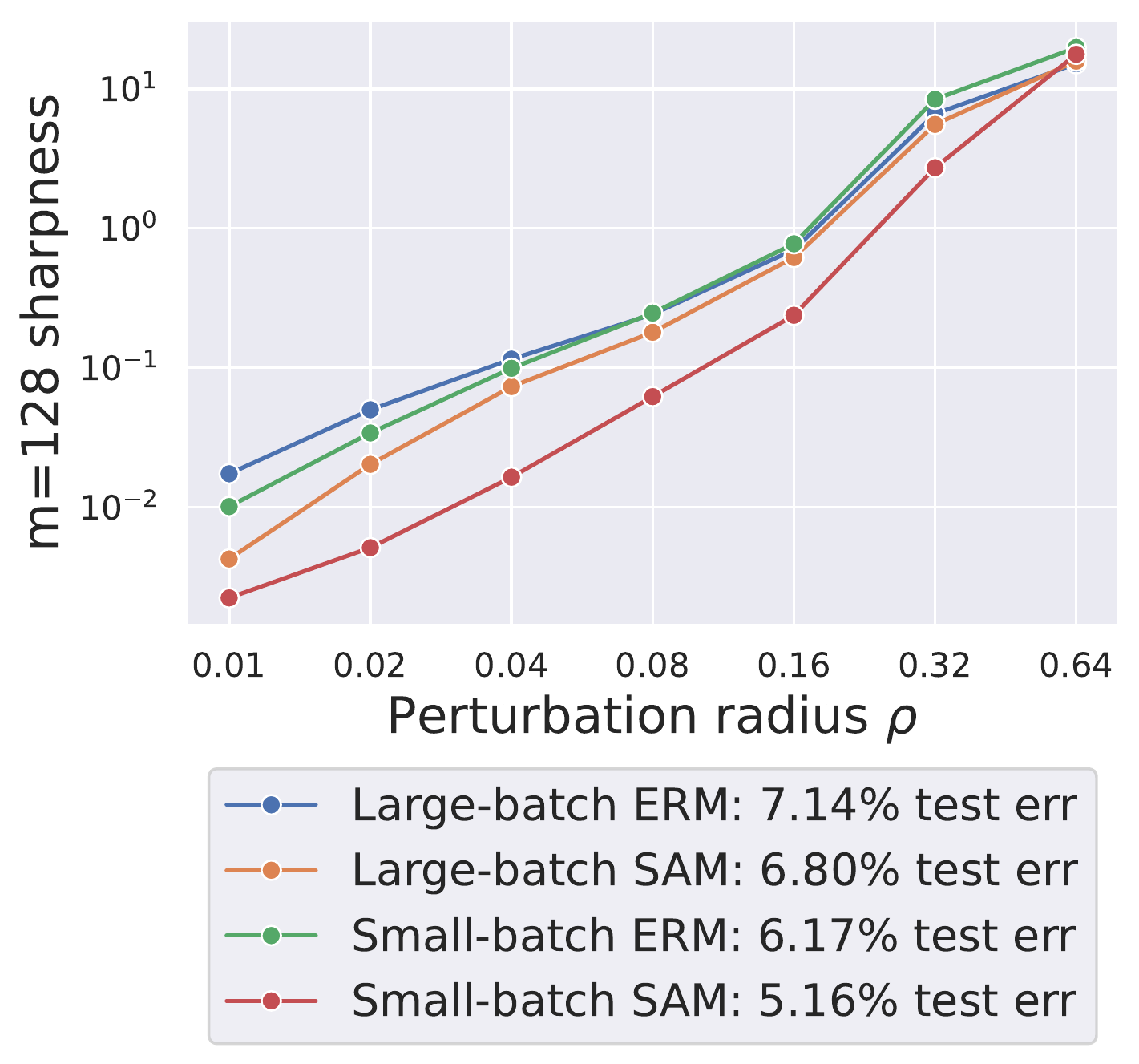}
    \end{subfigure}
    \begin{subfigure}[t]{.235\textwidth}
        \caption{\hspace{5mm}\textbf{ResNet-34 on CIFAR-100}}
        \vspace{-2mm}
        \includegraphics[width=1.0\columnwidth]{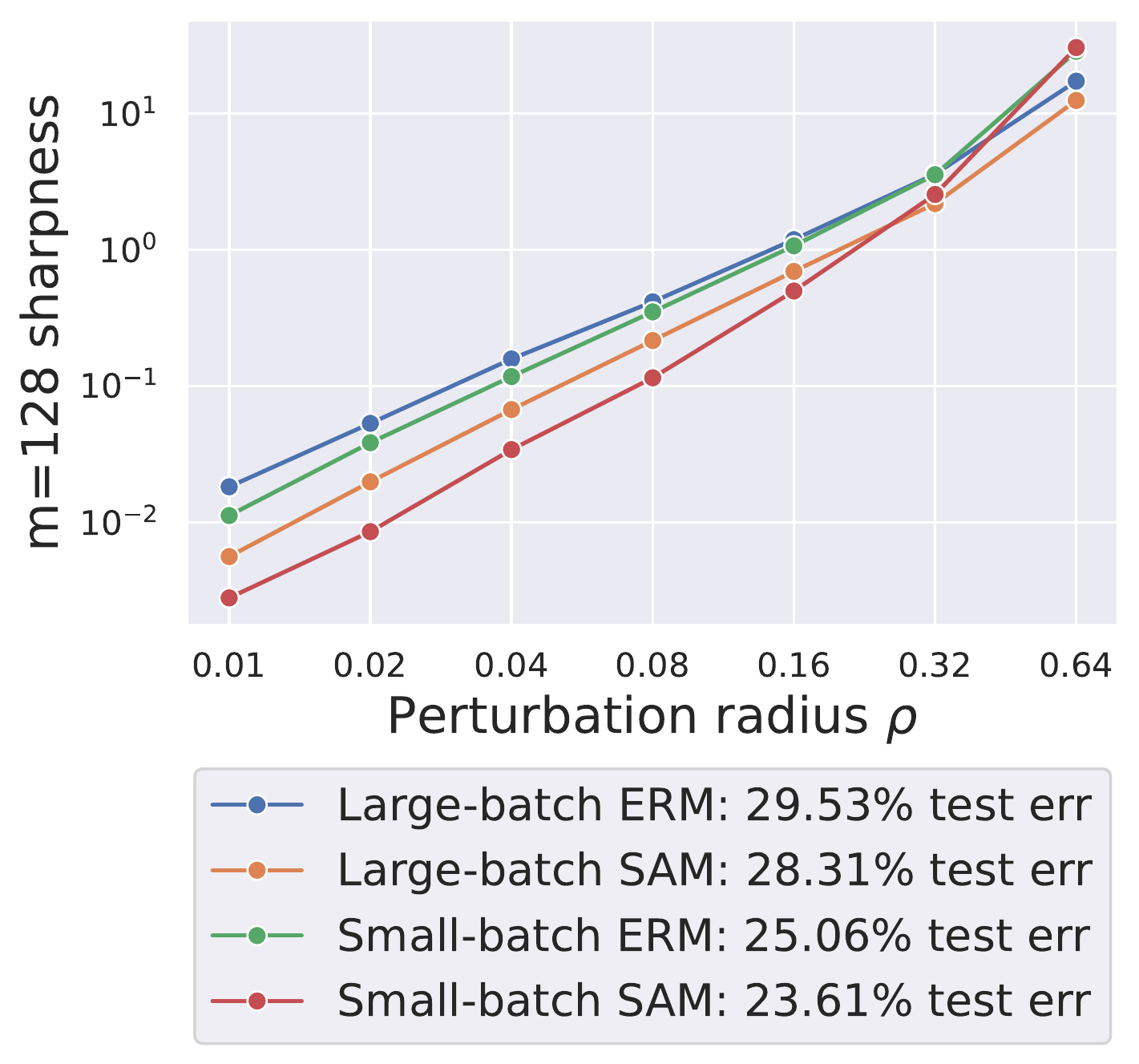}
    \end{subfigure}
    \vspace{-1mm}
    \caption{$m=128$ sharpness computed over different perturbation radii $\rho$ at the minima of ERM and SAM models trained with large (1024) and small batches (128). All models are trained with group normalization and achieve zero training error.}
    \label{fig:sharpness_erm_sam_bs}
\end{figure}

We further note that there are simple examples that illustrate that $m$-sharpness cannot be a universal quantity at distinguishing well-generalizing minima. E.g., consider a linear model $f_x(w) = \inner{w, x}$ and a decreasing margin-based loss $\l$, then the $1$-sharpness has a closed-form solution:
\begin{align} \nonumber
    &\sum_{i=1}^{n} \maxop_{\norm{\deltav}_2 \leq \rho} \l\left(y_i \inner{\wv + \deltav, \xv_i}\right) - \l\left(y_i \inner{\wv, \xv_i}\right) =\\
    &\sum_{i=1}^{n} \l\left(y_i\inner{\wv, \xv_i} - \rho \norm{\xv_i}_2\right) - \l\left(y_i \inner{\wv, \xv_i}\right). \nonumber
\end{align}
The $1$-sharpness is influenced only by the term $-\rho \norm{\xv_i}_2$ which does not depend on a specific $w$. In particular, it implies that all global minimizers $w^*$ of the training loss are \textit{equally sharp} according to the $1$-sharpness which, thus, cannot suggest which global minima generalize better. 

Since ($m$-)sharpness does not always distinguish better- from worse-generalizing minima, the common intuition about sharp vs. flat minima \citep{keskar2016large} can be incomplete. 
This suggests that it is likely that some other quantity is responsible for generalization which can be correlated with ($m$-)sharpness in \textit{some} cases, but not always.
This motivates us to develop a better understanding of the role of $m$ in $m$-SAM, particularly on simpler models which are amenable for a theoretical study.

\section{Understanding the Generalization Benefits of SAM}
In this section, we first check empirically whether the advantage of lower $m$ in $m$-SAM comes from a more accurate solution of the inner maximization problem or from specific properties of batch normalization. We conclude that it is not the case and hypothesize that the advantage comes rather from a better implicit bias of gradient descent induced by $m$-SAM. We characterize this implicit bias for diagonal linear networks showing that SAM can \textit{provably} improve generalization, and the improvement is larger for $1$-SAM than for $n$-SAM. Then we complement the theoretical results with experiments on deep networks showing a few intriguing properties of SAM.

\subsection{Testing Two Natural Hypotheses for Why Low $\bm{m}$ in $\bm{m}$-SAM Could be Beneficial}
\label{subsec:two_natural_hypotheses}
\begin{figure}[t]
    \centering
    \begin{subfigure}[t]{.235\textwidth}
        \caption{\hspace{3mm}\textbf{ResNet-18 on CIFAR-10}}
        \vspace{-2mm}
        \includegraphics[width=1.0\columnwidth]{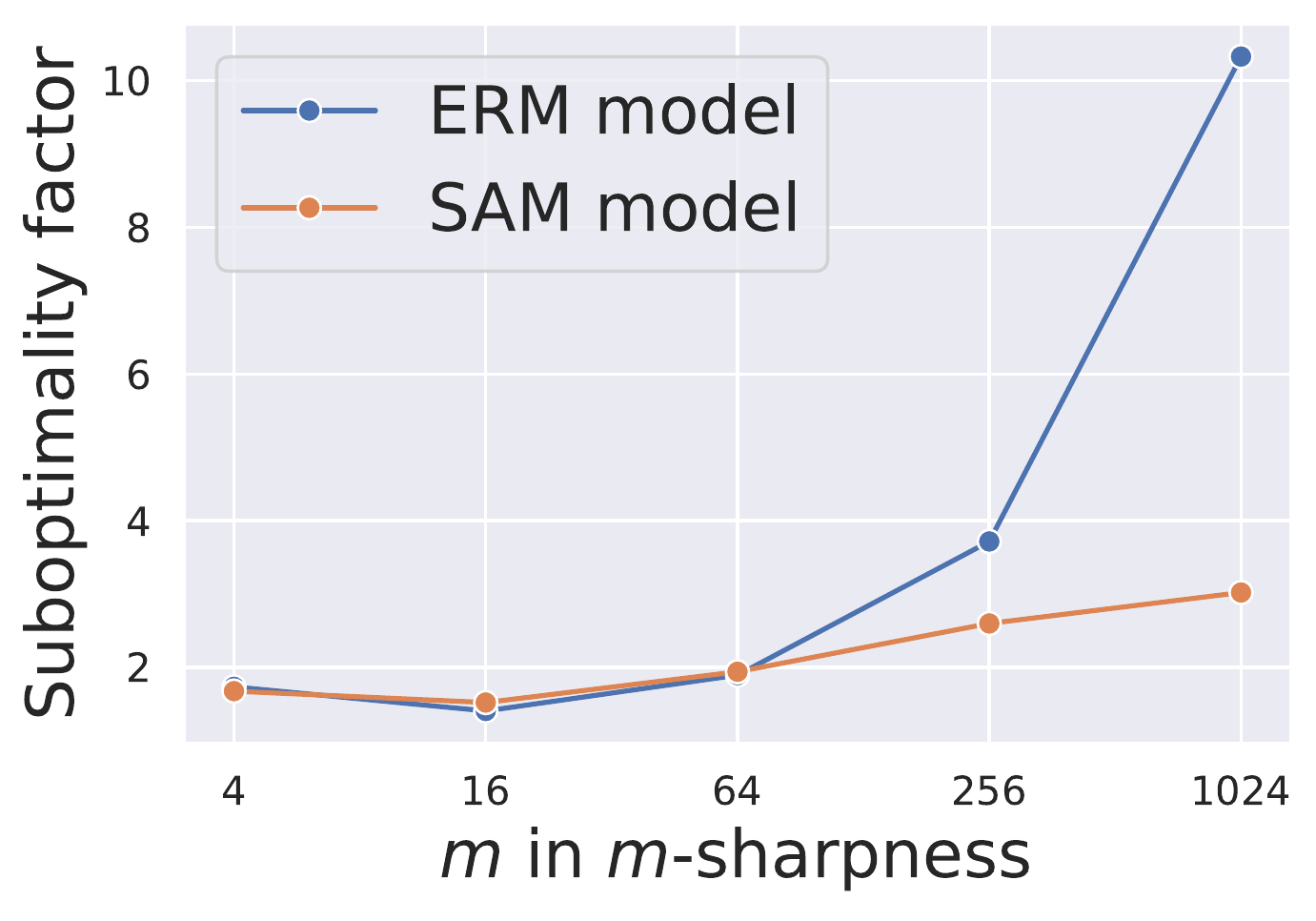}
    \end{subfigure}
    \begin{subfigure}[t]{.235\textwidth}
        \caption{\hspace{4mm}\textbf{ResNet-34 on CIFAR-100}}
        \vspace{-2mm}
        \includegraphics[width=1.0\columnwidth]{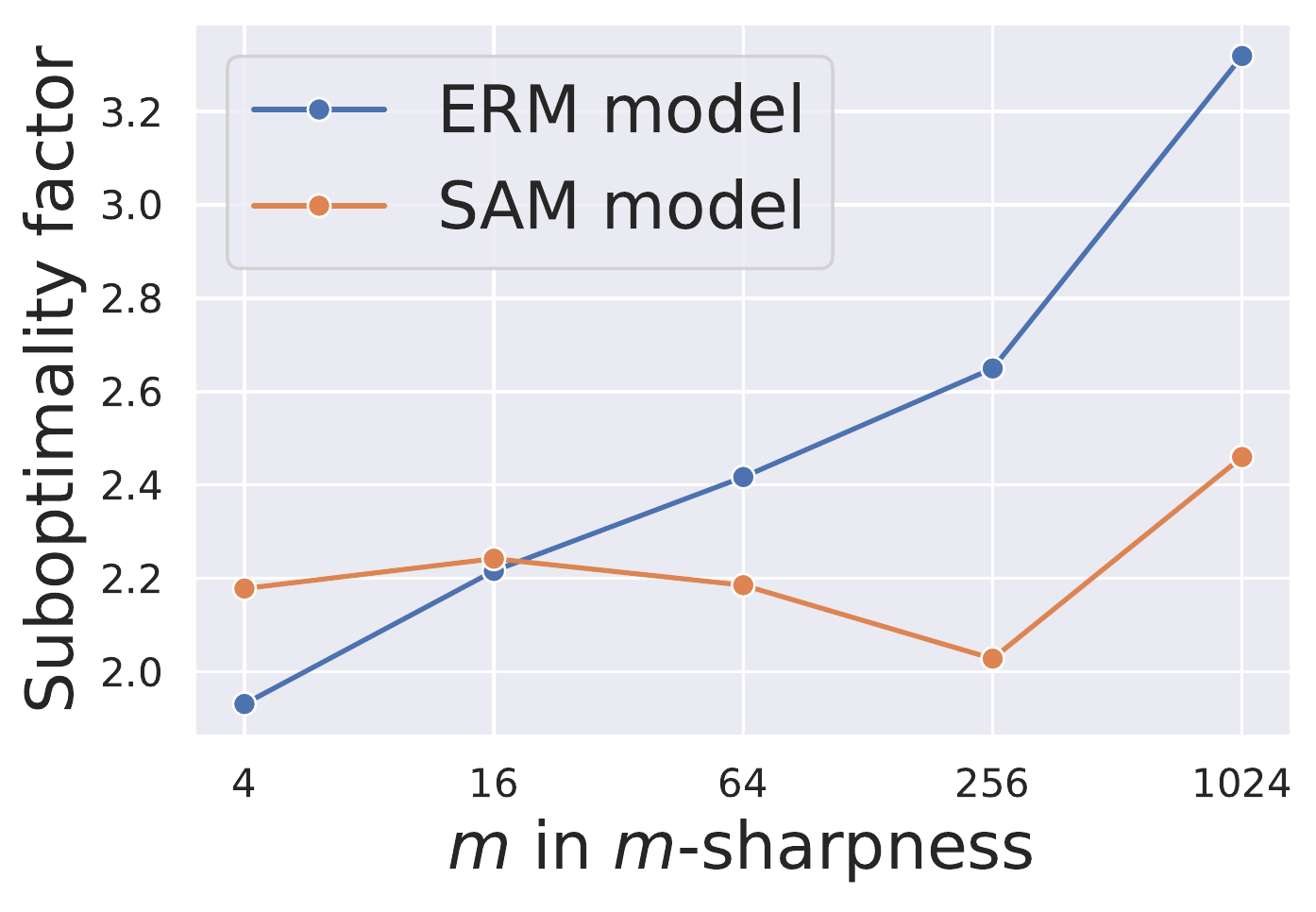}
    \end{subfigure}
    \vspace{-1.5mm}
    \caption{Suboptimality factor of $m$-sharpness ($\rho=0.1$) computed using 100 steps of projected gradient ascent compared to only 1 step for ERM and SAM models with group normalization.}
    \label{fig:sharpness_suboptimality}
\end{figure}
\begin{figure}[t]
    \centering
    \begin{subfigure}[t]{.235\textwidth}
        \caption{\hspace{4mm}\textbf{ResNet-18 on CIFAR-10}}
        \vspace{-2mm}
        \includegraphics[width=1.0\columnwidth]{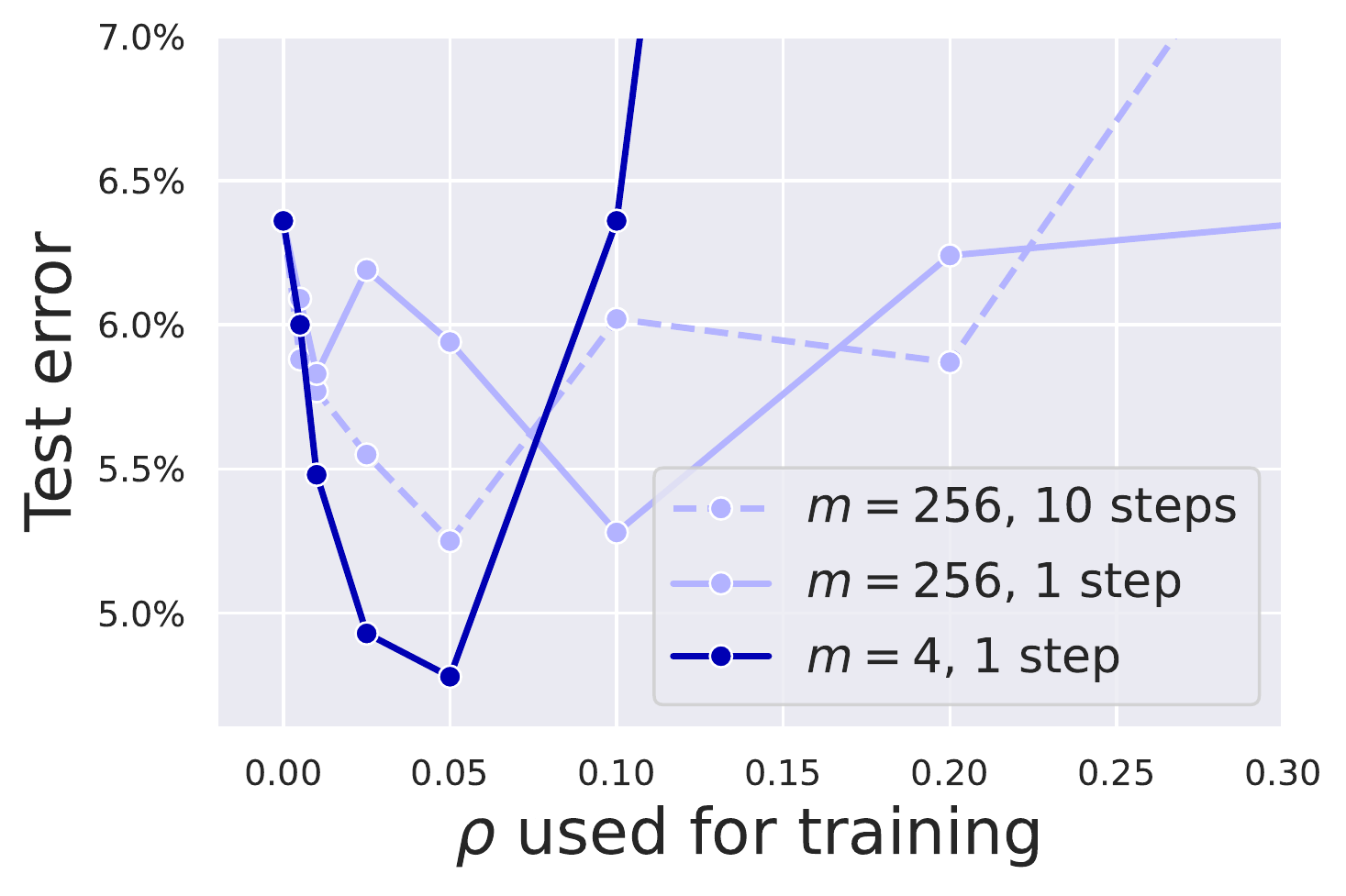}
    \end{subfigure}
    \begin{subfigure}[t]{.235\textwidth}
        \caption{\hspace{5mm}\textbf{ResNet-34 on CIFAR-100}}
        \vspace{-2mm}
        \includegraphics[width=1.0\columnwidth]{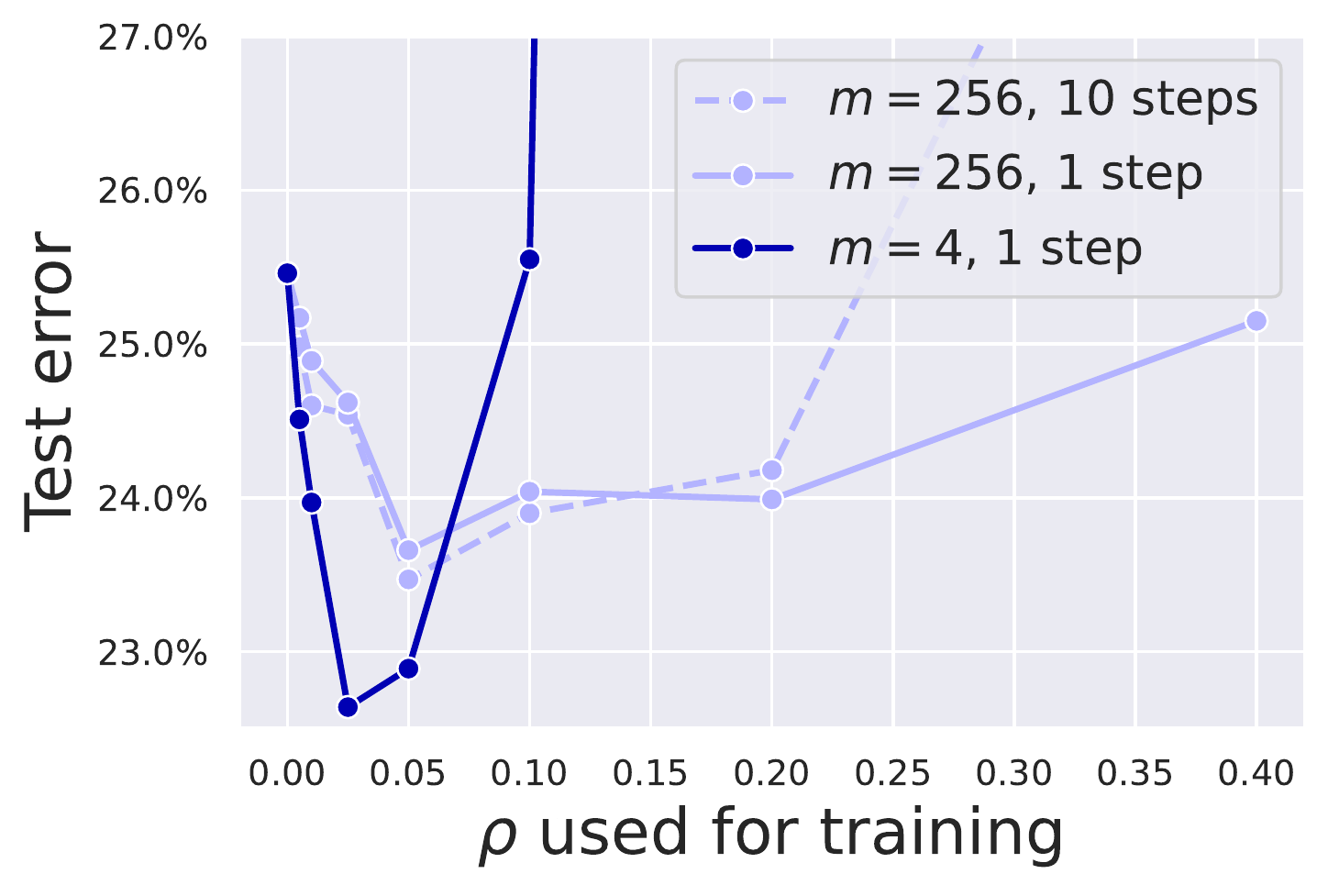}
    \end{subfigure}
    \vspace{-2mm}
    \caption{Test error of SAM models with group normalization trained with different numbers of projected gradient ascent steps ($10$ vs. $1$) for $m$-SAM and different $m$ values ($256$ vs. $4$) using batch size $256$.}
    \label{fig:test_err_sharpness_vs_rho_diff_m}
\end{figure}
As illustrated in Fig.~\ref{fig:gen_bound_is_not_predictive}, the success of $m$-SAM fully relies on the effect of low $m$ which is, however, remains unexplained in the current literature. As a starting point, we could consider the following two natural hypotheses for why low $m$ could be beneficial.

\textbf{Hypothesis 1}: \textbf{lower $m$ leads to more accurate maximization.} 
Since $m$-SAM relies only on a \textit{single} step of projected gradient ascent for the inner maximization problem in Eq.~(\ref{eq:m_sam}), it is unclear in advance how accurately this problem is solved. One could assume that using a lower $m$ can make the single-step solution more accurate as intuitively the function which is being optimized might become ``simpler'' due to fewer terms in the summation. Indeed, there is evidence towards this hypothesis: Fig.~\ref{fig:sharpness_suboptimality} shows the suboptimality factor between $m$-sharpness computed using 100 steps vs. 1 step of projected gradient ascent for $\rho=0.1$ (the optimal $\rho$ for $256$-SAM in terms of generalization) for ERM and SAM models. We can see that the suboptimality factor tends to increase over $m$ and can be as large as $10\times$ for the ERM model on CIFAR-10 for $m=1024$. This finding suggests that the standard single-step $m$-SAM can indeed fail to find an accurate maximizer and the value of $m$ can have a significant impact on it. 
However, despite this fact, using multiple steps in SAM \textit{does not} improve generalization as we show in Fig.~\ref{fig:test_err_sharpness_vs_rho_diff_m}. E.g., on CIFAR-10 it merely leads to a shift of the optimal $\rho$ from $0.1$ to $0.05$, without noticeable improvements of the test error. This is also in agreement with the observation from \citet{foret2021sharpnessaware} on why including second-order terms can slightly hurt generalization: solving the inner maximization problem more accurately leads to the fact that the same radius $\rho$ can become effectively too large (as on CIFAR-10) leading to worse performance.

\textbf{Hypothesis 2}: \textbf{lower $\bm{m}$ results in a better regularizing effect of batch normalization.}
As pointed out in \citet{hoffer2017train} and \citet{goyal2017accurate}, batch normalization (BN) has a beneficial regularization effect that depends on the mini-batch size. In particular, using the BN statistics from a smaller subbatch is coined as \textit{ghost batch normalization} \citep{hoffer2017train} and tends to improve generalization. Thus, it could be the case that the generalization improvement of $m$-SAM is due to this effect as its implementation assumes using a smaller subbatch of size $m$. To test this hypothesis, in Fig.~\ref{fig:test_err_sharpness_vs_rho_diff_m}, we show results of networks trained instead with \textit{group normalization} that does not lead to any extra dependency on the effective batch size. We can see that a significant generalization improvement by $m$-SAM is still achieved for low $m$ ($m=4$ for batch size $256$), and this holds for both datasets. Thus, the generalization improvement of $m$-SAM is not specific to BN.

We hypothesize instead that low-$m$ SAM leads to a better \textit{implicit} bias of gradient descent for commonly used neural network architectures, meaning that some important complexity measure of the model gets implicitly minimized over training that may not be obviously linked to $m$-sharpness.

\subsection{Provable Benefit of SAM for Diagonal Linear Networks}
\label{sec:impmain}
Here we theoretically study the implicit bias of full-batch $1$-SAM and $n$-SAM for diagonal linear networks on a sparse regression problem. We show that $1$-SAM has a better implicit bias than ERM \textit{and} $n$-SAM which explains its improved generalization in this setting.

\myparagraph{Implicit bias of $\bm{1}$-SAM and $\bm{n}$-SAM.}
\begin{figure}[t]
    \centering
    \includegraphics[width=.32\columnwidth]{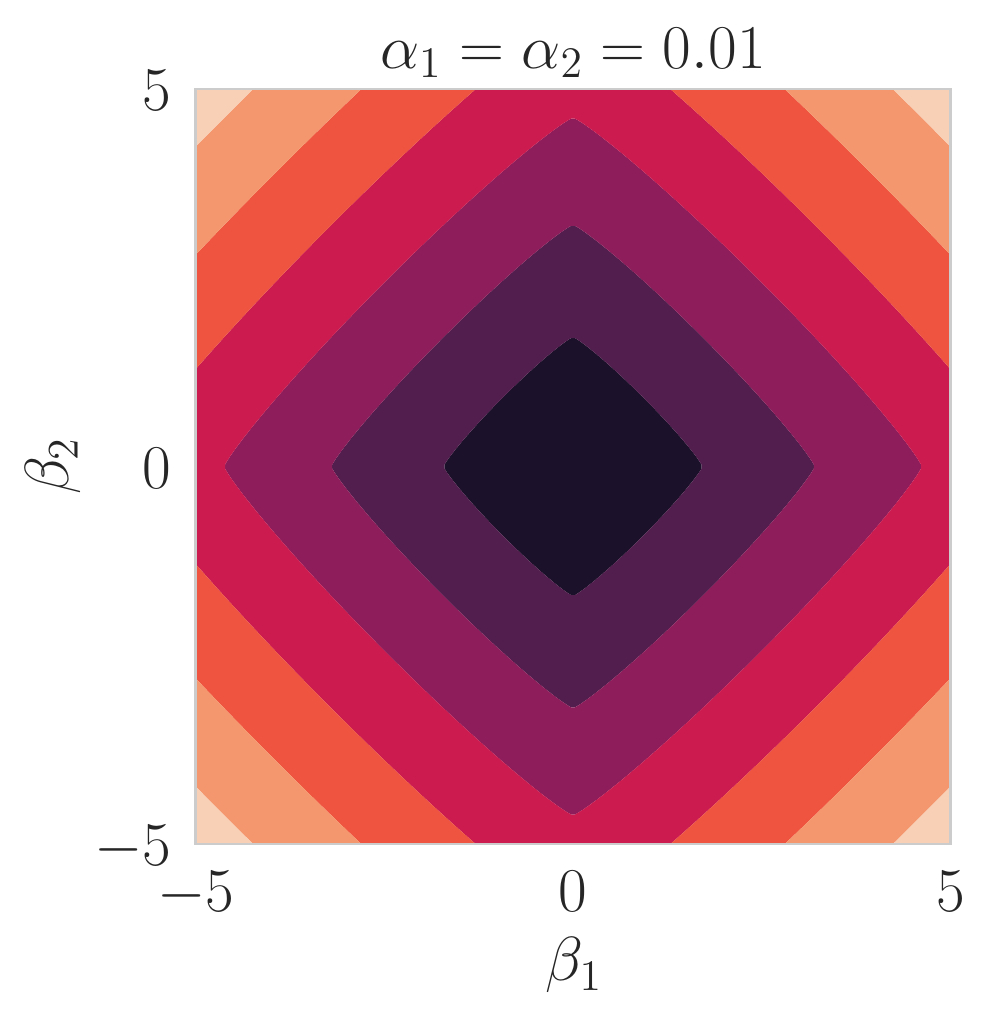}
    \includegraphics[width=.32\columnwidth]{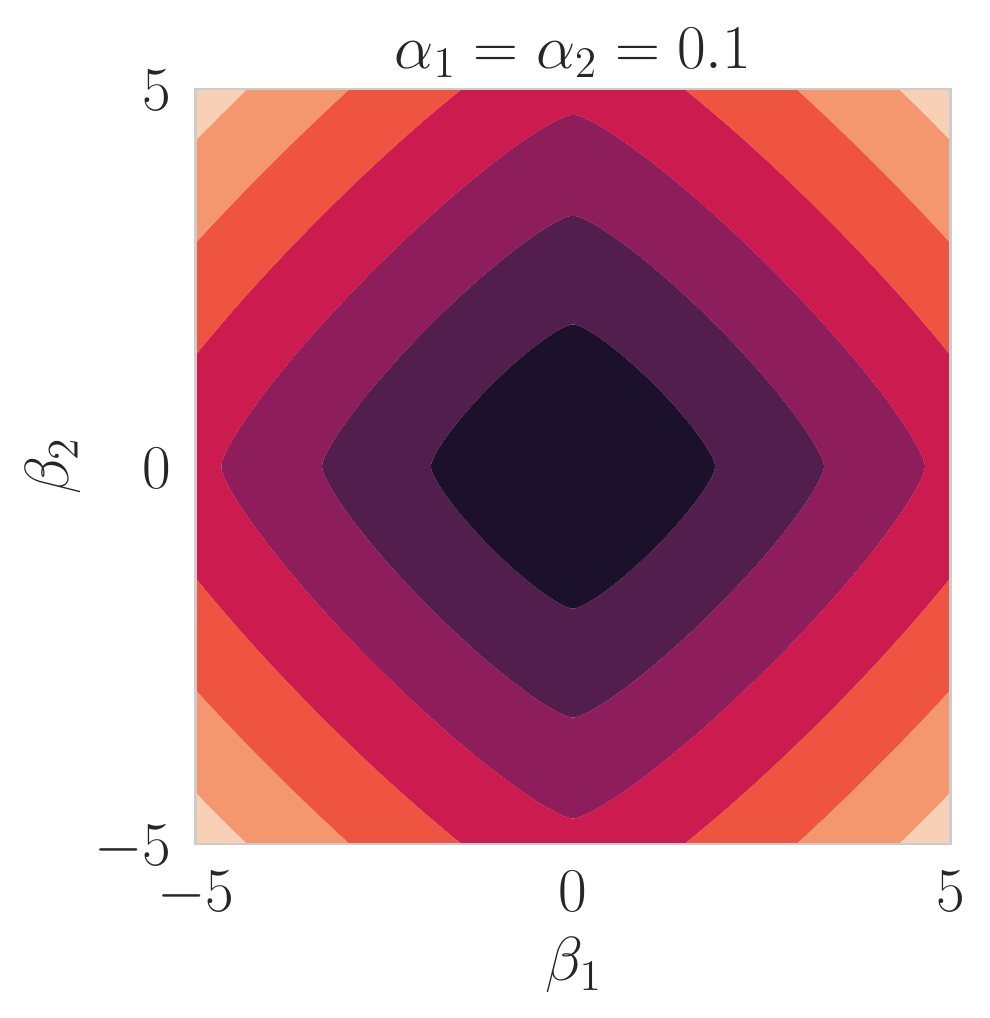}
    \includegraphics[width=.32\columnwidth]{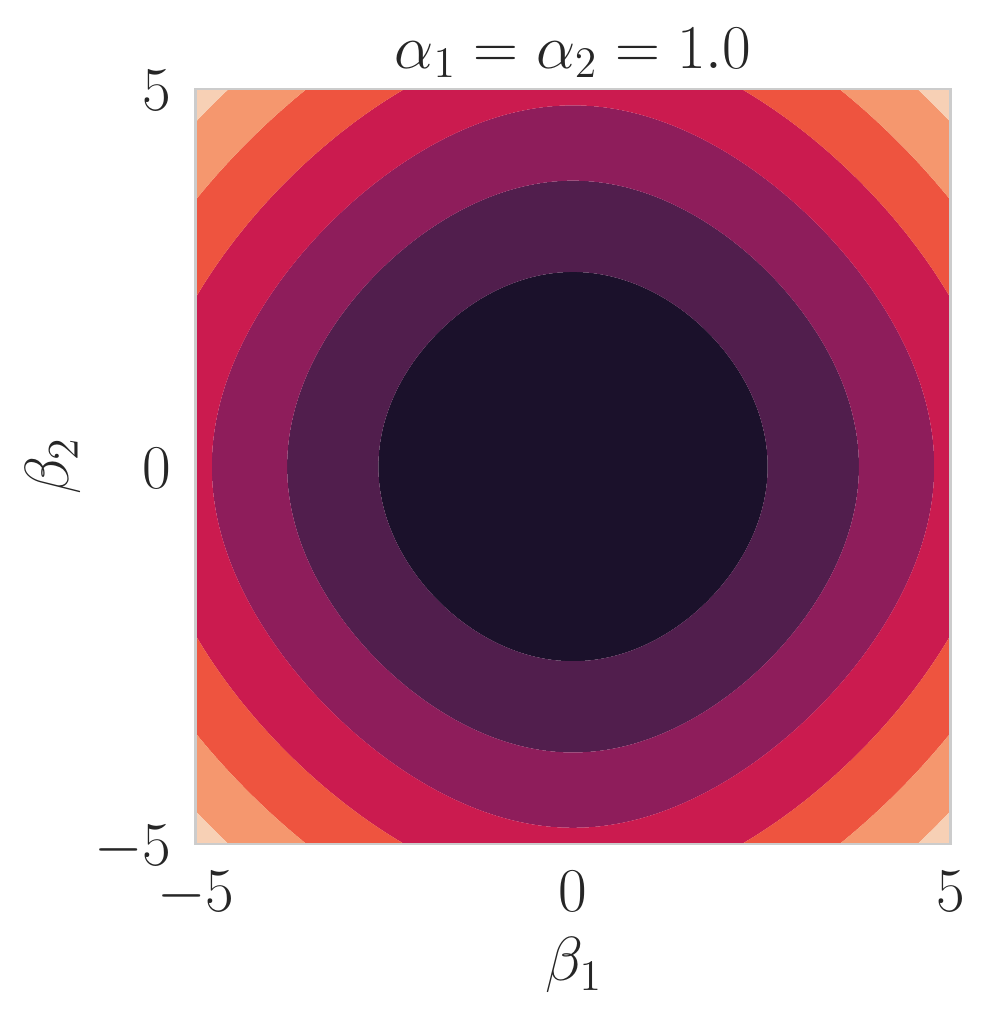}
    \vspace{-2mm}
    \caption{Illustration of the hyperbolic entropy $\phi_\alpha(\beta)$ for $\beta \in \R^2$ that interpolates between $\|\beta\|_1$ for small $\alpha$ and $\|\beta\|_2$ for large $\alpha$.}
    \label{fig:hyperbolic_entropy}
\end{figure}
The implicit bias of gradient methods is well understood for overparametrized linear models where all gradient-based algorithms enjoy the same implicit bias towards minimization of the $\ell_2$-norm of the parameters. For diagonal linear neural networks, where a linear predictor $\langle \beta, x\rangle$ can be parametrized via $\beta = w_+^2 - w_-^2$\footnote{See \citet{woodworth2020kernel} for why this parametrization is equivalent to a diagonal network $\beta = u \odot v$. Moreover, the signs of $u_i$ and $v_i$ will not change throughout training, hence the use of the notation $w_+$ and $w_-$.} with a parameter vector $w = \left[\begin{smallmatrix}w_+\\w_-\end{smallmatrix}\right] \in \R^{2 d}$, first-order algorithms have a richer implicit bias. %
We consider here an overparametrized sparse regression problem, meaning that the ground truth $\beta^*$ is a sparse vector, with the squared loss:
\vspace{-5mm}
\begin{align}\label{eq:quadpro}
    L(w) := \frac{1}{4 n} \sum_{i = 1}^n (\langle w_+^2 - w_-^2, x_i \rangle- y_i)^2, 
\end{align}
where overparametrization means that $n \ll d$ and there exist many $w$ such that $L(w)=0$. We note that in our setting, any global minimizer $w^*$ of $L(w^*)$ is also a global minimizer for the $m$-SAM algorithm for any $m \in \{1, \dots, n\}$ since all per-example gradients are zero and hence the ascent step of SAM will not modify $w^*$. Thus, any difference in generalization between $m$-SAM and ERM has to be attributed rather to the \textit{implicit} bias of each of these algorithms.

We first recall the seminal result of \citet{woodworth2020kernel} and refer the readers to App.~\ref{sec:app_implicit_bias} for further details. Assuming global convergence, the solution selected by the 
gradient flow initialized as $w_+ = w_- = \alpha \in \R_{>0}^d$ and denoted $\beta_\infty^\alpha$ solves the following constrained optimization problem:
\begin{align}
\label{eq:potential_function}
    \beta_\infty^\alpha = \underset{\beta \in \R^d \ \text{s.t.} \ X\beta = y}{\argmin} \phi_\alpha (\beta),
\end{align}
where the potential $\phi_\alpha$ is given as $\phi_\alpha  (\beta) =  \sum_{i=1}^d \alpha_i^2 q(\beta_i/\alpha_i^2)$ with  $q(z)=2-\sqrt{4+z^2}+z\arcsinh(z/2)$.
As illustrated in Fig.~\ref{fig:hyperbolic_entropy}, $\phi_\alpha$ interpolates between the $\ell_1$ and the $\ell_2$ norms of $\beta$ according to the initialization scale $\alpha$. Large $\alpha$'s lead to low $\ell_2$-type solutions, while small $\alpha$'s lead to low $\ell_1$-type solutions which are known to induce good generalization properties for sparse problems~\citep{woodworth2020kernel}. 

Our main theoretical result is that both $1$-SAM and $n$-SAM dynamics, when considered in their full-batch version (see Sec.~\ref{sec:app_fullbatch} for details), bias the flow towards solutions which minimize the potential $\phi_\alpha$ but with effective parameters $\alpha_{\text{1-SAM}}$ and $\alpha_{\text{n-SAM}}$ which are strictly smaller than $\alpha$ for a suitable inner step size $\rho$. In addition, typically $\|\alpha_{\text{1-SAM}}\|_1 < \|\alpha_{\text{n-SAM}}\|_1$ and, therefore, the solution chosen by $1$-SAM has better sparsity-inducing properties than the solution of $n$-SAM and standard ERM. 
\begin{theorem}[Informal] \label{theorem:implicit_bias_main}
    Assuming global convergence, the solutions selected by the full-batch versions of the $1$-SAM and $n$-SAM algorithms taken with infinitesimally small step sizes and initialized at $w_+ = w_- = \alpha \in \R_{>0}^d$, solve the optimization problem~(\ref{eq:potential_function}) with effective parameters:
    \begin{equation*}
       \alpha_{\text{1-SAM}} = \alpha \odot e^{- \rho \Delta_{\text{1-SAM}}  +O(\rho^2)},
       \ \ 
       \alpha_{\text{n-SAM}} = \alpha \odot e^{- \rho \Delta_{\text{n-SAM}}  +O(\rho^2)},
    \end{equation*} 
    where $\Delta_{\text{1-SAM}}, \Delta_{\text{n-SAM}} \in \R^d_+$ for which typically:
    \begin{align*}
         &\|\Delta_{\text{1-SAM}}\|_1\approx d \int_0^\infty L(w(s))ds 
         \text{\ \ \ and\ \ \ } \\
         &\|\Delta_{\text{n-SAM}}\|_1\approx \frac{d}{n} \int_0^\infty L(w(s))ds.
    \end{align*}
\end{theorem}
The results are formally stated in Theorem~\ref{theorem:biasmaxsum} and \ref{theorem:biassummax} in App.~\ref{sec:app_implicit_bias}. $1$-SAM has better implicit bias properties since its effective scale of $\alpha$ is considerably smaller than the one of $n$-SAM due to the lack of the $\frac{1}{n}$ factor in the exponent. It is worth noting that the vectors $\Delta_{\text{1-SAM}}$ and $\Delta_{\text{n-SAM}}$ are linked with the integral of the loss function along the flow. Thereby, the speed of convergence of the training loss impacts the magnitude of the biasing effect: the slower the convergence, the better the bias, similarly to what is observed for SGD in~\citet{pesme2021implicit}. Extending this result to stochastic implementations of $1$-SAM and $n$-SAM algorithms could be done following \citet{pesme2021implicit} but is outside of the scope of this paper.

\begin{figure}[t]
    \centering
    \includegraphics[width=.48\columnwidth]{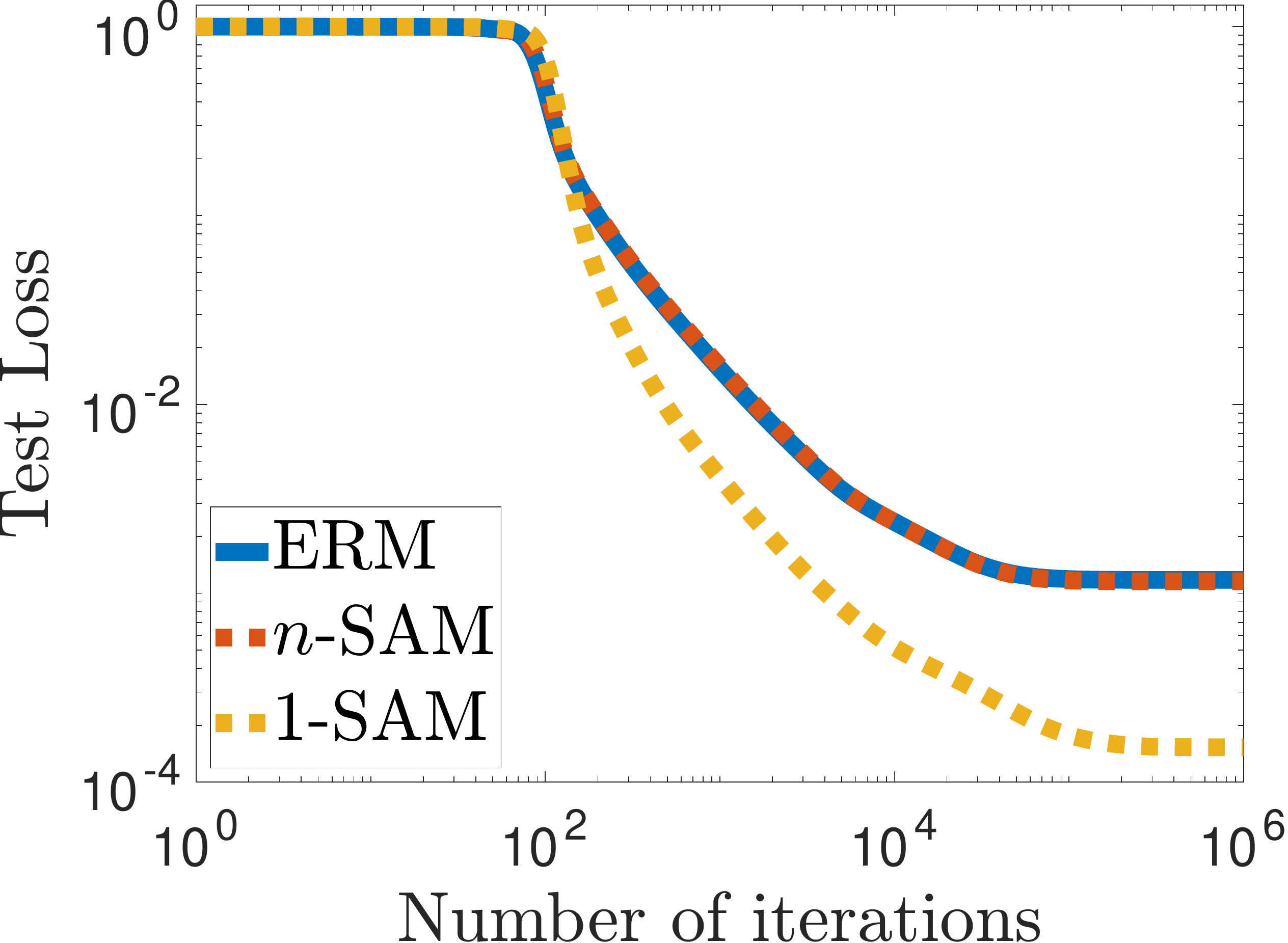}
    \includegraphics[width=.48\columnwidth]{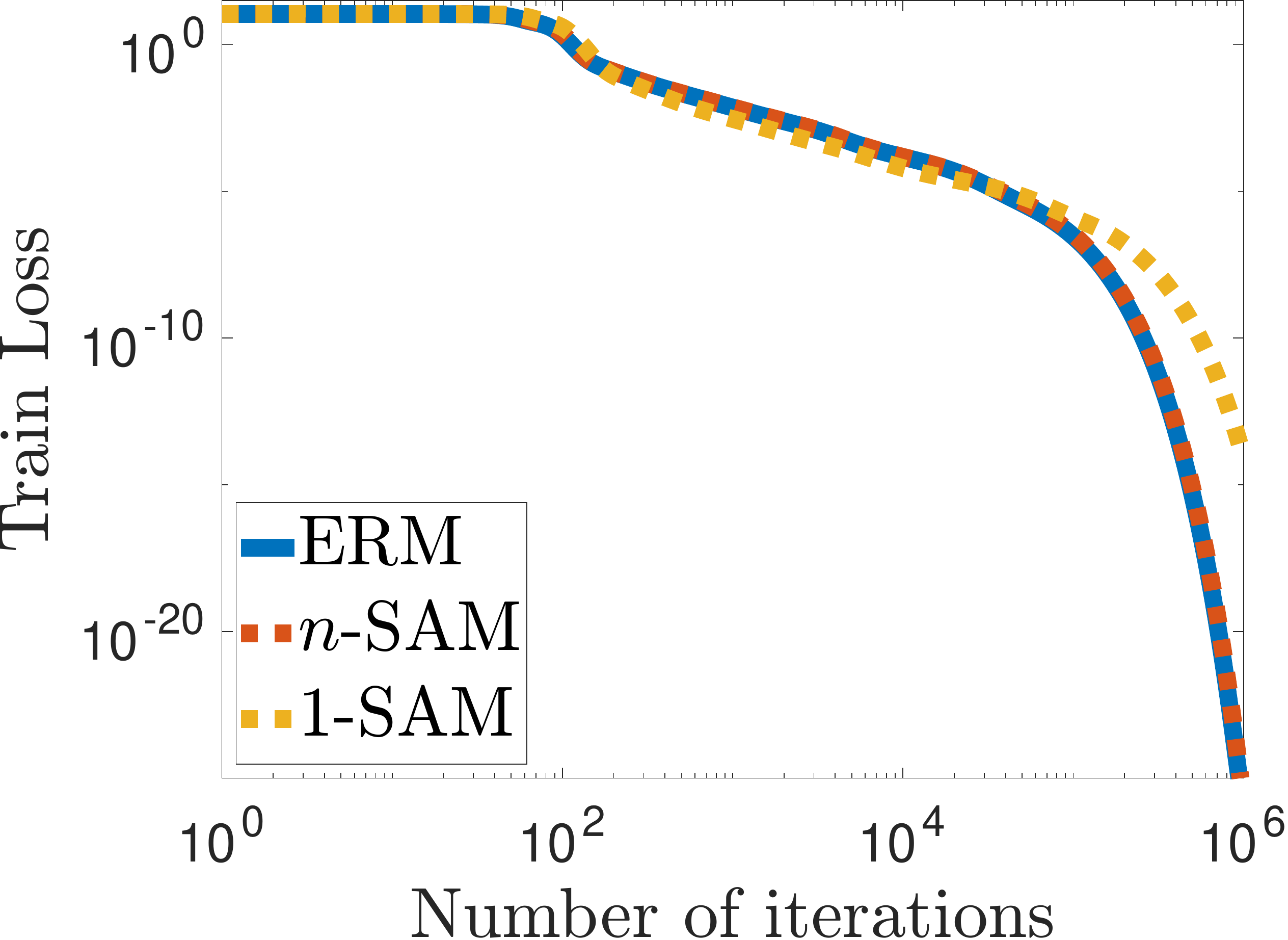}
    \caption{Implicit bias of $1$-SAM and $n$-SAM compared to ERM for a diagonal linear network on a sparse regression problem. We can see that $1$-SAM generalizes significantly better than $n$-SAM and ERM.}
    \label{fig:implicit_bias_sam}
\end{figure}
\myparagraph{Empirical evidence for the implicit bias.}
We compare the training and test loss of ERM, $1$-SAM, and $n$-SAM in Fig.~\ref{fig:implicit_bias_sam} for the same perturbation radius $\rho$, and for different $\rho$ in App.~\ref{subsec:comparison_1sam_nsam} (Fig.~\ref{fig:diag_net_gs_nsam_1sam}). 
As predicted, the methods show different generalization abilities: ERM and $n$-SAM achieve approximately the same performance whereas $1$-SAM clearly benefits from a better implicit bias. This is coherent with the deep learning experiments presented in Fig.~\ref{fig:gen_bound_is_not_predictive} on CIFAR-10 and CIFAR-100.
We also note that the training loss of all the variants is converging to zero but the convergence of $1$-SAM is slower. 
Additionally, we show a similar experiment with \textit{stochastic} variants of the algorithms in App.~\ref{subsec:comparison_1sam_nsam} (Fig.~\ref{fig:implicit_bias_sto}) where their performance is, as expected, better compared to their deterministic counterparts. %

\subsection{Empirical Study of the Implicit Bias in Non-Linear Networks}
\label{subsec:empirical_study_implicit_bias}
Here we conduct a series of experiments to characterize the implicit bias of SAM on \textit{non-linear} networks.

\myparagraph{The sparsity-inducing bias of SAM for a simple ReLU network.}
\begin{figure}[t]
    \centering
    \begin{subfigure}[t]{.235\textwidth}
        \caption{\hspace{4mm}\textbf{ERM}}
        \vspace{-2mm}
        \includegraphics[width=1.0\columnwidth]{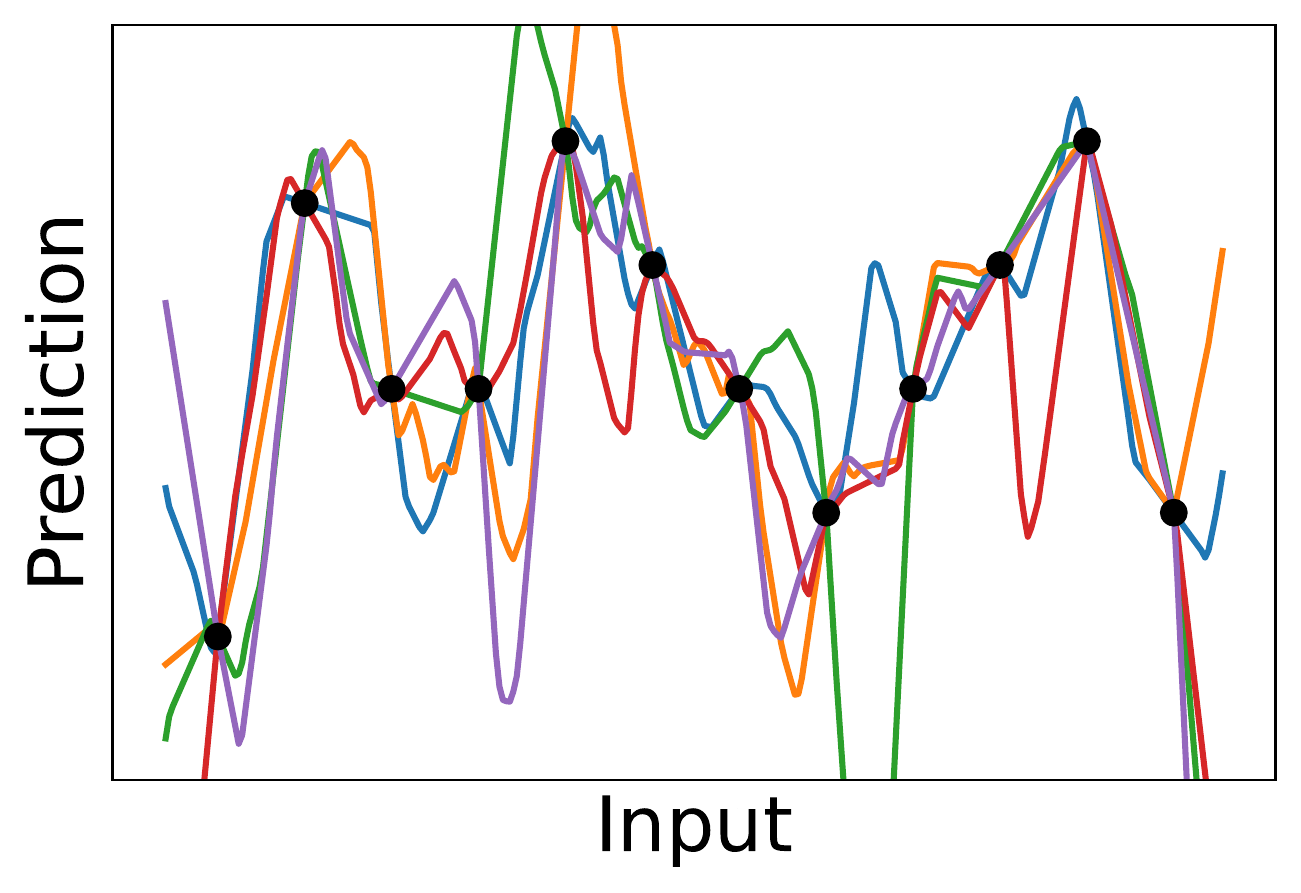}
    \end{subfigure}
    \begin{subfigure}[t]{.235\textwidth}
        \caption{\hspace{4mm}\textbf{SAM}}
        \vspace{-2mm}
        \includegraphics[width=1.0\columnwidth]{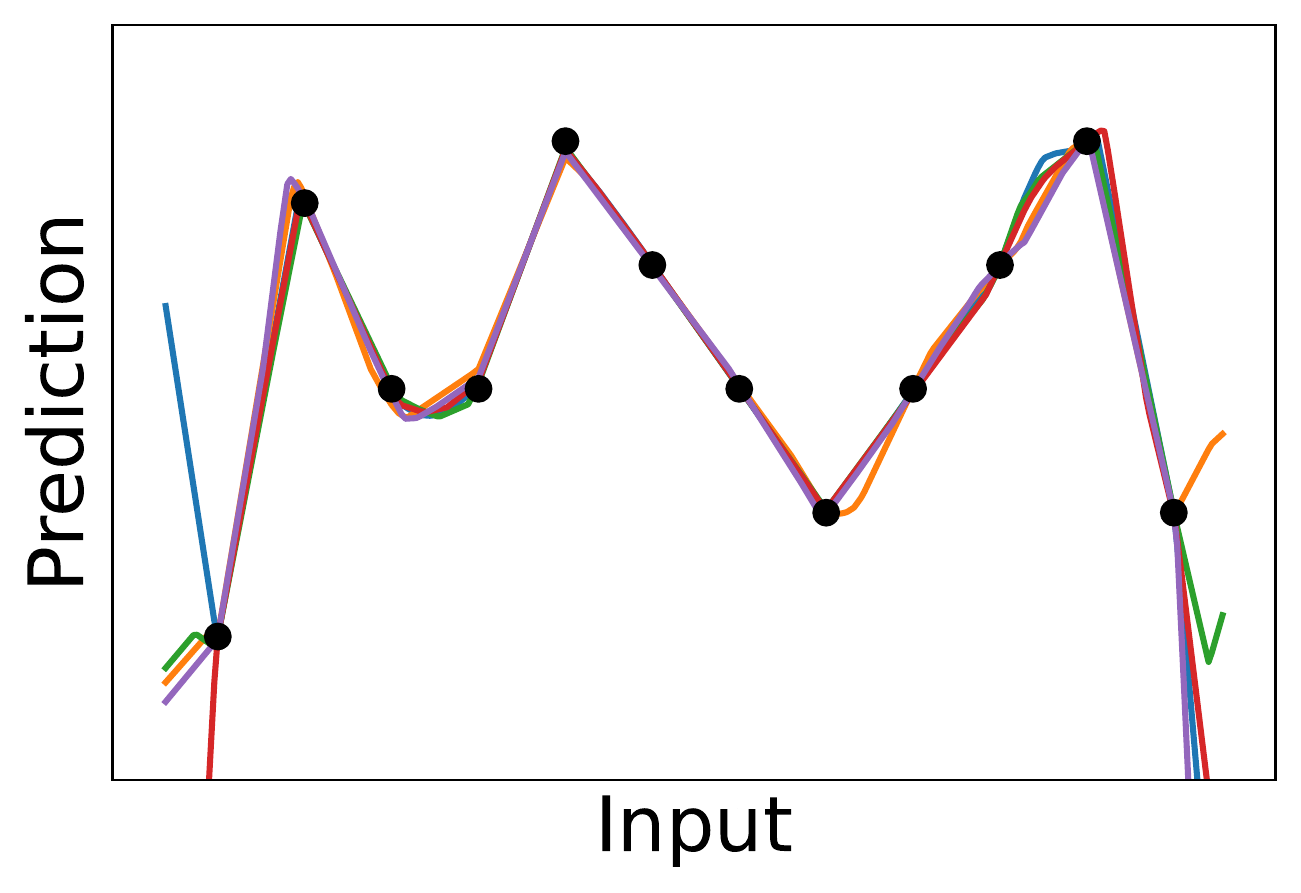}
    \end{subfigure}
    \caption{The effect of the implicit bias of ERM vs. SAM for a one hidden layer ReLU network trained with full-batch gradient descent. Each run is replicated over five random initializations.}
    \label{fig:implicit_bias_relu_net}
\end{figure}
We start from the simplest non-linear network: a one hidden layer ReLU network applied to a simple 1D regression problem from \citet{blanc2020implicit}.  We use it to illustrate the implicit bias of SAM in terms of the geometry of the learned function. 
For this, we train ReLU networks with 100 hidden units using full-batch gradient descent on the quadratic loss with ERM and SAM\footnote{Since $n=12$ for this task, we observed no substantial difference between $1$-SAM and $n$-SAM.} over five different random initializations. 
We plot the resulting functions in Fig.~\ref{fig:implicit_bias_relu_net}.
We observe that SAM leads to simpler interpolations of the data points than ERM, and it is much more stable over random initializations. In particular, SAM seems to be biased toward a sparse combination of ReLUs 
which is reminiscent of \citet{chizat20a} who show that the limits of the gradient flow can be described as a max-margin classifier that favors hidden low-dimensional structures by implicitly regularizing the $\mathcal F_1$ variation norm. 
Moreover, this also relates to our Theorem~\ref{theorem:implicit_bias_main} where sparsity rather shows up in terms of the lower $\ell_1$-norm of the resulting linear predictor.
This further illustrates that there can exist multiple ways in which one can describe the beneficial effect of SAM. 
For deep non-linear networks, however, the effect of SAM is hard to visualize, but we can still characterize some of its important properties. %

\myparagraph{The effect of SAM for deep networks at different stages of training.} 
\begin{figure}[t]
    \centering
    \begin{subfigure}[t]{.37\textwidth}
        \caption{\hspace{6mm}\textbf{ResNet-18 on CIFAR-10}}
        \vspace{-2mm}
        \includegraphics[width=1.0\columnwidth]{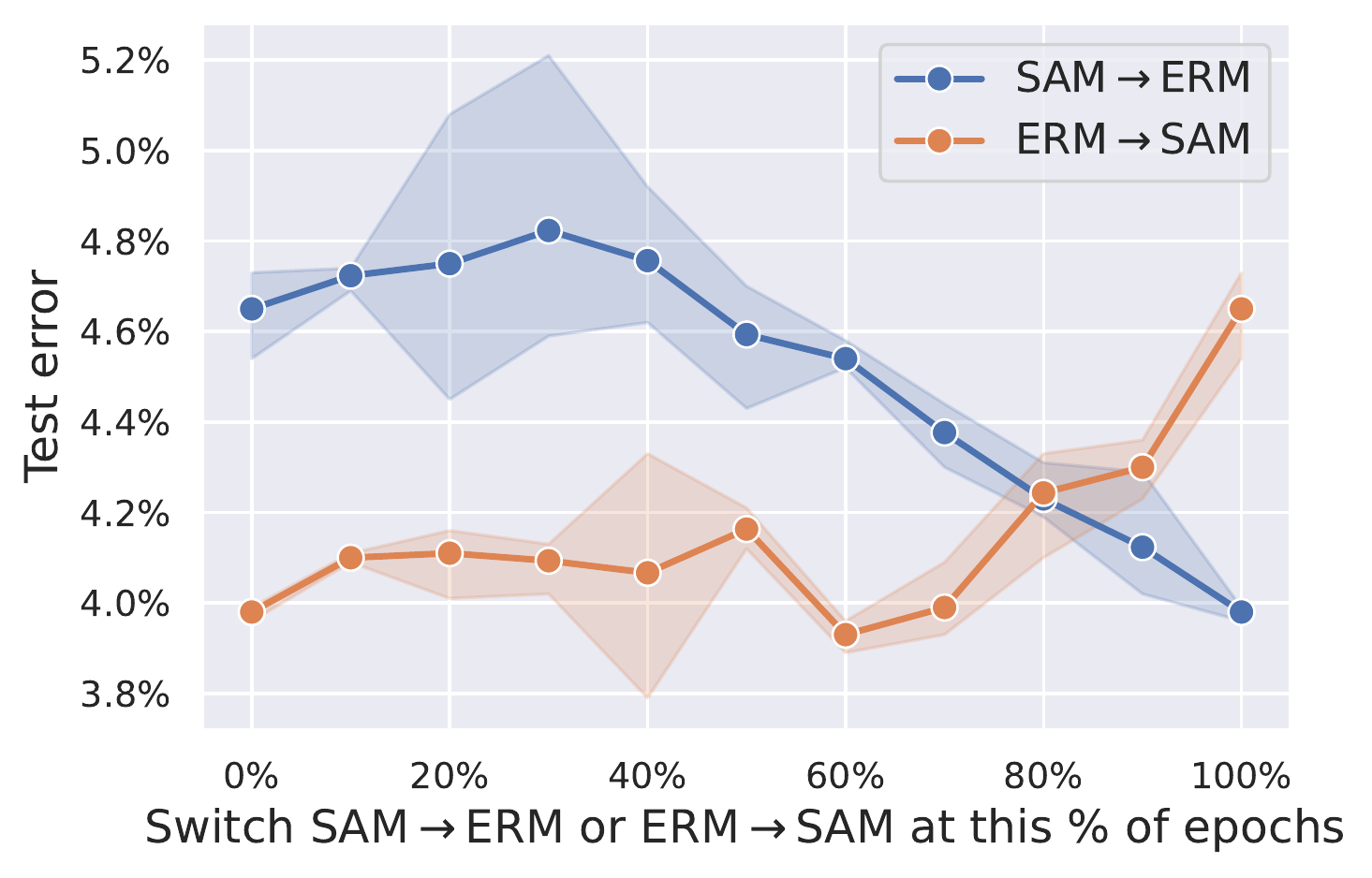}
    \end{subfigure}
    \begin{subfigure}[t]{.38\textwidth}
        \vspace{1mm}
        \caption{\hspace{6mm}\textbf{ResNet-34 on CIFAR-100}}
        \vspace{-2mm}
        \includegraphics[width=1.0\columnwidth]{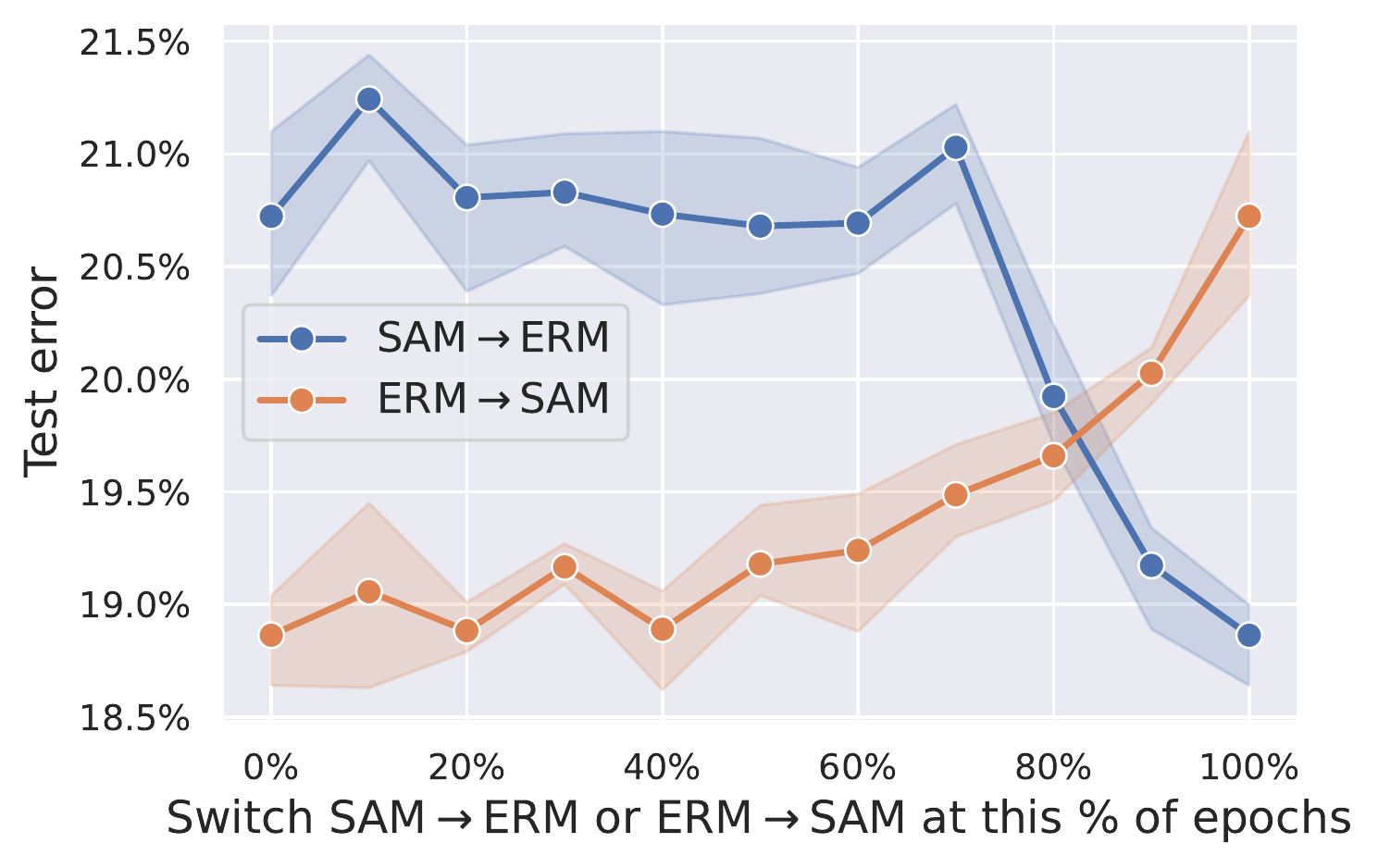}
    \end{subfigure}
    \vspace{-1.5mm}
    \caption{Test error of SAM~$\rightarrow$~ERM and ERM~$\rightarrow$~SAM when the methods are switched at different $\%$ of epochs. For example, for SAM~$\rightarrow$~ERM, 0\% corresponds to ERM and 100\% corresponds to SAM. We observe that a method which is run at the beginning of training has little influence on the final performance.}
    \label{fig:erm_sam_sam_erm}
\end{figure}
To develop a better understanding of the implicit bias of SAM for deep networks, we can analyze at which stages of training using SAM is necessary to get generalization benefits. One could assume, for example, that its effect is important \textit{only} early in training so that the first updates of SAM steer the optimization trajectory towards a better-generalizing minimum. In that case, switching from SAM to ERM would not degrade the performance. To better understand this, we train models first with SAM and then switch to ERM for the remaining epochs (SAM~$\rightarrow$~ERM) and also do a complementary experiment by switching from ERM to SAM (ERM~$\rightarrow$~SAM) and show results in Fig.~\ref{fig:erm_sam_sam_erm}. Interestingly, we observe that a method that is used at the beginning of training has little influence on the final performance. E.g., when SAM is switched to ERM within the first 70\% epochs on CIFAR-100, the resulting model generalizes as well as ERM. Furthermore, we note a high degree of continuity of the test error with respect to the number of epochs at which we switch the methods. This does not support the idea that the models converge to some entirely distinct minima and instead suggests convergence to different minima in a connected valley where some directions generalize progressively better.
Another intriguing observation is that enabling SAM only towards the end of training is sufficient to get a significant improvement in terms of generalization. We discuss this phenomenon next in more detail.

\begin{figure}[t]
    \centering
    \begin{subfigure}[t]{.235\textwidth}
        \caption{\hspace{4mm}\textbf{ResNet-18 on CIFAR-10}}
        \vspace{-2mm}
        \includegraphics[width=1.0\columnwidth]{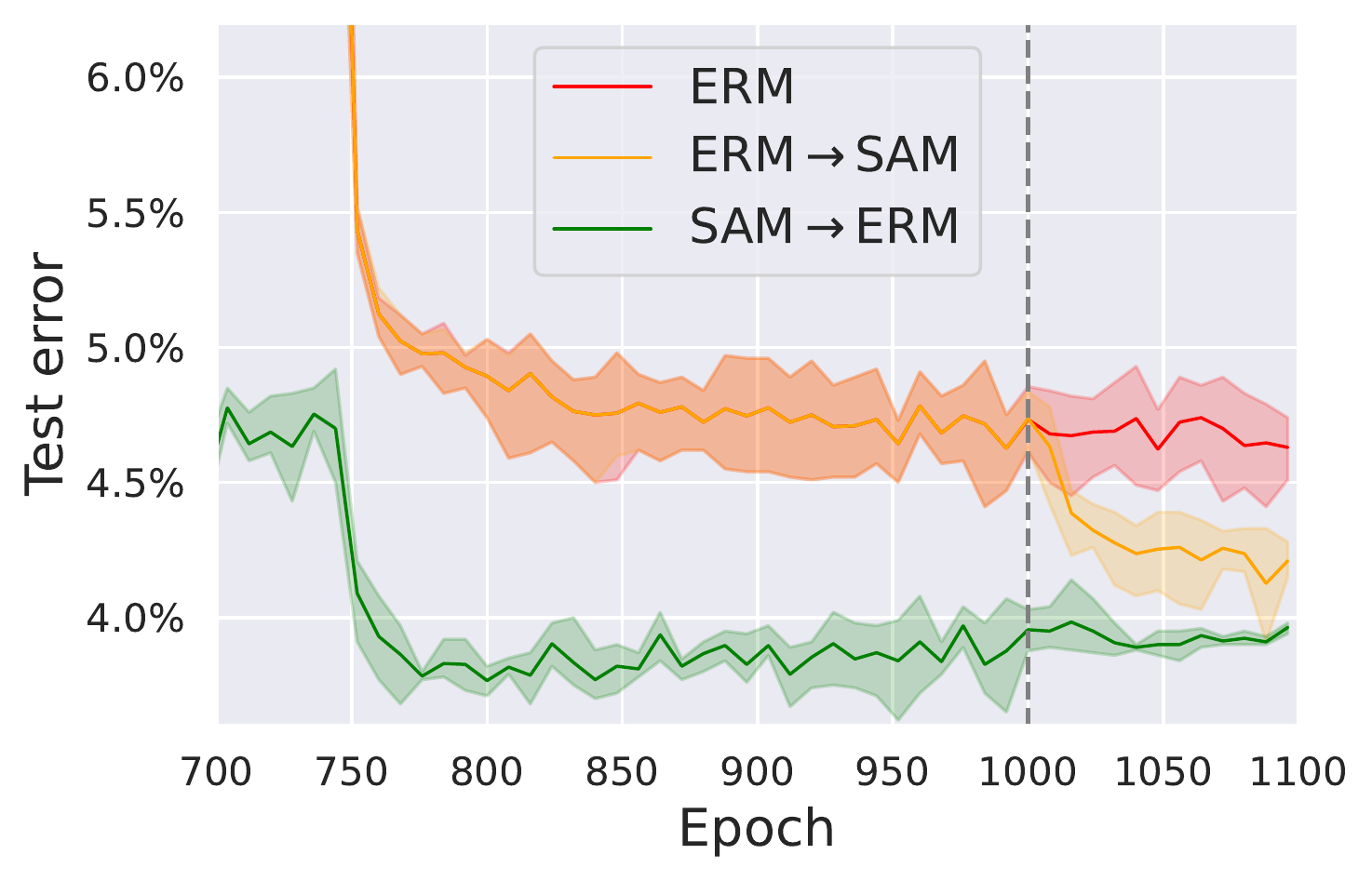}
    \end{subfigure}
    \begin{subfigure}[t]{.235\textwidth}
        \caption{\hspace{4mm}\textbf{ResNet-34 on CIFAR-100}}
        \vspace{-2mm}
        \includegraphics[width=1.0\columnwidth]{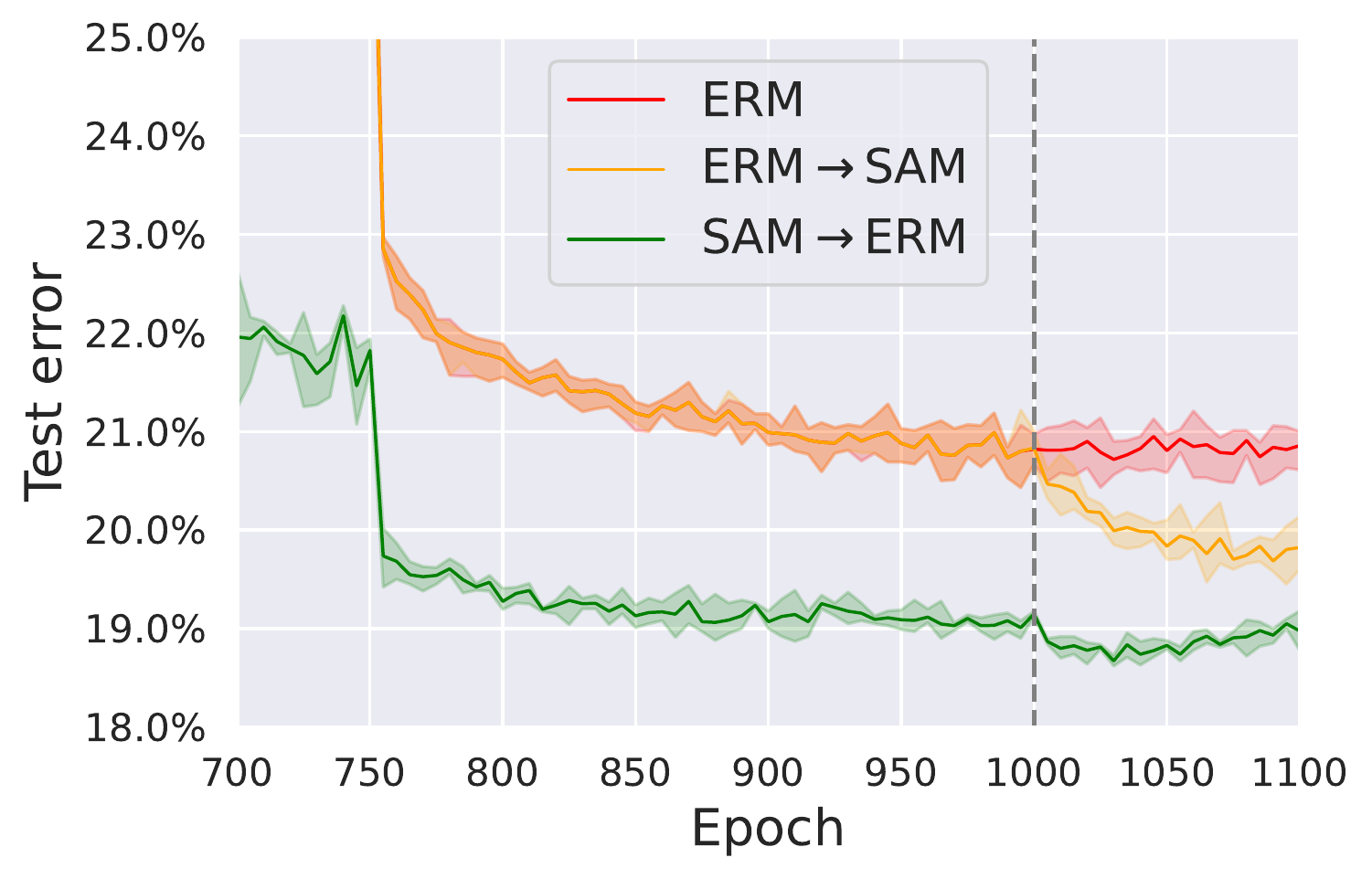}
    \end{subfigure}
    \vspace{-1.5mm}
    \caption{Test error over epochs for ERM compared to ERM~$\rightarrow$~SAM and SAM~$\rightarrow$~ERM training where the methods are switched only at the end of training. In particular, we can see that SAM can gradually escape the worse-generalizing minimum found by ERM.
    }
    \label{fig:erm_sam_finetuning}
\end{figure}
\begin{figure}[t]
    \centering
    \includegraphics[width=.85\columnwidth]{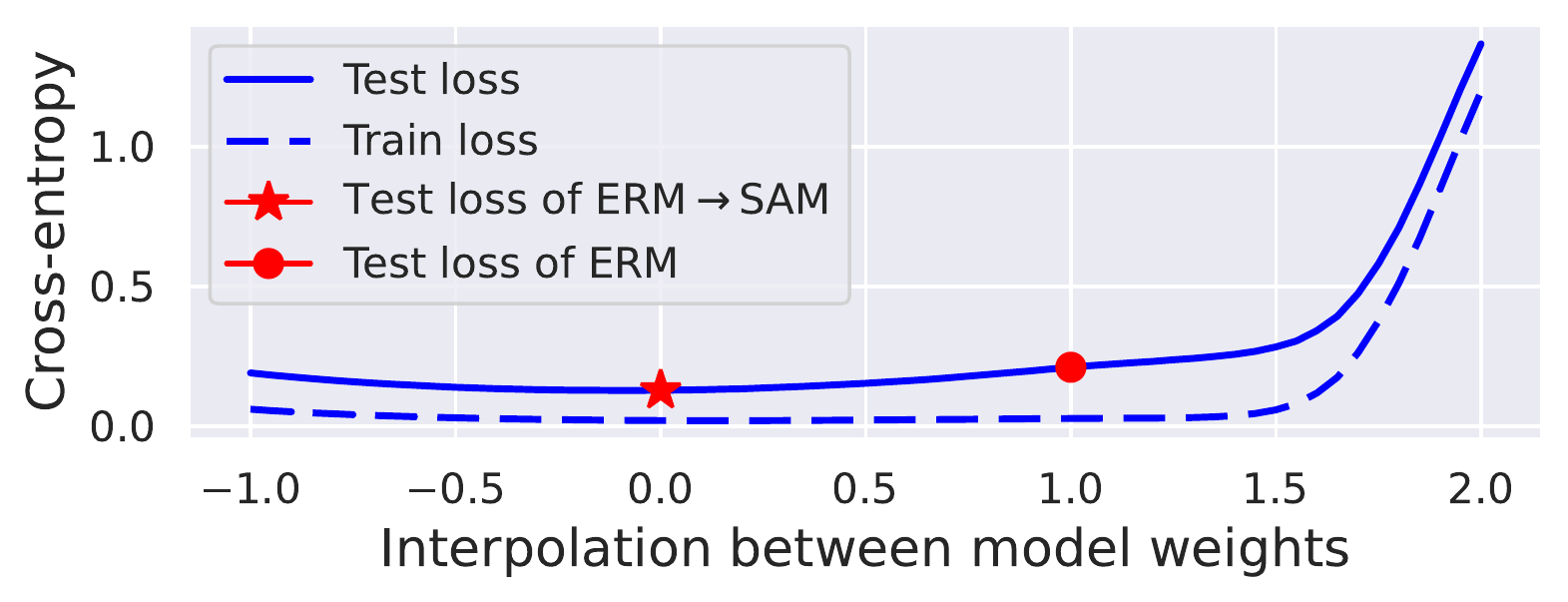}
    \vspace{-2mm}
    \caption{Loss interpolations between $w_{ERM \rightarrow SAM}$ and $w_{ERM}$ for a ResNet-18 trained on CIFAR-10.}
    \label{fig:deep_nets_loss_surfaces}
\end{figure}
\myparagraph{The importance of the implicit bias of SAM at the end of training.} 
We take a closer look on the performance of ERM~$\rightarrow$~SAM and SAM~$\rightarrow$~ERM when we switch between the methods only for the last $\approx10\%$ of epochs in Fig.~\ref{fig:erm_sam_finetuning} where we plot the test error over epochs.
First, we see that for SAM~$\rightarrow$~ERM, once SAM converges to a well-generalizing minimum thanks to its implicit bias, then it is not important whether we continue optimization with SAM or with ERM, and we do not observe significant overfitting when switching to ERM. 
At the same time, for ERM~$\rightarrow$~SAM we observe a different behavior: the test error clearly improves when switching from ERM to SAM. 
This suggests that SAM (using a higher $\rho$ than the standard value, see App.~\ref{sec:app_exp_details}) can gradually escape the worse-generalizing minimum which ERM converged to.
This phenomenon is interesting since it suggests a \textit{practically relevant} fine-tuning scheme that can save computations as we can start from any pre-trained model and substantially improve its generalization. %
Moreover, interestingly, the final point of the ERM~$\rightarrow$~SAM model is situated \textit{in the same basin} as the original ERM model as we show in Fig.~\ref{fig:deep_nets_loss_surfaces} which resembles the asymmetric loss interpolations observed previously for stochastic weight averaging \citep{he2019asymmetric}.

We make very similar observations regarding fine-tuning with SAM and linear connectivity also on a diagonal linear network as shown in App.~\ref{subsec:comparison_1sam_nsam} (Fig.~\ref{fig:diag_net_erm_sam_sam_erm}).
We believe the observations from Fig.~\ref{fig:erm_sam_finetuning} can be explained by our Theorem~\ref{theorem:implicit_bias_main} which shows that for diagonal linear networks, the key quantity determining the magnitude of the implicit bias for SAM is the integral of the loss over the optimization trajectory $w(s)$. In the case of ERM~$\rightarrow$~SAM, the integral is taken only over the last epochs but this can still be sufficient to improve the biasing effect. At the same time, for SAM~$\rightarrow$~ERM, the integral is already large enough due to the first 1000 epochs with SAM and switching back to ERM preserves the implicit bias. We discuss it in more detail in App.~\ref{subsec:comparison_1sam_nsam}.

\section{Understanding the Optimization Aspects of SAM}
\label{sec:convergence_sam} 

The results on the implicit bias of SAM presented above require that the algorithm converges to zero training error.
In the current literature, however, a convergence analysis (even to a stationary point) is missing for SAM.
In particular, we do not know what are the conditions on the training ERM loss, inner step size $\gamma_t$, and perturbation radius $\rho_t$ so that SAM is guaranteed to converge. We also do not know whether SAM converges to a stationary point of the ERM objective. 
To fill in this gap, we first theoretically study convergence of SAM and then relate the theoretical findings with empirical observations on deep networks.

\subsection{Theoretical Analysis of Convergence of SAM}
\label{subsec:convergence_analysis}

Here we show that SAM leads to convergence guarantees in terms of the standard training loss. 
In the following, we analyze the convergence of the $m$-SAM algorithm %
whose update rule is defined in Eq.~(\ref{eq:sambatch}). We make the following assumptions on the training loss $L(w) = \frac{1}{n} \sum_{i=1}^n \ell_i(w)$:
\begin{description}
 \item[\textbf{(A1)}](Bounded variance). \textit{There exists  $\sigma \geq 0$ s.t. $\E[ \| \nabla \ell_i(w)-\nabla L (w)\|^2]\leq \sigma^2$ for all $i \sim \mathcal U( \llbracket1,n \rrbracket)$ and $w\in \mathbb R ^d$.}
\item[\textbf{(A2)}] (Individual $\beta$-smoothness). \textit{There exists  $\beta \geq 0$ s.t. $ 
\| \nabla \ell_i (w) -\nabla \ell_i (v) \| \leq \beta \| w-v \|$ for all $w,v  \in \mathbb R ^d$ and $i\in \llbracket1,n \rrbracket$.}
\item [\textbf{(A3)}] (Polyak-Lojasiewicz). \textit{There exists $\mu > 0$ s.t. $\frac{1}{2}\| \nabla L (w) \|^2 \geq \mu (L(w)-L_*)$ for all $w,v  \in \mathbb R ^d$.}
\end{description}
Both assumptions \textbf{(A1)} and \textbf{(A2)} are standard in the optimization literature and should hold for neural networks with smooth activations and losses (such as cross-entropy). The assumption \textbf{(A2)} requires the inputs to be bounded but this is typically satisfied (e.g., images are all in $[0, 1]^d$). The assumption \textbf{(A3)} corresponds to easier problems (e.g., strongly convex ones) for which global convergence can be proven. 
We have the following convergence result:  
\begin{theorem}
    \label{theorem:samsto}
    Assume~\textbf{(A1)} and \textbf{(A2)} for the iterates~(\ref{eq:sambatch}). Then for any number of iterations $T\geq 0$, batch size $b$, and step sizes $\gamma_t=\frac{1}{\sqrt{T}\beta}$ and $\rho_t=\frac{1}{T^{1/4}\beta}$, we have: 
    \begin{align*}
        \frac{1}{T}\E \left[ \sum_{t=0}^{T-1}\| \nabla L (w_t) \|^2 \right] \leq \frac{4\beta}{\sqrt{T}} (L(w_0)-L_*) +\frac{8 \sigma^2}{b \sqrt{T}},
    \end{align*}
    In addition, under \textbf{(A3)}, with step sizes $\gamma_t=\min\{ \frac{8t+4}{3\mu(t+1)^2}, \frac{1}{2\beta}\}$ and $\rho_t=\sqrt{\gamma_t/\beta}$:
    \begin{align*}
        \E \left[L(w_T)\right] -L_* \leq \frac{3\beta^2( L(w_0) -L_*)}{\mu^2T^2} + \frac{22\beta\sigma^2}{ \mu^2 bT}.
    \end{align*}
\end{theorem}
We provide the proof in App.~\ref{sec:app_theory_convergencesto} and make several remarks: %
\begin{itemize}
\item We recover the rates of SGD with the usual condition on the step size $\gamma_t$~\citep{ghadimi2013stochastic, karimi2016linear}. 
\item The ascent step size $\rho_t$, however, has to be $O(\sqrt{\gamma_t})$ to ensure convergence, i.e., it tolerates a slower decrease than $\gamma_t$. This finding is aligned with the observation that the ascent step size should not be decreased as drastically as the descent step size when training neural networks (see Fig.~\ref{fig:test_err_over_decreasing_rho} in App.~\ref{subsec:app_decreasing_rho}). 
\item On the technical side, the proof relies on the bound $\langle  \nabla L( w_t + \eta \nabla L (w_t) ), \nabla L (w_t) \rangle \geq (1- \eta \beta ) \|  \nabla L (w_t) \|^2 $ which shows that SAM-step is well aligned with the gradient step (see Lemma~\ref{lem:gradsamstos} in App.~\ref{sec:app_theory_convergencesto}).
\end{itemize}

\subsection{Convergence of SAM for Deep Networks}
\label{subsec:opt_deep_networks}
Here we relate the convergence analysis to empirical observations for deep learning tasks.

\myparagraph{Both ERM and SAM converge for deep networks.}
\begin{figure*}[t]
    \centering
    \begin{subfigure}[t]{.37\textwidth}
        \caption{\hspace{8mm}\textbf{ResNet-18 on CIFAR-10}}
        \vspace{-2mm}
        \includegraphics[width=1.0\columnwidth]{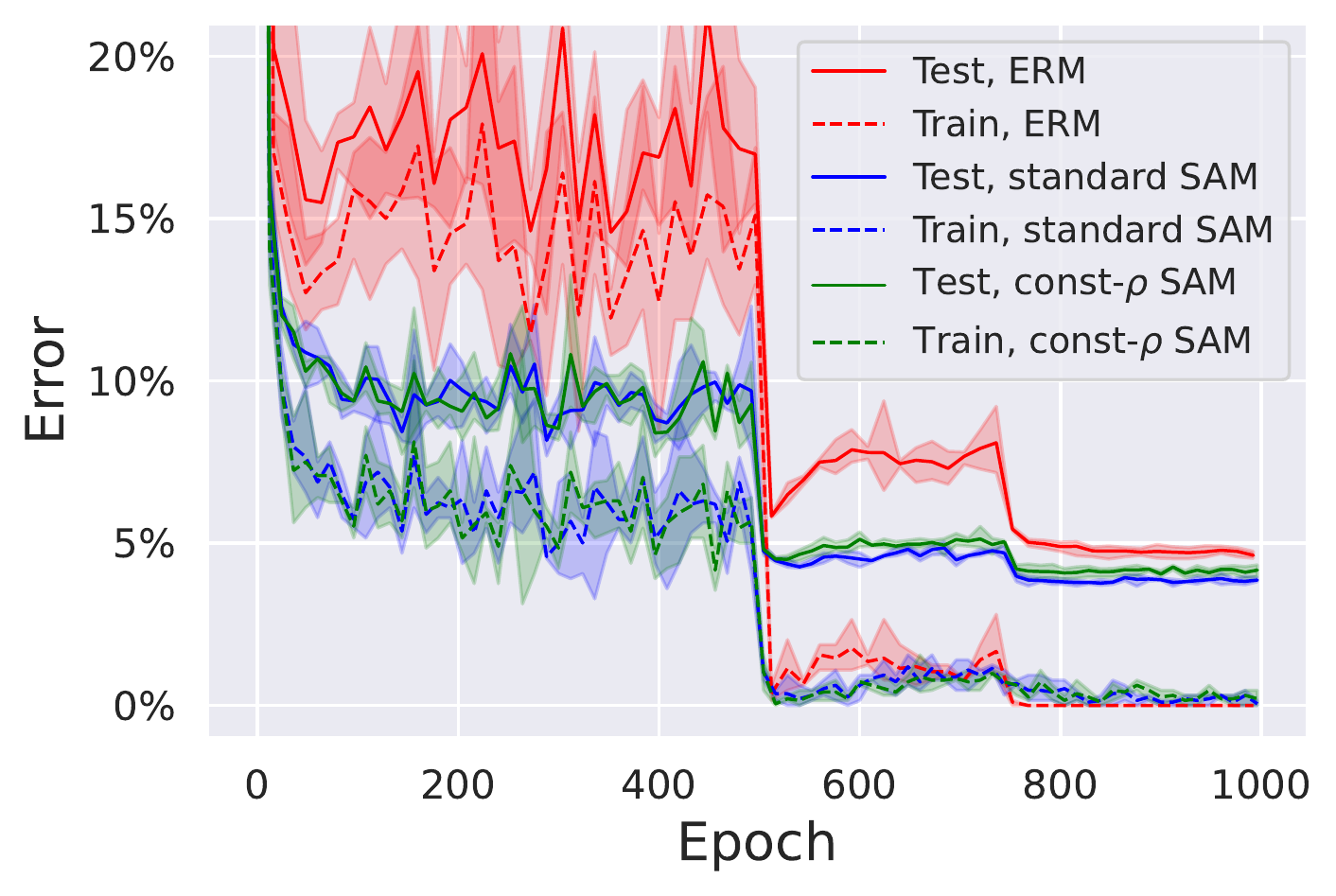}
    \end{subfigure}
    \quad \quad 
    \begin{subfigure}[t]{.37\textwidth}
        \caption{\hspace{8mm}\textbf{ResNet-34 on CIFAR-100}}
        \vspace{-2mm}
        \includegraphics[width=1.0\columnwidth]{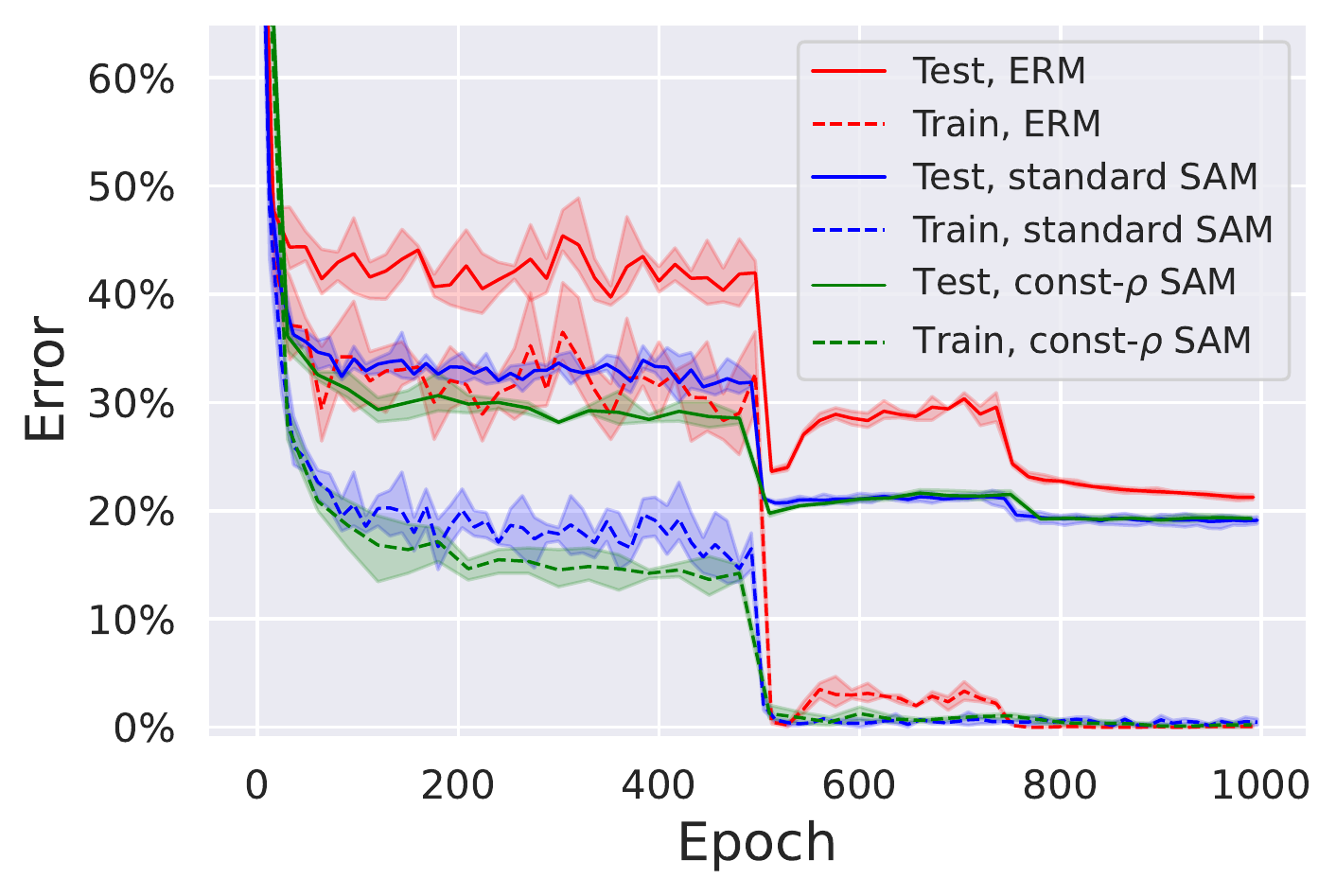}
    \end{subfigure}
    \vspace{-3mm}
    \caption{Training and test error of ERM, standard SAM, and SAM with a constant step size $\rho$ (i.e., without gradient normalization) over epochs. We can see that both ERM and SAM converge to zero training error and the gradient normalization is not crucial for SAM.
    }
    \label{fig:err_sam_vs_erm}
\end{figure*}
We compare the behavior of ERM and SAM by training a ResNet-18 on CIFAR-10 and CIFAR-100 for 1000 epochs 
(see App.~\ref{sec:app_exp_details} for experimental details) and plot the results over epochs in Fig.~\ref{fig:err_sam_vs_erm}. 
We observe that not only the ERM model but also the model trained with SAM fits all the training points and converges to a \textit{nearly zero training loss}: $0.0013\pm0.00002$ for ERM vs $0.0034\pm0.0004$ for SAM on CIFAR-10. 
However, the SAM model has significantly better generalization performance due to its implicit bias: $4.75\%\pm0.14\%$ vs. $3.94\%\pm0.09\%$ test error. %
Moreover, we observe no noticeable overfitting throughout training: the best and last model differ by at most 0.1\% test error for both methods. Finally, we note that the behavior of ERM vs. SAM on CIFAR-100 is qualitatively similar. %

\myparagraph{Performance of SAM with constant step sizes $\rho_t$.} Our convergence proof in Sec.~\ref{subsec:convergence_analysis} for non-convex objectives relies on constant step sizes $\rho_t$. However, the standard SAM algorithm as introduced in \citet{foret2021sharpnessaware} uses step sizes $\rho_t$ inversely proportional to the gradient norm. Thus, one can wonder if such step sizes are important for achieving better convergence or generalization.
Fig.~\ref{fig:err_sam_vs_erm} shows that on CIFAR-10 and CIFAR-100, both methods converge to zero training error at a similar speed. 
Moreover, they achieve similar improvements in terms of generalization: $3.94\%\pm0.09\%$ test error for standard SAM vs. $4.15\%\pm0.16\%$ for SAM with constant $\rho_t$ on CIFAR-10. For CIFAR-100, the test error matches almost exactly: $19.22\%\pm0.38\%$ vs. $19.30\%\pm0.38\%$. 
We also note that the optimal $\rho$ differs for both formulations: $\rho_t=0.2/\norm{\nabla}_2$ with normalization vs. $\rho_t=0.3$ without normalization, so simply removing the gradient normalization without doing a new grid search over $\rho_t$ can lead to suboptimal results.

\myparagraph{Is it always beneficial for SAM to converge to zero loss?}
\begin{figure}[t]
    \vspace{-2mm}
    \centering
    \begin{subfigure}[t]{.235\textwidth}
        \caption{\hspace{4mm}\textbf{ResNet-18 on CIFAR-10}}
        \vspace{-2mm}
        \includegraphics[width=1.0\columnwidth]{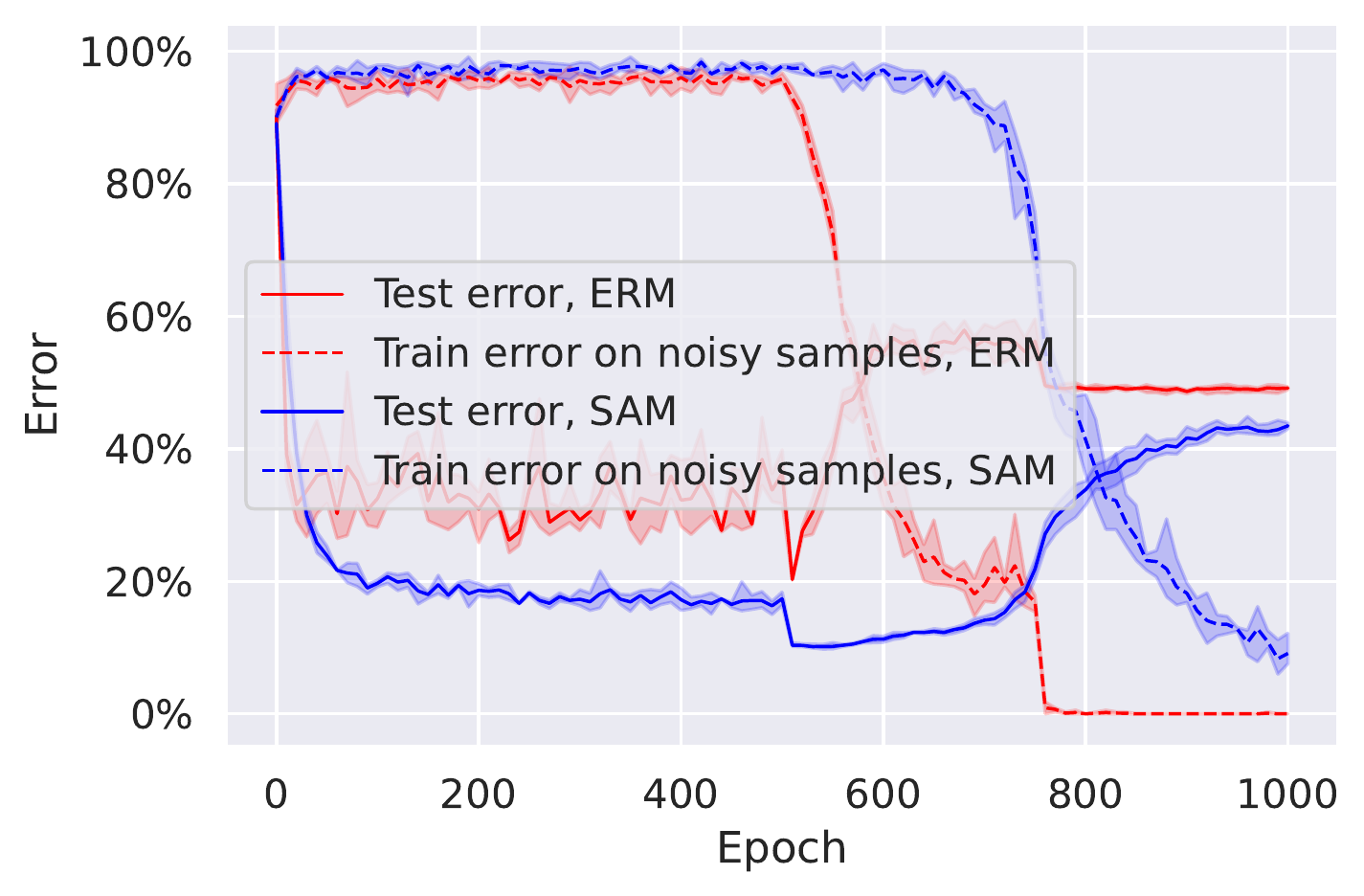}
    \end{subfigure}
    \begin{subfigure}[t]{.235\textwidth}
        \caption{\hspace{4mm}\textbf{ResNet-34 on CIFAR-100}}
        \vspace{-2mm}
        \includegraphics[width=1.0\columnwidth]{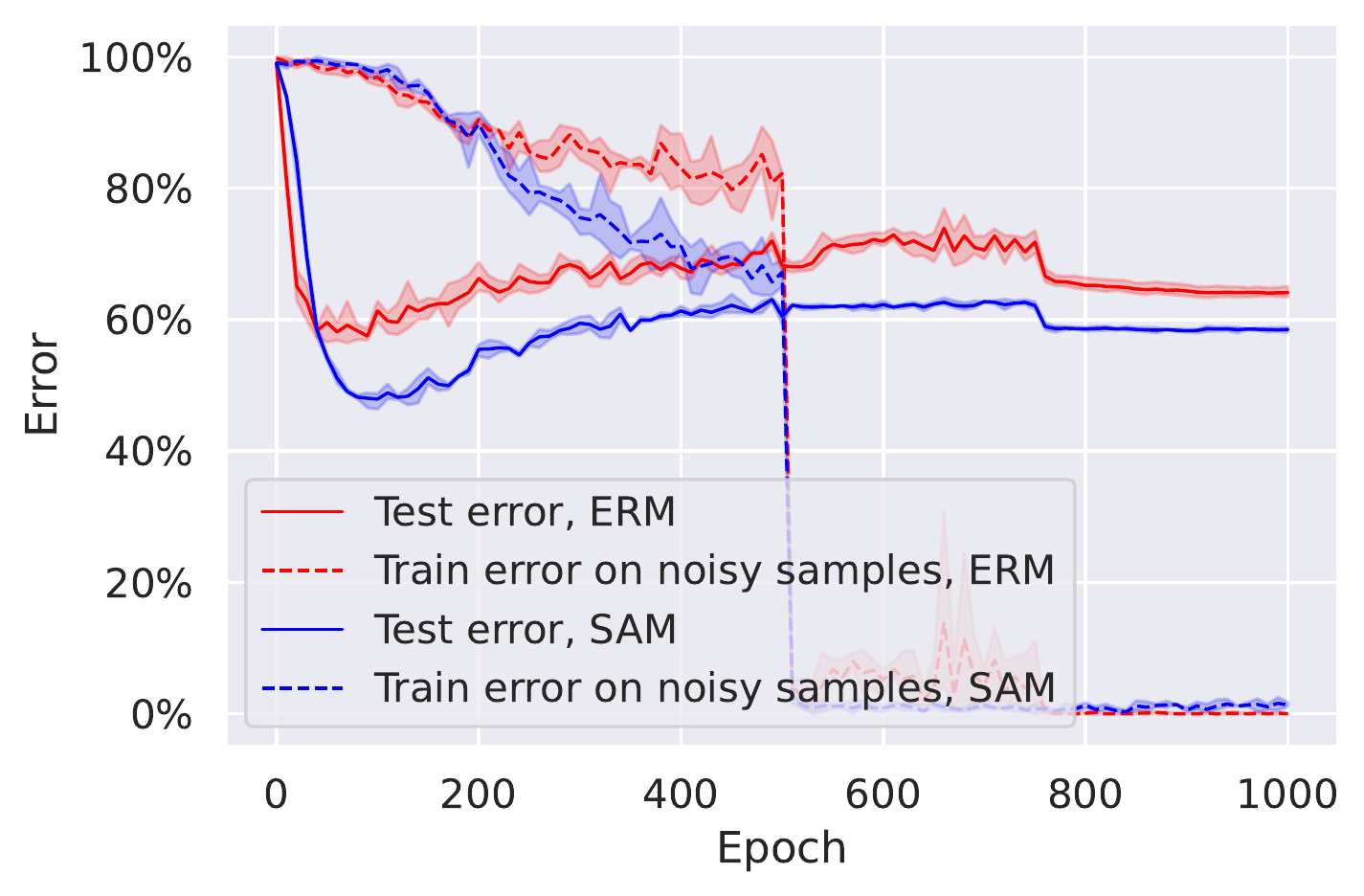}
    \end{subfigure}
    \vspace{-1.5mm}
    \caption{Error rates of ERM and SAM over epochs on CIFAR-10 and CIFAR-100 \textit{with 60\% label noise}. We see that the test error increases when the models fit the noisy samples.}
    \label{fig:label_noise_plots}
    \vspace{-2mm}
\end{figure}
Here we consider the setting of uniform label noise, i.e., when a fraction of the training labels is changed to random labels and kept fixed throughout the training. 
This setting differs from the standard noiseless case (typical for many vision datasets such as CIFAR-10) as converging to nearly zero training loss is harmful for ERM and leads to substantial overfitting.
Thus, one could assume that the beneficial effect of SAM in this setting can come from preventing convergence and avoiding fitting the label noise.
We plot test error and training error on noisy samples for a ResNet-18 trained on CIFAR-10 and CIFAR-100 with 60\% label noise in Fig.~\ref{fig:label_noise_plots}. We see that SAM noticeably improves generalization over ERM, although later in training SAM also starts to fit the noisy points which is in agreement with the convergence analysis. 
In App.~\ref{subsec:app_const_rho_sam_ln}, we confirm the same findings for SAM with constant $\rho_t$. 
Thus, SAM also requires early stopping either explicitly via a validation set or implicitly via restricting the number of training epochs as done, e.g., in \citet{foret2021sharpnessaware}. 
Interestingly, this experiment also suggests that the beneficial effect of SAM is observed not only close to a minimum but also along the whole optimization trajectory. 
Overall, we conclude that SAM can easily overfit and its convergence in terms of the training loss can be a negative feature for datasets with noisy labels.

\section{Conclusions} 
We showed why the existing justifications for the success of $m$-SAM based on generalization bounds and the idea of convergence to flat minima are incomplete. We hypothesized that there exists some other quantity which is responsible for the improved generalization of $m$-SAM which is implicitly minimized. We analyzed the implicit bias of $1$-SAM and $n$-SAM for diagonal linear networks showing that the implicit quantity which is minimized is related to the $\l_1$-norm of the resulting linear predictor, and it is stronger for $1$-SAM than for $n$-SAM. We further studied the properties of the implicit bias on non-linear networks empirically where we showed that fine-tuning an ERM model with SAM can lead to significant generalization improvements. Finally, we provided convergence results of SAM for non-convex objectives when used with stochastic gradient which we confirmed empirically for deep networks and discussed its relation to the generalization behavior of SAM.

\bibliographystyle{abbrvnat}
\bibliography{literature.bib}

\clearpage

\appendix
\onecolumn

\begin{center}
    \ \\
	\Large\textbf{Appendix}
\end{center}

\section*{Organization of the appendix}
The appendix is organized as follows:
\begin{itemize}
 \item Sec.~\ref{sec:app_fullbatch}:   implementations in the full-batch setting of $1$-SAM and $n$-SAM.
	\item Sec.~\ref{sec:app_implicit_bias}: proofs related to the implicit bias of $1$-SAM and $n$-SAM.
	\item Sec.~\ref{sec:app_theory_convergence}: proofs related to the convergence of different variants of SAM.
	\item Sec.~\ref{sec:app_exp_details}: experimental details for the experiments with deep networks and linear models.
	\item Sec.~\ref{sec:app_additional_experiments}: additional experiments complementary to the experiments in the main part.
\end{itemize}

\section{Implementations of the SAM Algorithm in the Full-Batch Setting}
\label{sec:app_fullbatch}
We define here the implementations of the $m$-SAM algorithm in the \textit{full-batch} setting for the two extreme values of $m$ we consider, i.e., $m=1$ and $m=n$. They correspond to the following objectives:
\begin{align}
    \text{\textbf{$n$-SAM}:} \ \minop_{\wv \in \R^{|w|}} \maxop_{\norm{\deltav}_2 \leq \rho}  \frac{1}{n}\sum_{i=1}^{n} \l_i(\wv + \deltav), \quad \quad
    \text{\textbf{$1$-SAM}:} \ \minop_{\wv \in \R^{|w|}} \frac{1}{n}\sum_{i=1}^{n} \maxop_{\norm{\deltav}_2 \leq \rho} \l_i(\wv + \deltav).
\end{align} 

The update rule of the SAM algorithm for these objectives amounts to a variant of gradient descent with step size $\gamma_t$ where the gradients are taken at intermediate points $w_{t+1/2}^i$, i.e., $w_{t+1} = w_t - \frac{\gamma_t}{n}\sum_{i=1}^n \nabla \ell_i (w_{t+1/2}^i)$. 
The updates, however, differ in how the points $w_{t+1/2}^i$ are computed since they approximately maximize different functions with inner step sizes $\rho_t$: 
\begin{align} \label{eq:sam_full_batch_update_rules}
    \text{\textbf{$n$-SAM:}} \ \ w_{t+1/2}^i = w_t + \frac{\rho_t}{n} \sum_{j=1}^n \nabla \ell_j(w_t), \quad \quad
    \text{\textbf{$1$-SAM:}} \ \ w_{t+1/2}^i = w_t + \rho_t \nabla \ell_i(w_t).
\end{align}

To make the SAM algorithm practical, \citet{foret2021sharpnessaware} propose to combine SAM with stochastic gradients which corresponds to the $m$-SAM algorithm defined in Eq.~(\ref{eq:sambatch}) in the main part.

\section{Theoretical Analysis of the Implicit Bias for Diagonal Linear Networks}
\label{sec:app_implicit_bias}

To understand why $m$-SAM is generalizing better than ERM, we consider the simpler problem of noiseless regression with $2$-layer diagonal linear network for which we can precisely characterize the implicit bias of different optimization algorithms.

\myparagraph{Optimization algorithms.}
We consider minimizing the training loss $L(w)$ using the following optimization algorithms: 
\begin{itemize}
\item Gradient descent with an infinitesimally small step size, i.e., the gradient flow limit:
    \begin{align}\label{eq:gf}
        \dot w_t = -\nabla L(w_t). 
    \end{align}
\item The $n$-SAM algorithm from Eq.~(\ref{eq:sam_full_batch_update_rules}) taken with an infinitesimally small outer step size and inner step size $\rho\geq0$:
    \begin{align}\label{eq:flowmaxsum}
        \dot w_t = -\nabla L(w_t+\rho \nabla L(w_t)).
    \end{align}
\item The $1$-SAM algorithm from Eq.~(\ref{eq:sam_full_batch_update_rules}) taken with an infinitesimally small outer step size and inner step size $\rho\geq0$: 
    \begin{align}\label{eq:flowsummax}
        \dot w_t &= -\frac{1}{n} \sum_{i=1}^n \nabla \ell_i(w_t+\rho \nabla \ell_i(w_t)).
    \end{align}
\end{itemize}

\myparagraph{Previous work: implicit bias of the gradient flow.} We first define the function $\phi_\alpha$ for $\alpha\in\mathbb R^d$ which will be very useful to precisely characterize the implicit bias of the optimization algorithms we consider: 
\begin{equation}\label{eq:qalpha}
    \phi_\alpha  (\beta) =  \sum_{i=1}^d \alpha_i^2 q(\beta_i/\alpha_i^2) \text{ where }  q(z)=\int_0^z \arcsinh(u/2)du = 2-\sqrt{4+z^2}+z\arcsinh(z/2).
\end{equation}
Following \citet{woodworth2020kernel}, one can show the following result for the gradient flow dynamics in Eq.~(\ref{eq:gf}). 
\begin{theorem}[Theorem 1 of \citet{woodworth2020kernel}]
    If the  solution $\beta_\infty$ of the gradient flow~(\ref{eq:gf}) started from $w_+ = w_- = \alpha \in \R_{>0}^d$ for the squared parameter problem in Eq.~(\ref{eq:quadpro}) satisfies $X\beta_{\infty}=y$, then 
    \begin{equation}\label{eq:optpro}
        \beta_{\infty} = \arg \min_{\beta\in \mathbb R^d} \phi_\alpha (\beta) \ \ \text{s.t.} \ \ X\beta=y, 
    \end{equation}
    where $\phi_\alpha $ is defined in Eq.~(\ref{eq:qalpha}).
\end{theorem}
It is worth noting that the implicit regularizer $\phi_\alpha$ interpolates between the $\ell_1$ and $\ell_2$ norms \citep[see][Theorem 2]{woodworth2020kernel}. Therefore the scale of the initialization determines the implicit bias of the gradient flow. The algorithm, started from $\alpha$, converges to the minimum $\ell_1$-norm interpolator for small $\alpha$ and to the minimum $\ell_2$-norm interpolator for large $\alpha$.
The proof follows from (a) the KKT condition for the optimization problem~(\ref{eq:optpro}): $\nabla \phi_\alpha (w) = X^\top \nu$ for a Lagrange multiplier $\nu$ and (b) the closed form solution obtained by integrating the gradient flow, $w= b(X^\top \nu)$ for some function $b$ and some vector $\nu$. Identifying $\nabla \phi_\alpha (w) = b^{-1}(w)$ leads to the solution.
Considering the same proof technique, we now derive the implicit bias for the $n$-SAM and $1$-SAM algorithms.

\subsection{Implicit Bias of the $\bm{n}$-SAM Algorithm.}
We start from characterizing the implicit bias of the $n$-SAM dynamics~(\ref{eq:flowmaxsum}) in the following theorem using the function $\phi_\alpha$ defined in Eq.~(\ref{eq:qalpha}). We will also make use of this notation: a parameter vector $w = \left[\begin{smallmatrix}w_+\\w_-\end{smallmatrix}\right] \in \R^{2 d}$, a concatenation of matrices $\tilde X = [X\ \ -X] \in \R^{n \times 2d}$ and a residual vector $r(t) = \tilde Xw(t)^2-y$.
\begin{theorem}\label{theorem:biasmaxsum}
    If the solution $\beta_\infty$ of the $n$-SAM gradient flow~(\ref{eq:flowmaxsum}) started from $w_+ = w_- = \alpha \in \R_{>0}^d$ for the squared parameter problem in Eq.~(\ref{eq:quadpro}) satisfies $X\beta_{\infty}=y$, then 
    \begin{equation*}
        \beta_{\infty} = \arg \min_{\beta} \phi_{\alpha_{\text{n-SAM}}} (\beta) \ \ \text{s.t.} \ \ X\beta=y, 
    \end{equation*}
    where $\alpha_{\text{n-SAM}} = \alpha \odot \exp\left(- \frac{2\rho}{n^2} \int_0^\infty (X^\top  r_s)^2ds +O(\rho^2)\right)$.
\end{theorem}
We note that for a small enough $\rho$, the implicit bias parameter $\alpha_{\text{n-SAM}}$ is smaller than $\alpha$. The scale of the vector $\frac{1}{n^2}\int_0^\infty (X^\top  r_s)^2ds$ which influences the implicit bias effect is related to the loss integral $\frac{d}{n}\int_0^\infty L(w(s))ds$ since $\| r_s\|^2 =n L(w(s))$  (see intuition in Eq.~(\ref{eq:lossmaxsum})). Thereby the speed of convergence of the loss controls the magnitude of the biasing effect. However in the case of $n$-SAM, as explained in Sec.~\ref{subsec:comparison_1sam_nsam}, this effect is typically negligible because of the extra prefactor $\frac{d}{n}$ and this implementation behaves similarly as ERM as shown in the experiments in Sec.~\ref{sec:impmain}.
\begin{proof} We follow the proof technique of \citet{woodworth2020kernel}. We denote the intermediate step of $n$-SAM as $w_{sam}(t)= w(t)+\rho \nabla L(w(t))$ and the residual of $w_{sam}(t)$ as $r_{sam}(t) = \tilde X w_{sam}(t)^2-y$. We start from deriving the equation satisfied by the flow
\begin{align*}
    \dot w(t) 
    & = -\nabla L(w_{sam}(t)) \\ 
    & = -\frac{1}{n} \tilde X^\top r_{sam}(t) \odot w_{sam}(t) \\
    & = -\frac{1}{n} \tilde X^\top r_{sam}(t) \odot \left( w(t) + \frac{\rho}{n} \left(\tilde X^\top r(t)\right) \odot w(t) \right).
\end{align*}
Now we can directly integrate this ODE to obtain an expression for $w(t)$: 
\begin{align*}
    w(t) = w(0)\odot \exp \left( -\frac{1}{n}\tilde X ^\top \int_0^t r_{sam}(s)ds \right)\odot \exp \left(-\frac{\rho}{n^2}  \int_0^t \left( \tilde X ^\top r_{sam}(s)\right) \odot \left(\tilde X ^\top r(s)\right) ds \right). 
\end{align*}
Using that the flow is initialized at $w(0)=\alpha$ and the definition of $\beta(t)$ yields to
\begin{align*}
    \beta(t) & = w_+(t)^2-w_-(t)^2 \\
    & = \alpha^2 \odot \exp \left( -\frac{2}{n} X^\top \int_0^t r_{sam}(s)ds \right) \odot \exp \left( -\frac{2\rho}{n^2}  \int_0^t \left(X ^\top r_{sam}(s)\right) \odot \left(X ^\top r(s)\right) ds \right) \\
    & \quad \quad- \alpha^2 \odot \exp \left( \frac{2}{n} X ^\top \int_0^t r_{sam}(s)ds \right) \odot \exp \left( -\frac{2\rho}{n^2}  \int_0^t \left(X ^\top r_{sam}(s)\right) \odot \left(X ^\top r(s)\right) ds \right) \\
    & = 2\alpha^2 \odot \exp \left( -\frac{2\rho}{n^2}  \int_0^t \left(X ^\top r_{sam}(s)\right) \odot \left(X ^\top r(s)\right) ds \right) \odot \sinh \left( -\frac{2}{n} X ^\top \int_0^t r_{sam}(s)ds \right). 
\end{align*}
Recall that we are assuming that $\beta_\infty$ is a global minimum of the loss, i.e., $X \beta_\infty = y$. Thus, $\beta_\infty$ has to simultaneously satisfy 
\begin{align*}
    X \beta_\infty = y \text{\ \ and\ \ } \beta_\infty = b_{\alpha_\text{n-SAM}}(X^\top\nu),
\end{align*}
where $b_\alpha(z) = 2 \alpha^2 \odot \sinh(z) $ and  $\nu = -\frac{2}{n} \int_0^\infty r_{sam}(s)ds$, and 
\begin{equation}\label{eq:alpha_n_sam}
 \alpha_\text{n-SAM} = \alpha \odot \exp \left(-\frac{2\rho}{n^2}  \int_0^\infty (X^\top r_{sam}(s)) \odot (X^\top r(s)) ds \right).
\end{equation} 
Next we combine the flow expression $b_{\alpha_\text{n-SAM}}^{-1}(\beta_\infty) = X^\top\nu$ with a KKT condition $\nabla \phi_\alpha (w) = X^\top \nu$ and get that
\begin{equation*}
    \nabla \phi_\alpha (\beta) = b_\alpha^{-1}(\beta) = \arcsinh \left( \frac{1}{2\alpha^2} \odot \beta \right).
\end{equation*}
Integration of this equation leads to $\phi_\alpha(\beta) = \sum_{i=1}^d \alpha_i^2 q(\beta_i/\alpha_i^2)$ where $q(z)=\int_0^z \arcsinh(u/2)du = 2-\sqrt{4+z^2}+z\arcsinh(z/2)$, i.e., exactly the potential function defined in Eq.~(\ref{eq:qalpha}). Thus, we conclude that $\beta_\infty$ satisfies the KKT conditions $X \beta_\infty = y$ and $\nabla \phi_\alpha (\beta_\infty) = X^\top\nu$ for the minimum norm interpolator problem:
\begin{equation*}
    \min_{\beta \in \R^d} \phi_{\alpha} (\beta) \ \ \ \text{s.t.} \ \ \ X\beta=y,
\end{equation*}
which proves the first part of the result. 

Now to get the expression for $\alpha_{\text{n-SAM}}$, we apply the definition of $r_{sam}(s)$ and obtain 
\begin{align*}
    r_{sam}(t) 
    &= \tilde X w_{sam}(t)^2 -y \\
    &= \tilde X \left( w(t) + \frac{\rho}{n} \left(\tilde X^\top r(t)\right) \odot w(t) \right)^2 - y  \\
    &= r(t) + \frac{2\rho}{n} \tilde X \left(\tilde X^\top r(t)\right) \odot w(t) + \frac{\rho^2}{n^2} \tilde X \left(\tilde X^\top r(t)\right)^2 \odot w(t)^2  \\
    &= r(t) + \frac{2\rho}{n} X \left(X^\top r(t)\right)\odot (w_+(t) + w_-(t)) + \frac{\rho^2}{n^2} X \left(X^\top r(t)\right)^2\odot (w_+(t)^2 + w_-(t)^2).
\end{align*}
Thus we conclude that $X^\top r_{sam}(t) = X^\top r(t) + O(\rho)$ which we plug in Eq.~(\ref{eq:alpha_n_sam}) to obtain the second part of the theorem:
\begin{align*}
    \alpha_{\text{n-SAM}} = \alpha \odot \exp \left(- \frac{2\rho}{n^2} \int_0^\infty (X^\top  r_s)^2 ds +O(\rho^2)\right).
\end{align*}
\end{proof}

\subsection{Implicit Bias of the $\bm{1}$-SAM Algorithm}
We characterize similarly the implicit bias of the $1$-SAM dynamics~(\ref{eq:flowsummax}) in the following theorem using the function $\phi_\alpha$ defined in Eq.~(\ref{eq:qalpha}).
\begin{theorem} \label{theorem:biassummax}
    If the solution $\beta_\infty$ of the $1$-SAM gradient flow~(\ref{eq:flowsummax}) started from $w_+ = w_- = \alpha \in \R_{>0}^d$ for the squared parameter problem in Eq.~(\ref{eq:quadpro}) satisfies $X\beta_{\infty}=y$, then 
    \begin{equation*}
        \beta_{\infty} = \arg \min_{\beta} \phi_{\alpha_{\text{1-SAM}}} (\beta) \ \ \text{s.t.} \ \ X\beta=y, 
    \end{equation*}
    where $\alpha_{\text{1-SAM}} = \alpha \odot \exp\left(- \frac{8\rho}{n} \int_0^\infty  \sum_{i=1}^n x_i^2 (x_i^\top \beta(s)-y_i)^2 ds + O(\rho^2)\right)$.
    
    In addition, assume that there exist $R,B\geq 0$ such that almost surely (1) the inputs are bounded $\|x\|_2 \leq R$ and (2) the trajectory of the flow is bounded $\|\beta(t)\|_2 \leq B$ for all $t\geq 0$. Then for all $\rho \leq \frac{1}{4 R^2 \sqrt{ B(B+\|\beta_*\|_2)}}$, we have that $\alpha_{\text{1-SAM},i} \leq \alpha_i$ for $i \in \{1, \dots, d\}$.
\end{theorem}
\begin{proof}
The proof follows the same lines as the proof of Theorem~\ref{theorem:biasmaxsum}. 
We denote a concatenation of positive and negative copies of the $i$-th training example as $\tilde x_i = \left[\begin{smallmatrix}x_i\\-x_i\end{smallmatrix}\right]\in \R^{2d}$, the intermediate step of $1$-SAM based on the $i$-th training example as $w_{sam}^{(i)}(t) \in \R^d$, the residuals of $w(t)$ and $w_{sam}^{(i)}(t)$ on the $i$-th training example as $r_i(t) = \tilde x_i^\top w(t)^2-y_i$ and $r_{sam,i}(t) = \tilde x_i^\top w_{sam}^{(i)}(t)^2-y_i$.
Then we have that the dynamics of the flow~(\ref{eq:flowsummax}) satisfies
\begin{align*}
    \dot w(t) 
    & = -\frac{1}{n} \sum_{i=1}^n \nabla \ell_i(w_{sam}^{(i)}(t)) \\
    & = -\frac{1}{n} \sum_{i=1}^n r_{sam,i}(t) \cdot \tilde x_i \odot w_{sam}^{(i)}(t) \\
    & = -\frac{1}{n} \sum_{i=1}^n r_{sam,i}(t) \cdot \tilde x_i \odot w(t) \odot \left(\mathbf{1} + 4\rho r_{i}(t) \tilde x_i \right).
\end{align*}
Integration of this ODE leads to
\begin{align*}
     w(t) 
     =  w(0)\odot \exp \left( -\frac{1}{n} \tilde X^\top \int_0^t r_{sam}(s)ds\right) \odot \exp \left( -\frac{4\rho}{n} \sum_{i=1}^n  \tilde x_i^2  \int_0^t  r_{sam,i}(s) r_{i}(s) ds\right).
\end{align*}
The rest of the proof is similar to the one of Theorem~\ref{theorem:biasmaxsum} and we directly obtain that 
\begin{equation}
     \alpha_\text{1-SAM} = \alpha \odot \exp \left(-\frac{8\rho}{n} \sum_{i=1}^n  \tilde x_i^2  \int_0^t  r_{sam,i}(s) r_{i}(s) ds \right).
\end{equation}
Using the definition of $r_{sam,i}(t)$ we have 
\begin{align*}
    r_{sam,i}(t) &= \tilde x_i^\top w_{sam}(t)^2-y_i \\
    & = \tilde x_i^\top w(t)^2\odot \left(\mathbf{1} + 4\rho r_i(t)\tilde x_i\right)^2 -y_i \\
    & = \tilde x_i^\top w(t)^2\odot \left(\mathbf{1} + 8\rho r_i(t) \tilde x_i + 16\rho^2 r_i(t)^{2} \tilde x_i^2\right) -y_i \\
    & = r_i(t) + 8 \rho r_i(t) \left(w_+(t)^2 + w_-(t)^2\right)^\top x_i^2 + 16 \rho^2 r_i(t)^2 \left(w_+(t)^2 - w_-(t)^2\right)^\top x_i^3 \\
    & = r_i(t) + 8 \rho r_i(t) \left(w_+(t)^2 + w_-(t)^2\right)^\top x_i^2 + 16 \rho^2 r_i(t)^2 \beta(t)^\top x_i^3
\end{align*}
And therefore 
\begin{align}\label{eq:main_part_of_exp_one_sam}
    x_i^2  r_{sam,i}(t)r_{i}(t)= r_{i}(t)^2 x_i^2 \odot \left( \mathbf{1}  + 8 \rho \left(w_+(t)^2+w_-(t)^2\right)^\top x_i^2 + 16 \rho^2 r_{i}(t) \beta(t)^\top x_i^3 \right)
\end{align}
This leads to the result stated in the theorem
\begin{equation}
     \alpha_{\text{1-SAM}} = \alpha \odot \exp\left(- \frac{8\rho}{n} \int_0^\infty  \sum_{i=1}^n x_i^2 (x_i^\top \beta(s)-y_i)^2 ds + O(\rho^2)\right).
\end{equation}

Additionally, from Eq.~(\ref{eq:main_part_of_exp_one_sam}) we can conclude that having $\rho$ such that $1 + 16 \rho^2 r_{i}(t) \beta(t)^\top x_i^3 \geq 0$ is sufficient to guarantee that $\alpha_{\text{1-SAM},i} \leq \alpha_i$ for every $i$. We can use Cauchy-Schwarz inequality twice to upper bound $|r_{i}(t) \beta(t)^\top x_i^3|$:
\begin{align*}
    |r_{i}(t) \beta(t)^\top x_i^3| 
    &= |x_i^\top (\beta - \beta_*) \beta(t)^\top x_i^3| \leq \| x_i\|_2 \|\beta(t) - \beta_*\|_2 \|\beta(t)\|_2 \|x_i^3\|_2 \\
    &\leq \| x_i\|_2^4 (\|\beta(t)\|_2+\|\beta_*\|_2)\|\beta(t)\|_2 \leq R^4 (B + \|\beta_*\|_2) B
\end{align*}
Thus, we have that $\rho^2 r_{i}(t) \beta(t)^\top x_i^3 \geq -\rho^2 R^4 (B + \|\beta_*\|_2) B \geq -\frac{1}{16}$ which leads to the upper bound stated in the theorem $\rho \leq \frac{1}{4 R^2 \sqrt{ B(B+\|\beta_*\|_2)}}$.
\end{proof}

\subsection{Comparison between $\bm{1}$-SAM and $\bm{n}$-SAM}
\label{subsec:comparison_1sam_nsam}

\myparagraph{Theoretical comparison.}
We wish to compare the two leading terms of the exponents in $\alpha_{\text{n-SAM}}$ and $\alpha_{\text{1-SAM}}$:
\begin{align*}
    I_{\text{n-SAM}}(t)= \frac{1}{n^2} \left(X^\top r(t)\right)^2 = \frac{1}{n^2}\left(\sum_{i=1}^n x_i r_i(t) \right)^2 
    \text{\ \ and\ \ }
    I_{\text{1-SAM}}(t)=\frac{1}{n} \sum_{i=1}^n  x_i^2  r_{i}(t)^2,
\end{align*} 
and relate them to the loss values at $w(t)$.

We first note that using Cauchy-Schwarz inequality can directly imply that $I_{\text{1-SAM},i}(t) \geq I_{\text{n-SAM},i}(t)$. However, we aim at obtaining a more quantitative result, even though the following derivations will be informal. Comparing the $\ell_1$-norms of $I_{\text{n-SAM}}(t)$ and $I_{\text{1-SAM}}(t)$ amounts to compare the following two quantities: 
\begin{align*}
    \|I_{\text{n-SAM}}(t)\|_1 &= (w(t)-w_*)^\top \left[\frac{1}{n}\sum_{i=1}^n x_i x_i^\top\right]^2 (w(t)-w_*), \\
    \|I_{\text{1-SAM}}(t)\|_1 &= (w(t)-w_*)^\top \left[\frac{1}{n}\sum_{i=1}^n \|x_i\|_2^2 x_i x_i^\top\right] (w(t)-w_*).
\end{align*}
We can compare the typical operator norms of the random matrices that define the two quadratic forms. If we assume that $x_i \sim \mathcal{N}(0, I_d)$, then following the Bai-Yin’s law, the operator norm of a Wishart matrix is with high probability $\| \frac{1}{n}\sum_{i=1}^n x_ix_i^\top  \|_{op} \approx \frac{d}{n}$ and that with high probability, the squared norm of a Gaussian vector is  $\| x_i\|_2^2 \approx d $. Therefore we obtain that 
\begin{align*}
    \left\| \left[\frac{1}{n}\sum_{i=1}^n x_i x_i^\top\right]^2 \right\|_{op} &=  \left\| \frac{1}{n}\sum_{i=1}^n x_i x_i^\top \right\|_{op}^2 \approx \frac{d^2}{n^2}, \\
    \left\| \frac{1}{n}\sum_{i=1}^n \|x_i\|^2x_ix_i^\top \right\|_{op} &\approx d \left\|  \frac{1}{n}\sum_{i=1}^n x_ix_i^\top \right\|_{op} \approx \frac{d^2}{n}.
\end{align*}
Therefore in the overparametrized regime ($d>>n$), we typically have that $\frac{\left\| I_{\text{1-SAM}}(t)\right\|_1}{\left\|I_{\text{n-SAM}}(t)\right\|_1} \approx n$ and the biasing effect of $1$-SAM would tend to be $O(n)$ times better compared to $n$-SAM.

However, this first insight only enables to compare $I_{\text{n-SAM}}(t)$ and $I_{\text{1-SAM}}(t)$. It is not informative on the intrinsic biasing effect of $n$-SAM and $1$-SAM. With this aim, we would like to relate the  quantities $I_{\text{n-SAM}}(t)$ and $I_{\text{1-SAM}}(t)$ to the loss function evaluated in $w(t)$. 
Using the concentration of Wishart matrices, i.e., $\frac{1}{d}[ X X^\top] \approx I$ for large dimension $d$, we have with high probability
\begin{align}
    \|I_{\text{n-SAM}}(t)\|_1& =\frac{1}{n^2 } (w(t)-w_*)^\top X^\top X X^\top  X (w(t)-w_*) \nonumber \\
    &= \frac{d}{n^2 } (w(t)-w_*)^\top  X^\top\frac{1}{d}[ X X^\top]  X (w(t)-w_*) \nonumber \\
    &\approx \frac{d}{n} (w(t)-w_*)^\top \frac{1}{n} [X^\top   X] (w(t)-w_*)\nonumber \\
    &= \frac{d}{n}L(w(t)) \label{eq:losssummax}.
\end{align}
And using the concentration of Gaussian vectors, we also have that
\begin{align}
    \|I_{\text{1-SAM}}(t)\|_1 &= (w(t)-w_*)^\top \frac{1}{n}\sum_{i=1}^n \|x_i\|^2x_ix_i^\top (w(t)-w_*)\nonumber \\
    &\approx d (w(t)-w_*)^\top \frac{1}{n}\sum_{i=1}^n x_ix_i^\top (w(t)-w_*) \nonumber \\
    &= d L(w(t)). \label{eq:lossmaxsum}
\end{align}
These approximations provide some intuition on why the biasing effect of $1$-SAM and $n$-SAM can be related to the integral of the loss and that typically the difference is on the order of $n$. We let a formal derivation of these results as future work. 
\begin{figure}[t]
    \centering
    \begin{minipage}[c]{.37\linewidth}
        \includegraphics[clip, trim=8mm 65mm 20mm 70mm, width=\linewidth]{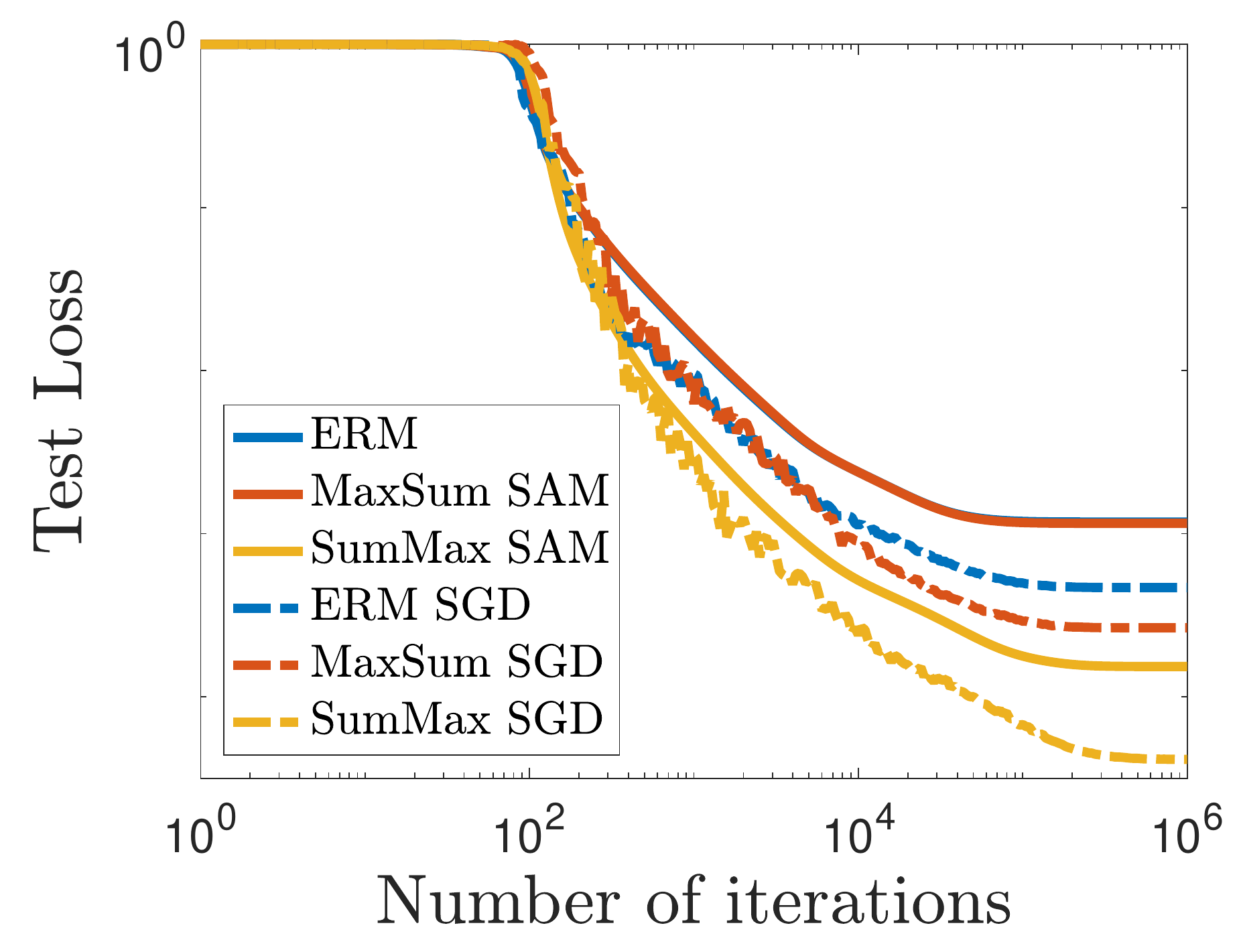}
    \end{minipage}
    \quad \quad
    \begin{minipage}[c]{.37\linewidth}
        \includegraphics[clip, trim=8mm 65mm 20mm 70mm, width=\linewidth]{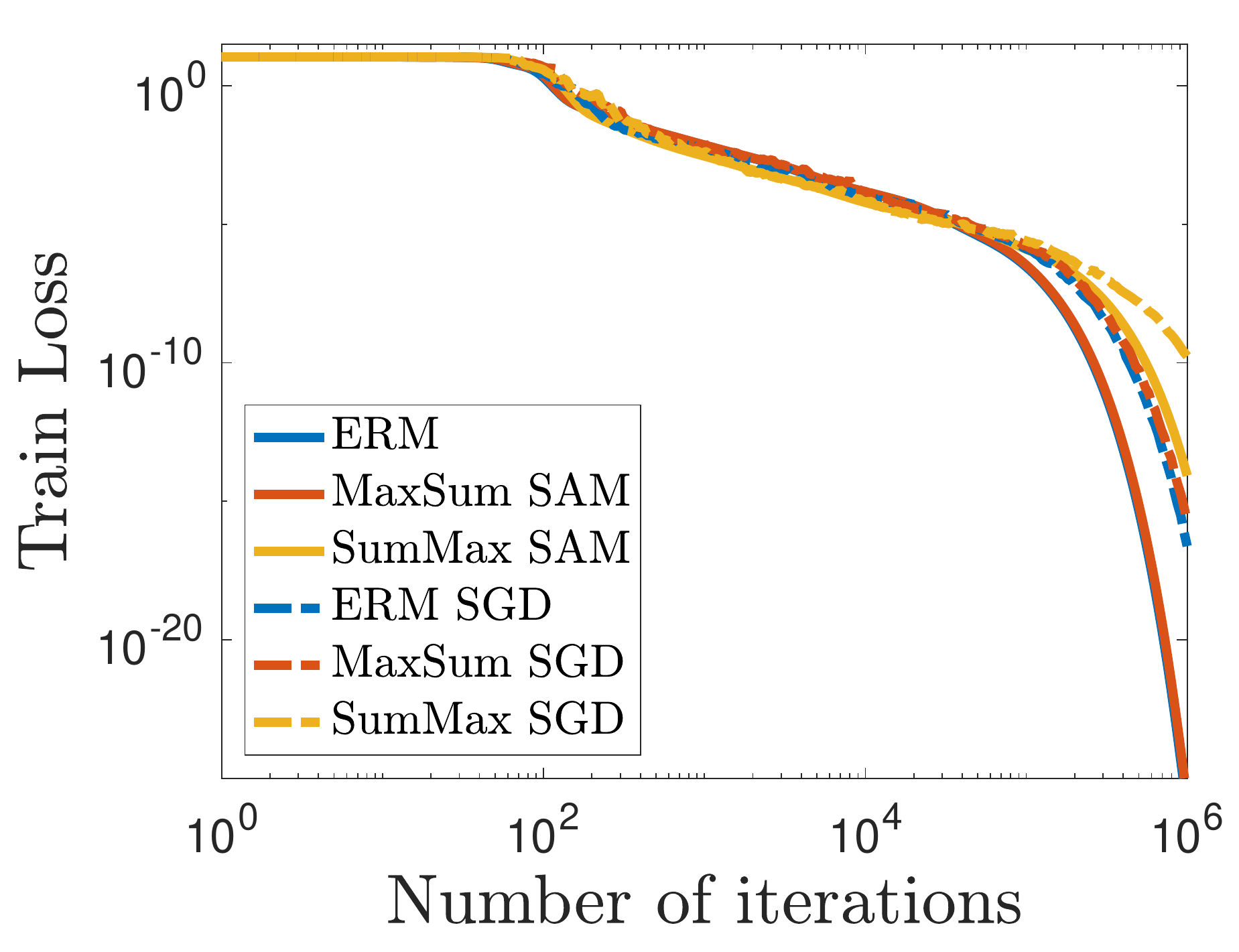}
    \end{minipage}
    \caption{Implicit bias of SAM on a sparse regression problem using a diagonal linear network with $d=30$, $n=20$, $x_i\sim \mathcal{N}(0,I)$, $\kappa=\|\beta_*\|_0=3$, $y_i=x_i^\top \beta_*$. All methods are initialized at $\alpha= 0.01$ and used with step size $\gamma = 1/d$ and $\rho=1/d$. We can see that $1$-SAM (SumMax) SGD converges to a solution which generalizes better (left plot) and enjoys a different implicit bias from the other methods. At the same time, all algorithms converge to a global minimum of $f$ at linear rate (right plot). The convergence speed is inversely proportional to the biasing effect.
    }
   \label{fig:implicit_bias_sto}
\end{figure}

\myparagraph{Experiments with stochastic ERM, $\bm{n}$-SAM, $\bm{1}$-SAM.} 
We provide an additional experiment to investigate the performance of stochastic implementations of the ERM, $n$-SAM and $1$-SAM. As explained by~\citet{pesme2021implicit}, we observe in Fig.~\ref{fig:implicit_bias_sto} that the stochastic implementations enjoy a better implicit bias than their deterministic counterparts.  We note that the fact that small batch versions generalize better than full batch version is commonly observed in practice for deep networks~\citet{keskar2016large}. We let the characterization of the implicit bias of these stochastic implementations as future works.

\myparagraph{Grid search over $\bm{\rho}$ for $\bm{n}$-SAM vs. $\bm{1}$-SAM.}
\begin{figure}[t]
    \centering
    \includegraphics[width=.38\columnwidth]{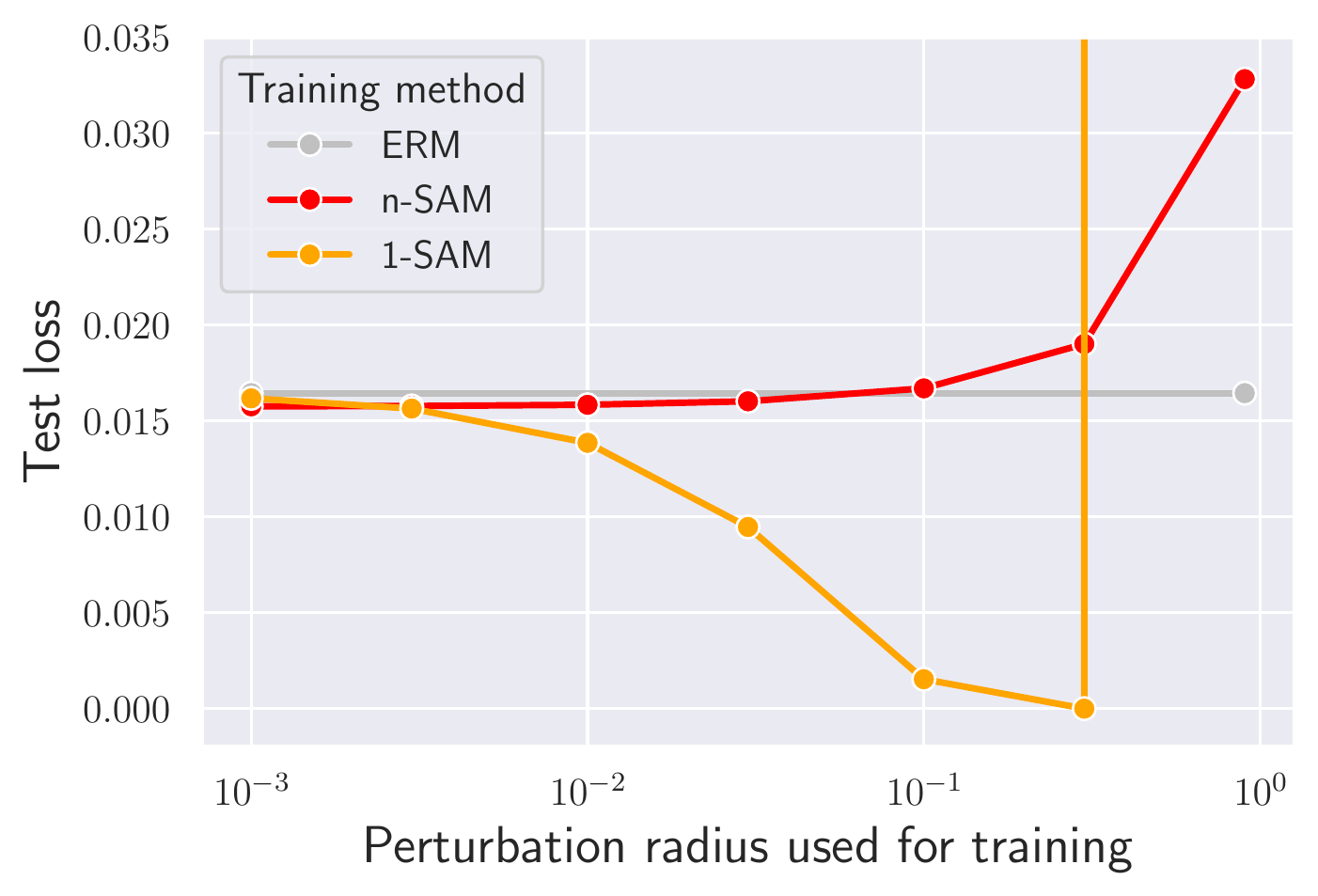}
    \caption{A grid search over $\rho$ for full-batch $n$-SAM vs. $1$-SAM ($\alpha=0.05$, $\gamma=15/d$ for all methods). We can see that even with the optimal $\rho$, $n$-SAM generalizes much worse than $1$-SAM which is coherent with our deep learning experiments in Fig.~\ref{fig:gen_bound_is_not_predictive}.}
    \label{fig:diag_net_gs_nsam_1sam}
\end{figure}
We note that for Fig.~\ref{fig:implicit_bias_sam} and Fig.~\ref{fig:implicit_bias_sto}, we used a fixed $\rho$ which was the same for both $n$-SAM and $1$-SAM. Tuning $\rho$ for each method separately can help to achieve a better test loss for both methods as shown in Fig.~\ref{fig:diag_net_gs_nsam_1sam}. We can see that $1$-SAM still significantly outperforms ERM and $n$-SAM for the optimally chosen radius $\rho$ and that $n$-SAM leads only to marginal improvements.

\myparagraph{Connection to the ERM~$\rightarrow$~SAM and SAM~$\rightarrow$~ERM experiment.}
\begin{figure}[t]
\centering
 \begin{minipage}[c]{.33\linewidth}
    \includegraphics[width=\linewidth]{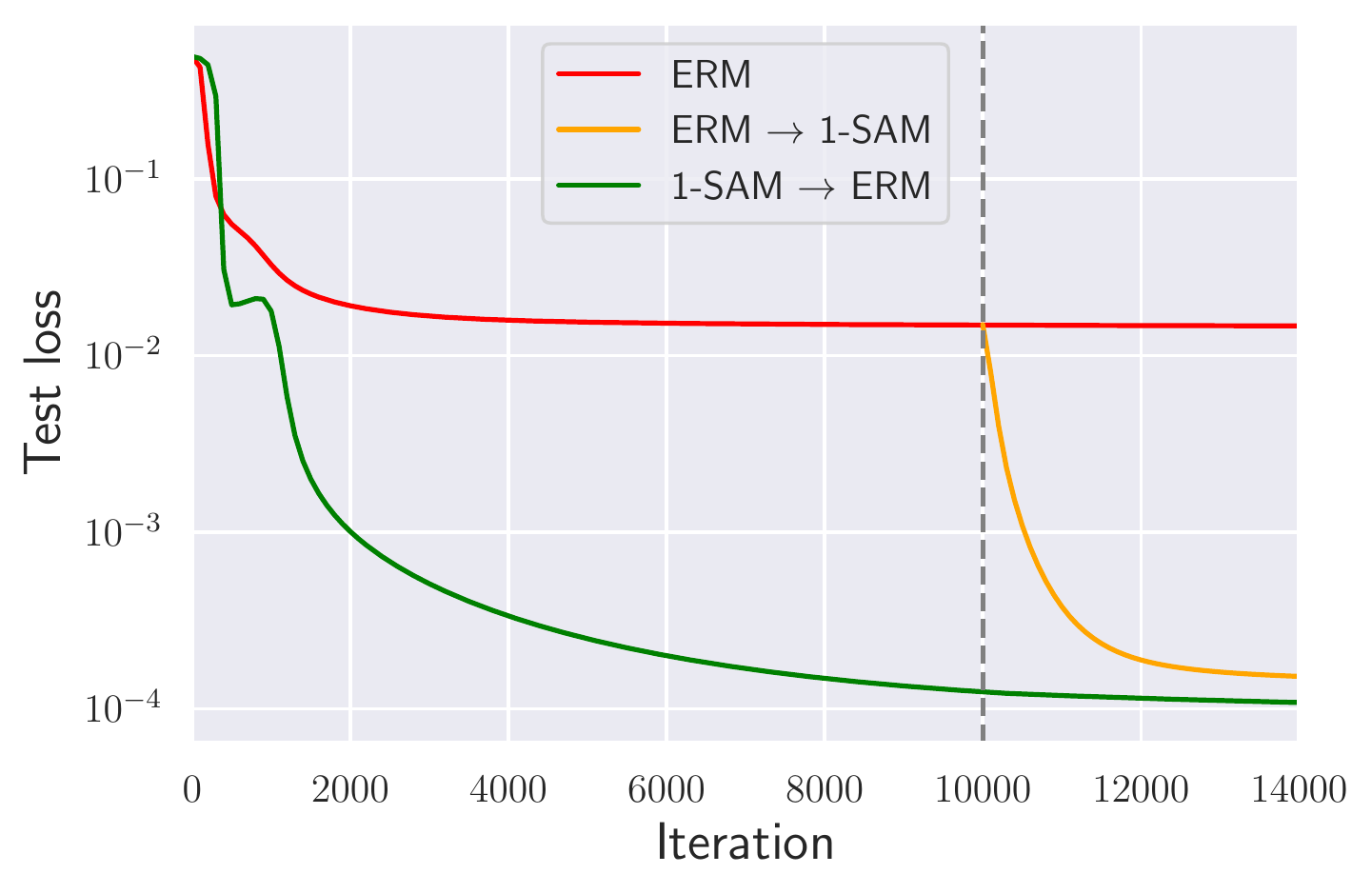}
    \subcaption{(a) Test loss over epochs}
 \end{minipage}
 \begin{minipage}[c]{.33\linewidth}
    \includegraphics[width=\linewidth]{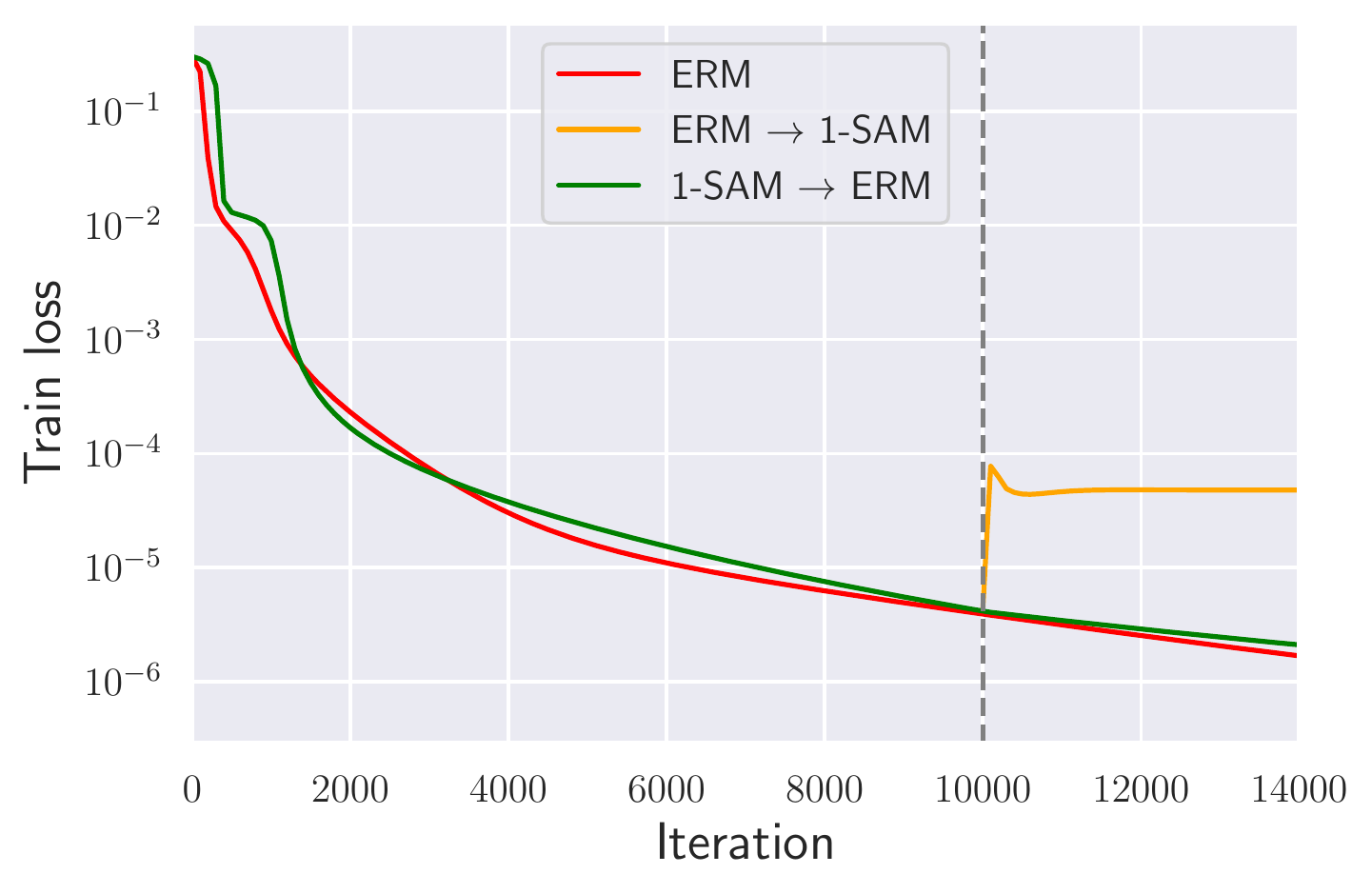}
    \subcaption{(b) Training loss over epochs}
 \end{minipage}
 \begin{minipage}[c]{.3175\linewidth}
    \includegraphics[width=\linewidth]{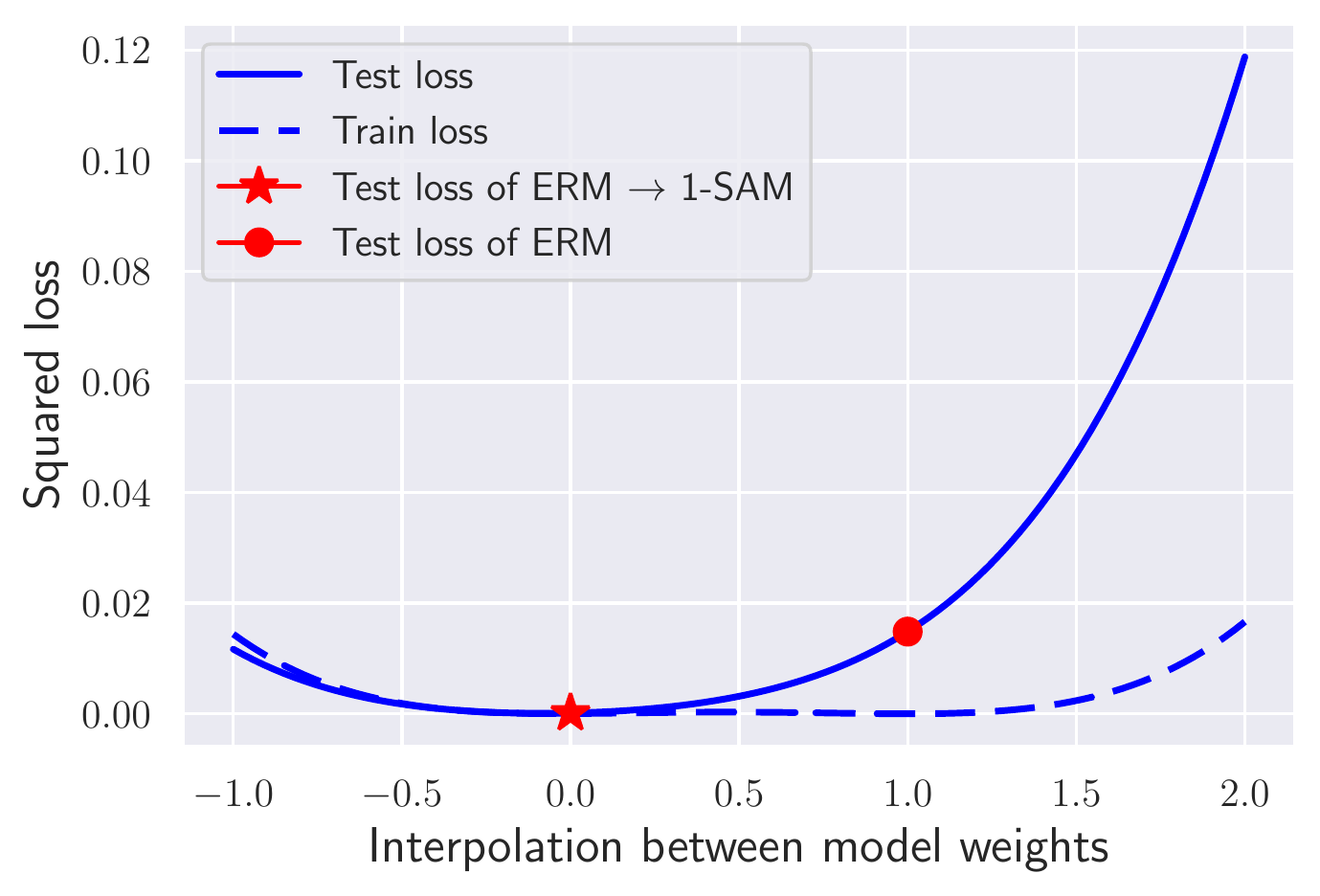}
    \subcaption{(c) Loss interpolations}
 \end{minipage}
  \caption{Test loss (a) and training loss (b) for full-batch ERM compared to ERM $\rightarrow$ $1$-SAM and $1$-SAM $\rightarrow$ ERM on a diagonal linear network where we switch between the methods after 10k iterations. We can see that $1$-SAM can quickly escape the worse-generalizing minimum found by ERM. Moreover, in (c) we show loss interpolations between ERM $\rightarrow$ $1$-SAM and ERM that show that they are linearly connected and situated in the same basin.}
   \label{fig:diag_net_erm_sam_sam_erm}
\end{figure}
Here we provide further details on the connection between Theorem~\ref{theorem:implicit_bias_main} and the empirical results in Fig.~\ref{fig:erm_sam_finetuning}.
First of all, we show in Fig.~\ref{fig:diag_net_erm_sam_sam_erm} that the same observations as we observed for deep networks also hold on a diagonal linear network. In this experiment, we used the initialization scale $\alpha=0.05$, $\rho_{\text{1-SAM}}=0.175$, and $\rho_{\text{GD} \rightarrow \text{1-SAM}}=10.0$. We note that we had to take $\rho_{\text{GD} \rightarrow \text{1-SAM}}$ significantly larger than $\rho_{\text{1-SAM}}$ since after running GD, we are already near a global minimum where the gradients (which are also used for the ascent step of SAM) are very small so we need to increase the inner step size $\rho_{\text{GD} \rightarrow \text{1-SAM}}$ to observe a difference. In addition, a loss interpolation between $w_{\text{GD} \rightarrow \text{1-SAM}}$ and $w_{\text{GD}}$ reveals linear connectivity between the two found minima suggesting that both minima are situated in the same asymmetric basin, similarly to what we observed for deep networks in Fig.~\ref{fig:deep_nets_loss_surfaces}.

First we note that Theorem~\ref{theorem:implicit_bias_main} can be trivially adapted to the case where SAM is used with varying inner step size $\rho_t$, and would therefore show that for diagonal linear networks, the key quantity determining the magnitude of the implicit bias for SAM is the integral of the step size $\rho_s$ times the loss over the optimization trajectory $w(s)$, i.e., $\|\Delta_{\text{1-SAM-}\rho_s}\|_1\approx d \int_0^\infty \rho_s L(w(s))ds$ which leads to a smaller value in the exponent $\alpha_{\text{1-SAM-}\rho_s} = \alpha e^{- \rho \Delta_{\text{1-SAM-}\rho_s} +O(\rho^2)}$, thus decreasing the effective $\alpha$ and biasing the flow to a sparser solution.

In the case of ERM~$\rightarrow$~$1$-SAM, it amounts to consider a step size $\rho_s = 0$ if $s<t$ and $\rho_s = \rho$ after the switch. Therefore the integral is taken only over the last epochs, and  $\|\Delta_{\text{1-SAM-t-}\infty}\|_1\approx d \int_t^\infty L(w(s))ds$ where the integral starts at the time step $t$. The resulting $\|\Delta_{\text{1-SAM-t-}\infty}\|_1$ is smaller than $\|\Delta_{\text{1-SAM}}\|_1$ but it can still be sufficient (especially, when using a higher $\rho$ as we do for Fig.~\ref{fig:diag_net_erm_sam_sam_erm}) to improve the biasing effect so that it leads to noticeable improvements in generalization.

At the same time, for $1$-SAM~$\rightarrow$~ERM, which amounts to consider a step size $\rho_s = \rho$ if $s<t$ and $\rho_s =0$ after the switch, the integral is already large enough due to the first 1000 epochs with SAM, leading to a term $\|\Delta_{\text{1-SAM-0-t}}\|_1\approx d \int_0^t L(w(s))ds$ and switching back to ERM preserves the implicit bias due to a low enough effective $\alpha$. This explains why switching back to ERM does not negatively affect generalization of the model.

\section{Convergence of the SAM Algorithm}
\label{sec:app_theory_convergence}
In this section we provide proofs of convergence for SAM. We consider first the full-batch SAM algorithm and then its stochastic version.

\subsection{Convergence of Full-Batch $\bm{n}$-SAM}
\label{sec:appoptdet}
We first consider the full-batch version of SAM, i.e., the following update rule:
\begin{align}\label{eq:samgd}
     w_{t+1} = w_t - \gamma \nabla L\left( w_t + \rho \nabla L(w_t)\right).
\end{align}
We note that this update rule is reminiscent of the extra-gradient algorithm \citep{korpelevich1977extragradient} but with an \textit{ascent} in the inner step instead of a \textit{descent}. Moreover, this update rule can also be seen as a realization of the general extrapolated gradient descent framework suggested in \citet{lin2020extrapolation}. However, taking an \textit{ascent} step for extrapolation is not discussed there, and the convergence properties of the update rule from Eq.~(\ref{eq:samgd}), to the best of our knowledge, have not been proven.

\myparagraph{Summary of the convergence results.}
Let us first recall the definition of $\beta$-smoothness which we will use in our proofs. 
\begin{description}
\item[\textbf{(A2')}] ($\beta$-smoothness). \textit{There exists  $\beta\!>\!0$ such that $ 
\| \nabla L (w) -\nabla L (v) \| \leq \beta \| w-v \|$ for all $w,v  \in \mathbb R ^d$.}
\end{description}
When the function $L$ is $\beta$-smooth, convergence to stationary points can be obtained. 
\begin{theorem}
\label{theorem:samdet}
Assume~\textbf{(A2')}. For any  $\gamma<1/\beta$ and $\rho< 1/\beta$,  the iterates~(\ref{eq:samgd}) satisfy for all $T\geq 0$:
\begin{align*}
    \frac{1}{T} \sum_{t=0}^{T-1}\| \nabla L (w_t) \|^2 \leq \frac{2}{\gamma (1-\rho \beta) T} (L(w_0)-L_*),
\end{align*}
If, in addition, the function $L$ satisfies \textbf{(A3)}, then:
\begin{align*}
    L(w_T) -L_* \leq \left( 1-\frac{\gamma (1-\rho \beta) \mu}{2}\right)^T( L(w_0) -L_*).
\end{align*}
\end{theorem}
We can make the following remarks:
\begin{itemize}
    \item We recover the rates of gradient descent but with constants increasing with the ascent step size $\rho$. 
    \item The condition $\rho<1/\beta$ is necessary since the point $w+1/\beta\nabla L(w)$ can be a local maximum of $L$. Such $w$ would be a fixed point of the algorithm without being a stationary point of $L$. 
    \item The proof crucially relies on the bound $\langle  \nabla L( w_t + \rho \nabla L (w_t) ), \nabla L(w_t) \rangle \geq (1- \rho \beta ) \|  \nabla L (w_t) \|^2 $ which shows that the SAM step is well-aligned with the gradient step (see Lemma~\ref{lem:gradsam}) and on a descent inequality similar to the classical one for gradient descent (see Lemma~\ref{lem:smoothdecrease}). 
    \item For non-convex functions, full details are provided in Theorem~\ref{theorem:samdetnoncvx}. When the function satisfies in addition Polyak-Lojasiewicz inequality, a stronger result holds which is stated in Theorem~\ref{theorem:samdetnoncvxpl}.
    \item For convex functions, $\langle  \nabla L( w_t + \rho \nabla L(w_t) ), \nabla L (w_t) \rangle \geq \|  \nabla L (w_t) \|^2 $  and convergence holds for any step size $\rho$ given that $\gamma \rho$ is small enough. Details are provided in Theorem~\ref{theorem:convex}. %
\end{itemize}

\myparagraph{Auxiliary Lemmas.}
The following lemma shows that the SAM update is well correlated with the gradient $\nabla L (w)$ and will be a cornerstone to our proof.
\begin{lemma} \label{lem:gradsam}
Let $L$ be a differentiable function and $w\in\R^d$. We have the following bound for any $\rho\geq 0$:
\begin{align*}
  \langle  \nabla L( w + \rho \nabla L (w) ), \nabla L (w) \rangle \geq (1+\alpha \rho ) \|  \nabla L(w) \|^2 \, 
\text{ where }  
    \alpha = \begin{cases}
-\beta  &\text{ if $L$ is $\beta$-smooth},\\
0 &\text{ if $L$ is convex}\\
\mu &\text{ if $L$ is $\mu$-strongly convex}.
\end{cases}
\end{align*}
\end{lemma}

\begin{proof}
We simply add and subtract a term $\|\nabla L(w) \|^2$ in order to make use of classical inequalities bounding $\langle \nabla L (w_1)-\nabla L (w_2),w_1-w_2\rangle$ by $\| w_1-w_2\|^2$ for smooth or convex functions and $w_1,w_2\in\R^d$. 
\begin{align*}
   \langle  \nabla L( w + \rho \nabla L (w) ), \nabla L (w) \rangle & =   \langle  \nabla L( w + \rho \nabla L (w) )-  \nabla L (w) , \nabla L (w) \rangle ) +   \|  \nabla L (w) \|^2\\
   &= 1/\rho  \langle  \nabla L( w + \rho \nabla L (w) )-  \nabla L (w) , \rho \nabla L (w) \rangle +  \|  \nabla L (w) \|^2 \\
   &\geq (1+\alpha \rho)  \|  \nabla L (w) \|^2,
\end{align*}
where the last inequality is using that 
\begin{align*}
  \langle \nabla L(w_1) -\nabla L(w_2) , w_1-w_2\rangle \geq \alpha \| w_2-w_1\|^2 , \text{ where } 
  \alpha = \begin{cases}
-\beta  &\text{ if $L$ is $\beta$-smooth},\\
0 &\text{ if $L$ is convex}\\
\mu &\text{ if $L$ is $\mu$-strongly convex}.
\end{cases}
\end{align*}
\end{proof}

The next lemma shows that the decrease of function values of the SAM algorithm defined in Eq.~(\ref{eq:samgd}) can be controlled similarly as in the case of gradient descent~\citep{Nes04}. 
\begin{lemma}
    \label{lem:smoothdecrease}
    Assume $\textbf{(A2')}$. For any $\gamma \leq 1/\beta$, the iterates~(\ref{eq:samgd}) satisfy for all $t \geq 0$: 
    \begin{align*}
    L(w_{t+1}) \leq L(w_t) -\gamma (1 - \rho \beta) \Big(1 -\frac{\gamma \beta }{2}(1-\rho \beta)\Big)\| \nabla L ( w_t) \|^2.  
    \end{align*}
    If, in addition, the function $L$ satisfies $\textbf{(A3)}$ with potentially $\mu=0$, then for all $\gamma,\rho\geq 0$ such that $\gamma \beta (2-\rho \beta) \leq 2$, we have
    \begin{align*}
    L(w_{t+1}) \leq L(w_t) - \gamma \Big(1 - \frac{\gamma \beta}{2}  + \rho \mu \big(1 - \gamma \beta - \frac{\gamma \rho \beta^2}{2} \big)   \Big)\| \nabla L ( w_t) \|^2.
    \end{align*}
\end{lemma}
We note that the constraints on the step size are different depending on the assumptions on the function $L$. In the non-convex case, $\rho$ has to be smaller than $1/\beta$, whereas in the convex case, it has to be smaller than $2/\beta$. 
\begin{proof}
Let us define by $w_{t+1/2} = w_t + \rho \nabla L (w_t)$ the SAM ascent step.
Using the smoothness of the function $L$ (Assumption \textbf{(A2')}),  we obtain
\begin{align*}
L(w_{t+1}) \leq L(w_t) - \gamma \langle \nabla L( w_{t+1/2}  ) , \nabla L (w_t) \rangle + \frac{\gamma^2 \beta}{2} \| \nabla L (w_{t+1/2}) \|^2.
\end{align*}
The main trick is to use the binomial squares
\begin{align*}
   \| \nabla L (w_{t+1/2}) \|^2  =   - \| \nabla L ( w_t) \|^2  + \| \nabla L (w_{t+1/2}) - \nabla L(w_t) \|^2  +2 \langle \nabla L( w_{t+1/2}  ) , \nabla L (w_t) \rangle,
\end{align*}
to bound 
\begin{align*}
L(w_{t+1}) &\leq L(w_t) - \gamma \langle \nabla L( w_{t+1/2}  ) , \nabla L (w_t) \rangle + \frac{\gamma^2 \beta}{2} \| \nabla L (w_{t+1/2}) \|^2 
\\
           &= L(w_t) -  \frac{\gamma^2 \beta}{2} \| \nabla L ( w_t) \|^2  +\frac{\gamma^2 \beta}{2} \| \nabla L (w_{t+1/2}) - \nabla L(w_t) \|^2  - \gamma (1 - \gamma \beta )\langle \nabla L( w_{t+1/2}  ) , \nabla L (w_t) \rangle \\
           & \leq L(w_t) -\gamma [1 - \rho  \beta -\frac{\gamma \beta }{2}(1-\rho  \beta)^2]\| \nabla L ( w_t) \|^2,  
\end{align*}
where we have used Lemma~\ref{lem:gradsam} and that $\| \nabla L (w_{t+1/2}) - \nabla L(w_t) \|^2 \leq \beta^2 \| w_{t+1/2} - w_t \|^2\leq \beta^2 \rho^2 \| \nabla L (w_t)\|^2$.

If, in addition, the function $L$ is convex then we can use its co-coercivity~\citep{Nes04} to bound $\| \nabla L (w_{t+1/2}) - \nabla L(w_t) \|^2\leq \beta  \langle \nabla L (w_{t+1/2}) - \nabla L(w_t), w_{t+1/2} -w_t \rangle $ and obtain a tighter bound:
\begin{align*}
L(w_{t+1}) &\leq L(w_t) - \gamma \langle \nabla L( w_{t+1/2}  ) , \nabla L (w_t) \rangle + \frac{\gamma^2 \beta}{2} \| \nabla L (w_{t+1/2}) \|^2  \\
           &= L(w_t) -  \frac{\gamma^2 \beta}{2} \| \nabla L ( w_t) \|^2  +\frac{\gamma^2 \beta}{2} \| \nabla L (w_{t+1/2}) - \nabla L(w_t) \|^2  - \gamma (1 - \gamma \beta )\langle \nabla L( w_{t+1/2}  ) , \nabla L (w_t) \rangle \\
           & \leq L(w_t) -  \gamma (1 - \frac{\gamma \beta}{2} )\| \nabla L ( w_t) \|^2   - \gamma (1 - \gamma \beta - \frac{\gamma \rho  \beta^2}{2} ) \langle \nabla L( w_{t+1/2}  ) -\nabla L(w_t), \nabla L (w_t) \rangle \\ 
            & \leq L(w_t) -  \gamma (1 - \frac{\gamma \beta}{2}  + \rho \mu (1 - \gamma \beta - \frac{\gamma \rho  \beta^2}{2} )   )\| \nabla L ( w_t) \|^2,   
\end{align*}
where we have used Lemma~\ref{lem:gradsam}.
\end{proof}

\myparagraph{Convergence proofs.}
Using the previous Lemma~\ref{lem:smoothdecrease} recursively, we can bound the average gradient value of the iterates~(\ref{eq:samgd}) of SAM algorithm and ensure convergence to stationary points. 
\begin{theorem}
    \label{theorem:samdetnoncvx}
    Assume~\textbf{(A2')}. For any  $\gamma<1/\beta$ and $\rho< 1/\beta$,  the iterates~(\ref{eq:samgd}) satisfies for all $T\geq 0$:
    \begin{align*}
        \frac{1}{T}\sum_{t=0}^T \| \nabla L ( w_t) \|^2    \leq \frac{L(w_0)-L(w_T)}{T\gamma   (1 - \rho  \beta) [1 -\frac{\gamma \beta }{2}(1-\rho  \beta)] }. 
    \end{align*}
\end{theorem}
\begin{proof}
    Using the Lemma~\ref{lem:smoothdecrease} we obtain 
    \begin{align*}
       \gamma (1 - \rho \beta) \Big(1 -\frac{\gamma \beta }{2}(1-\rho \beta)\Big)\| \nabla L ( w_t) \|^2\leq L(w_t) - L(w_{t+1}) .  
    \end{align*}
    And summing these inequalities for $t=0,\dots,T-1$ yields 
    \begin{align*}
        \frac{1}{T}\sum_{t=0}^{T-1} \| \nabla L ( w_t) \|^2    \leq \frac{L(w_0)-L(w_T)}{T\gamma   (1 - \rho \beta) [1 -\frac{\gamma \beta }{2}(1-\rho \beta)] }. 
    \end{align*}
\end{proof}

When the function $L$ additionally satisfies a Polyak-Lojasiewicz condition \textbf{(A3)}, linear convergence of the function value to the minimum function value can be obtained. This is the object of the following theorem: 
\begin{theorem}
\label{theorem:samdetnoncvxpl}
Assume \textbf{(A2')} and \textbf{(A3)}. For any $\gamma<1/\beta$ and $\rho< 1/\beta$, the iterates~(\ref{eq:samgd}) satisfies for all $T\geq 0$:
\begin{align*}
    L(w_t) - L_* \leq \Big( 1-2 \gamma\mu  (1 - \rho \beta) \Big(1 -\frac{\gamma \beta }{2}(1-\rho \beta)\Big) \Big)^{t}  ( L(w_{0}) - L_* ).
\end{align*}
\end{theorem}
\begin{proof}
Using the Lemma~\ref{lem:smoothdecrease} and that the function $L$ is $\mu$ Polyak-Lojasiewicz (Assumption \textbf{(A3)}) we obtain 
\begin{align*}
  L(w_{t+1})  \leq L(w_t)  - 2\mu  \gamma (1 - \rho L) \Big(1 -\frac{\gamma \beta }{2}(1-\rho L)\Big)  (L(w_t)-L_*). 
\end{align*}
And subtracting the optimal value $L_*$ we get
\begin{align*}
    L(w_t) - L_* &\leq \Big( 1-2 \gamma\mu  (1 - \rho \beta) \Big(1 -\frac{\gamma \beta }{2}(1-\rho \beta)\Big) \Big) ( L(w_{t-1}) - L_* )\\
  & \leq \Big( 1-2 \gamma\mu  (1 - \rho \beta) \Big(1 -\frac{\gamma \beta }{2}(1-\rho \beta)\Big) \Big)^{t}  ( L(w_{0}) - L_* ).
\end{align*}
\end{proof}

When the function $L$ is convex, convergence of the average of the iterates can be proved. 
\begin{theorem}
\label{theorem:convex}
Assume~\textbf{(A2')} and $L$ convex. For any step sizes $\gamma$ and $\rho$ such that $\gamma \beta ( 1+\rho \beta) <2$, then the averaged $\bar w_T = \frac{1}{T}\sum_{t=0}^{T-1} w_t$ of the iterates~(\ref{eq:samgd}) satisfies for all $T\geq 0$:
\begin{align*}
   L(\bar w_T)-L_* \leq \frac{2\rho \beta +1}{ \gamma (2-\gamma \beta (1+\rho \beta )) T} \| w_0 - w_* \|^2,
\end{align*}
If, in addition, the function $L$ is $\mu$-strongly convex, then:
\begin{align*}
    \|w_T-w_*\|^2 \leq \big( 1-\gamma \mu ( 2-\gamma \beta (1+\rho \beta))  \big)^T( 2\rho+1) \|w_0-w_*\|^2.
\end{align*}
\end{theorem}
The proof is using a different astute Lyapunov function which works for the non-strongly convex case.
\begin{proof}
Let us define by $V_t= [L(w_t)-L(w_*)]+  \frac{1}{2\rho}\| w_t - w_*\|^2$ and by $w_{t+1/2} = w_t + \rho \nabla L (w_t)$ the SAM ascent step.
\begin{align*}
    V_{t+1}-V_t & \leq -\frac{\gamma }{\rho} \langle \nabla L (w_{t+1/2}), w_t-w_*\rangle - \gamma \langle \nabla L (w_{t+1/2}), \nabla L ( w_t)\rangle + \frac{\gamma^2}{2\rho}(1+\rho \beta)  \| \nabla L ( w_{t+1/2}) \|^2 \\
    & =-\frac{\gamma }{\rho} \langle \nabla L (w_{t+1/2}), w_t+\rho  \nabla L ( w_t) -w_*\rangle + \frac{\gamma^2}{2\rho}(1+\rho \beta)  \| \nabla L ( w_{t+1/2}) \|^2 \\ 
    & = -\frac{\gamma }{\rho} \langle \nabla L (w_{t+1/2}), w_{t+1/2} -w_*\rangle + \frac{\gamma^2}{2\rho}(1+\rho \beta)  \| \nabla L ( w_{t+1/2}) \|^2 \\ 
    & \leq  -\frac{\gamma }{\rho} (1-\frac{\gamma \beta}{2}(1+\rho \beta))  \langle \nabla L (w_{t+1/2}), w_{t+1/2} -w_*\rangle. 
\end{align*}

If $L$ is convex then $L(w_{t+1/2}) - L(w_*) \leq \langle \nabla L (w_{t+1/2}), w_{t+1/2} -w_*\rangle $ and therefore we obtain 
\begin{align*}
 \frac{\gamma }{\rho} \left(1-\frac{\gamma \beta}{2}(1+\rho \beta)\right)  \left(L(w_{t+1/2}) - L(w_*)\right) \leq  V_t - V_{t+1}.
\end{align*}
Using the definition of $w_{t+1/2}$ we always have that $ L(w_{t+1/2})\geq L(w_t) +\rho \| \nabla L ( w_t)\|^2 $ therefore 
\begin{align*}
 \frac{\gamma }{\rho} \left(1-\frac{\gamma \beta}{2}(1+\rho \beta)\right)  \left(L( w_t) - L(w_*)\right) \leq  V_t - V_{t+1}.
\end{align*}
And taking the sum and using Jensen inequality we finally obtain:
\begin{align*}
 L( \frac{1}{T} \sum_{t=0}^T w_t) - L(w_*) \leq \frac{V_0 - V_{T+1}}{T \frac{\gamma }{\rho} (1-\frac{\gamma \beta}{2}(1+\rho \beta)) }.
\end{align*}

If $L$ is $\mu$-strongly convex, we use that $ \langle \nabla L (w_{t+1/2}), w_{t+1/2} -w_*\rangle  \geq \mu \| w_{t+1/2} - w_*\|^2  $ to obtain
\begin{align*}
    \| w_{t+1/2} - w_*\|^2 =  \|  w_t + \rho \nabla L (w_t) - w_*\|^2 &= \| w_t -w_*\|^2 + 2 \rho  \langle \nabla L(w_t), w_t-w_*\rangle + \rho^2 \| \nabla L (w_t) \|^2 \\
  &\geq    \| w_t -w_*\|^2 + 2 \rho  \langle \nabla L(w_t), w_t-w_*\rangle \\
  &\geq  \| w_t -w_*\|^2 + 2 \rho [L(w_t) -L(w_*)] \\
  &\geq 2\rho V_t.
\end{align*}

Therefore we have 
\begin{align*}
    V_{t+1}\leq  \left(1 - \gamma \mu  (2-\gamma \beta (1+\rho \beta))\right)  V_t \leq \left(1- \gamma \mu  (2-\gamma \beta (1+\rho \beta)) \right)^{t+1} V_0.
\end{align*}
\end{proof}

\subsection{Convergence of Stochastic SAM}
\label{sec:app_theory_convergencesto}

\subsubsection{Convergence of $n$-SAM}
When the SAM algorithm is implemented with the $n$-SAM objective as optimization objective, two different batches are used in the ascent and descent steps. We obtain the $n$-SAM algorithm defined as
\begin{equation}\label{eq:samstod}
    w_{t+1} = w_t -\frac{\gamma_t }{b}\sum_{i\in I_t} \nabla \ell_i \big(  w_t +\frac{\rho_t }{b} \sum_{i\in J_t} \nabla \ell_i(w_t) \big),
\end{equation}
where $I_t$ and $J_t$ are two \textit{different} mini-batches of data of size $b$. For this variant of the SAM algorithm, we obtain the following convergence result.
 \begin{theorem}
 \label{theorem:samstodiffn}
Assume~\textbf{(A1)}, \textbf{(A2')} for the iterates~(\ref{eq:samstod}). For any $T\geq 0$ and for step sizes $\gamma_t= \frac{1}{\sqrt{T}\beta}$ and $\rho_t=\frac{1}{T^{1/4}\beta}$, we have: 
\begin{align*}
    \frac{1}{T}\E \left[\sum_{t=0}^{T-1}\| \nabla L (w_t) \|^2 \right] \leq \frac{4}{\beta \sqrt{T}} (L(w_0)-L_*) +\frac{8 \sigma^2}{b \sqrt{T}},
\end{align*}
In addition, under \textbf{(A2)}, with step sizes $\gamma_t=\min\{ \frac{8t+4}{3\mu(t+1)^2}, \frac{1}{2\beta}\}$ and $\rho_t=\sqrt{\gamma_t/\beta}$:
\begin{align*}
    \E \left[L(w_T)\right] - L_* \leq \frac{3\beta^2( L(w_0) -L_*)}{\mu^2T^2} + \frac{22\beta\sigma^2}{b \mu^2T}
\end{align*}
\end{theorem}
We obtain the same convergence result as in Theorem~\ref{theorem:samsto}, but under the relaxed smoothness assumption~\textbf{(A2')}.

As in the deterministic case, the proof relies on two lemmas which shows that the SAM update is well correlated with the gradient and that the decrease of function values can be controlled.

\myparagraph{Auxiliary lemmas.}
The following lemma shows that the SAM update is well correlated with the gradient $\nabla L (w_t)$. Let us denote by $\nabla L_{t+1}(w)=\frac{1}{b}\sum_{i\in I_t} \nabla \ell_i (w) $, $\nabla L_{t+1/2}(w)=\frac{1}{b}\sum_{i\in J_t} \nabla \ell_i (w)$, and $w_{t+1/2}= w_t + \rho \nabla L_{t+1/2}(w_t)$ the SAM ascent step. 
\begin{lemma} \label{lem:gradsamsto}
Assume  \textbf{(A1)} and \textbf{(A2)}.  Then for all $\rho\geq0$, $t\geq 0$ and $w\in \mathbb R^d$, 
\begin{align*}
  \E \langle  \nabla L_{t+1}(w+\rho \nabla L_{t+1/2}(w) ), \nabla L (w) \rangle \geq (1/2-\beta \rho ) \|\nabla L (w)\|^2 -\frac{\beta^2 \rho^2 \sigma^2}{2}.
\end{align*}
\end{lemma}
The proof is similar to the proof of Lemma~\ref{lem:gradsam}. Only the stochasticity of the noisy gradients has to be taken into account. For this goal, we consider instead the update which would have been obtained without noise, and bound the remainder using the bounded variance assumption $\textbf{(A1)}$.
\begin{proof}
Let us denote by $\hat w= w + \rho \nabla L (w)$, the true gradient step. We first add and subtract $\nabla L_{t+1/2}(\hat w)$
\begin{align*}
   \langle \nabla L_{t+1}(w+\rho \nabla L_{t+1/2}(w) ) , \nabla L (w) \rangle & =   \langle \nabla L_{t+1}(w+\rho \nabla L_{t+1/2}(w) )- \nabla L_{t+1}(\hat w) , \nabla L (w) \rangle - \langle  \nabla L_{t+1}(\hat w) , \nabla L (w)\rangle.
   \end{align*}
We bound the two terms separately. We use the smoothness of $L$ (Assumption \textbf{(A2')}) to bound the first term: 
\begin{align*}
   -  \E \langle \nabla L_{t+1}(w+\rho \nabla L_{t+1/2}(w) )- \nabla L_{t+1}(\hat w) , \nabla L (w) \rangle & = - \E \langle \nabla L(w+\rho \nabla L_{t+1/2}(w) )- \nabla L(\hat w) , \nabla L (w)\rangle \\
    &\leq  \frac{1}{2}\E\| \nabla L(w+\rho \nabla L_{t+1/2}(w) )- \nabla L(\hat w) \|^2  +\frac{1}{2}\| \nabla L (w)\|^2 \\
    &\leq  \frac{\beta^2}{2} \E\| w+\rho \nabla L_{t+1/2}(w) - \hat w \|^2 +\frac{1}{2}\| \nabla L (w)\|^2 \\
    &\leq  \frac{\beta^2 \rho^2}{2} \E\| \nabla L_{t+1/2} (w)-\nabla L (w) \|^2  +\frac{1}{2} \| \nabla L (w_t)\|^2 \\
    &\leq  \frac{\beta^2 \rho^2 \sigma^2}{2b}  + \frac{1}{2}\ \| \nabla L (w)\|^2, 
   \end{align*}
 where we have used that the variance of a mini-batch of size $b$ is bounded by $\sigma^2/b$.  Note that this term can be equivalently bounded by $\beta\rho\sigma/\sqrt{b} \| \nabla L (w)\|$ if needed. For the second term, we directly apply Lemma~\ref{lem:gradsam} to obtain 
  \begin{align*}
  \E \langle  \nabla L_{t+1}(\hat w) , \nabla L (w)\rangle&=  \E \langle  \nabla L(\hat w), \nabla L (w)\rangle \geq  (1- \beta\rho)   \|  \nabla L (w) \|^2.
\end{align*}
\end{proof}
The next lemma shows that the decrease of function values of stochastic $n$-SAM can be controlled similarly as for standard stochastic gradient descent. 
\begin{lemma}
\label{lem:smoothdecreasedn}
Let us assume \textbf{(A1, A2')} then for all $\gamma \leq \frac{1}{2\beta}$ and  $\rho\leq \frac{1}{2\beta}$, the iterates~(\ref{eq:samstod}) satisfies
\begin{align*}
\E L(w_{t+1}) \leq \E L(w_t) - \frac{\gamma}{4}\E\| \nabla L ( w_t) \|^2  +\gamma \beta \sigma^2 (\gamma + \rho^2 \beta ). 
\end{align*}
\end{lemma}
This lemma is analogous to Lemma~\ref{lem:smoothdecrease} in the stochastic case. The proof is very similar, with the slight difference that Lemma~\ref{lem:gradsamsto} is used instead of Lemma~\ref{lem:gradsam}.

\begin{proof}
Let us define by $w_{t+1/2} = w_t + \rho \nabla L_{t+1/2} (w_t)$.
Using the smoothness of the function $L$ \textbf{(A2)}, we obtain
\begin{align*}
L(w_{t+1}) \leq L(w_t) - \gamma \langle \nabla L_{t+1}( w_{t+1/2}  ) , \nabla L (w_t) \rangle + \frac{\gamma^2 \beta}{2} \| \nabla L_{t+1} (w_{t+1/2}) \|^2.
\end{align*}
Taking the expectation and using that the variance is bounded \textbf{(A1)} yields to 
\begin{align*}
\E L(w_{t+1}) &\leq \E L(w_t) - \gamma \E \langle \nabla L( w_{t+1/2}  ) , \nabla L (w_t) \rangle + \frac{\gamma^2 \beta}{2} \E \| \nabla L_{t+1} (w_{t+1/2}) \|^2 \\
&\leq \E L(w_t) - \gamma \E \langle \nabla L( w_{t+1/2}  ) , \nabla L (w_t) \rangle + {\gamma^2 \beta} \E \| \nabla L_{t+1} (w_{t+1/2}) - \nabla L (w_{t+1/2}) \|^2 + {\gamma^2 \beta} \E \| \nabla L (w_{t+1/2}) \|^2  \\
&\leq \E L(w_t) - \gamma \E \langle \nabla L( w_{t+1/2}  ) , \nabla L (w_t) \rangle +{\gamma^2 \beta \frac{\sigma^2}{b}}  + {\gamma^2 \beta} \E \| \nabla L (w_{t+1/2}) \|^2.  
\end{align*}
The main trick is still to use the binomial squares
\begin{align*}
   \| \nabla L ( w_{t+1/2}) \|^2  =   - \| \nabla L ( w_t) \|^2  + \| \nabla L ( w_{t+1/2}) - \nabla L(w_t) \|^2  +2 \langle \nabla L(  w_{t+1/2}  ) , \nabla L (w_t) \rangle \\
\end{align*}
to bound 
\begin{align*}
\E L(w_{t+1}) &\leq \E L(w_t) - \gamma \E  \langle \nabla L(  w_{t+1/2}  ) , \nabla L (w_t) \rangle + \frac{\gamma^2 \beta}{2} \E\| \nabla L ( w_{t+1/2}) \|^2 + \gamma^2 \sigma^2 \beta/b
\\
  &= \E L(w_t) -  {\gamma^2 L} \E\| \nabla L ( w_t) \|^2  +{\gamma^2 \beta} \E \| \nabla L ( w_{t+1/2}) - \nabla L(w_t) \|^2  
  \\
  & \quad - \gamma (1 - 2 \gamma \beta ) \E \langle \nabla L(  w_{t+1/2}  ) , \nabla L (w_t) \rangle  + \gamma^2 \sigma^2 \beta/b \\
&= \E L(w_t) -  {\gamma^2 \beta} \E\| \nabla L ( w_t) \|^2  +{\gamma^2 L^3} \E \|   w_{t+1/2}- w_t \|^2  
  \\
  & \quad - \gamma (1 - 2 \gamma \beta )(1/2+\alpha \rho ) \E  \| \nabla L ( w_t) \|^2 +  \gamma (1 - 2 \gamma L ) \sigma^2 \rho^2 \beta^2/2 + \gamma^2 \sigma^2 \beta/b \\
     &= \E L(w_t) -  {\gamma^2 \beta} \E\| \nabla L ( w_t) \|^2  +{\gamma^2 \beta^3 \rho^2 } \E \|  \nabla L_{t+1/2}( w_t) \|^2  
  \\
  & \quad - \gamma (1 - 2 \gamma \beta )(1/2+\alpha \rho ) \E  \| \nabla L ( w_t) \|^2 +  \gamma (1 - 2 \gamma \beta ) \sigma^2/b \rho^2 \beta^2/2 + \gamma^2 \sigma^2 \beta/b \\
    &= \E L(w_t) -  {\gamma^2 \beta} \E\| \nabla L ( w_t) \|^2  +{2\gamma^2 \beta^3 \rho^2 } \E \|  \nabla L( w_t) \|^2  +{2\gamma^2 \beta^3 \rho^2 } \sigma^2/b
  \\
  & \quad - \gamma (1 - 2 \gamma \beta )(1/2+\alpha \rho ) \E  \| \nabla L ( w_t) \|^2 +  \gamma (1 - 2 \gamma \beta ) \sigma^2 \rho^2 \beta^2/2 + \gamma^2 \sigma^2 \beta/b\\
           & \leq L(w_t) - \frac{\gamma}{2}[1-2\rho \beta (1-2\gamma \beta(1-\rho \beta ))] \E\| \nabla L ( w_t) \|^2  +\gamma \sigma^2 \beta/b [ \gamma  +\rho^2 L /2 (1+2\gamma \beta ) ]
\end{align*}
where we have used Lemma \ref{lem:gradsamsto} and that $\| \nabla L (w_{t+1/2}) - \nabla L(w_t) \|^2 \leq \beta^2 \| w_{t+1/2} - w_t \|^2$.
\end{proof}

Using Lemma~\ref{lem:smoothdecreasedn} we directly obtain the following convergence result.
\begin{theorem}
Assume \textbf{(A1)} and \textbf{(A2')}. For $\gamma\leq 1/(2\beta)$ and $\rho\leq 1/(2\beta)$, the iterates~(\ref{eq:sambatch}) satisfies:
\begin{align*}
    \frac{1}{T}\sum_{t=0}^{T-1} \E \| \nabla L ( w_t) \|^2    \leq 4 \frac{ L(w_0)-\E L(w_T)}{T\gamma} + {4 T \sigma^2 \beta (\gamma +\rho^2 \beta)}/b. 
\end{align*}
\end{theorem}
This theorem gives the first part of Theorem~\ref{theorem:samstodiffn}. The proof of the stronger result obtained when the function is in addition PL (Assumption \textbf{(A3)}) is similar to the proof of Theorem 3.2 of \citet{gower}, only the constants are changing.

\subsubsection{Convergence of $m$-SAM}
In the $m$-SAM algorithm, the \textit{same} batch is used in the ascent and descent steps unlike in the $n$-SAM algorithm analyzed above. We obtain then iterates~(\ref{eq:sambatch}) for which we have stated the convergence result in Theorem~\ref{theorem:samsto} in the main part. The proof follows the same lines as above with the minor difference that we are assuming the \textit{individual} gradients $\nabla f_t$ are Lipschitz (Assumption \textbf{(A2)})  to control the alignment of the expected SAM direction. Let us denote by $\nabla L_{t}(w) = \frac{1}{b}\sum_{i\in J_t} \nabla \ell_i (w)$.
\begin{lemma} \label{lem:gradsamstos}
Assume \textbf{(A1-2)}. Then we have for all $w\in \R^d$, $\rho\geq 0$ and $t\geq 0$
\begin{align*}
  \E \langle  \nabla L_{t}(w + \rho \nabla L_t (w)), \nabla L (w) \rangle \geq (1/2 -\rho \beta) \|  \nabla L (w) \|^2 -\frac{{ \beta}^2 \rho^2 \sigma^2}{2b}. 
\end{align*}
\end{lemma}
The proof is very similar to the proof of Lemma~\ref{lem:gradsamsto}. The only difference is that the Assumption \textbf{(A2)} is used instead of \textbf{(A2')}.
\begin{proof}
Let us denote by $\hat w = w + \rho \nabla L (w)$, the true gradient step. We first add and subtract $\nabla L_{t}(\hat w)$
\begin{align*}
   \langle \nabla L_{t}(w + \rho \nabla L_t (w)), \nabla L (w) \rangle & =   \langle \nabla L_{t}(w + \rho \nabla L_t (w))- \nabla L_{t}(\hat w) , \nabla L (w) \rangle - \langle  \nabla L_{t}(\hat w) , \nabla L (w)\rangle.
   \end{align*}
We bound the two terms separately. We use the smoothness of $L_t$ to bound the first term (Assumption \textbf{(A2)}): 
\begin{align*}
   -   \langle \nabla L_{t}(w + \rho \nabla L_t (w))- \nabla L_{t}(\hat w ) , \nabla L (w) \rangle
    &\leq  \frac{1}{2}\| \nabla L_{t}(w + \rho \nabla L_t (w))- \nabla L_{t}(\hat w) \|^2  +\frac{1}{2}\| \nabla L (w)\|^2 \\
    &\leq  \frac{{ \beta}^2}{2} \E\| w + \rho \nabla L_t (w)- \hat w \|^2 +\frac{1}{2}\| \nabla L (w)\|^2 \\
    &\leq  \frac{{ \beta}^2 \rho^2}{2} \| \nabla L_{t} (w)-\nabla L (w) \|^2  +\frac{1}{2} \| \nabla L (w)\|^2. 
   \end{align*}
And taking the expectation, we obtain: 
   \begin{align*}
   -   \E\langle \nabla L_{t}(w + \rho \nabla L_t (w))- \nabla L_{t}(\hat w) , \nabla L (w) \rangle 
    &\leq  \frac{{ \beta}^2 \rho^2 \sigma^2}{2b}  + \frac{1}{2}\E \| \nabla L (w)\|^2.
   \end{align*}
For the second term, we apply directly Lemma~\ref{lem:gradsam}
  \begin{align*}
  \E \langle  \nabla L_{t}(\hat w) , \nabla L (w_t)\rangle &= \langle  \nabla L(\hat w), \nabla L (w)\rangle \geq  (1- \beta \rho)   \|  \nabla L (w) \|^2.
\end{align*}
Assembling the two inequalities yields the result.
\end{proof}

The next lemma shows that the decrease of function values of the $m$-SAM algorithm can be controlled similarly as in the case of gradient descent. It is analogous to Lemma~\ref{lem:smoothdecreasedn} where \textit{different} batches are used in both the ascent and descent steps of SAM algorithm.
\begin{lemma}\label{lem:smoothdecreasesn}
Assume \textbf{(A1-2)}. For all $\gamma \leq \frac{1}{\beta}$ and  $\rho \leq \frac{1}{4\beta}$, the iterates~(\ref{eq:sambatch}) satisfy
\begin{align*}
    \E L(w_{t+1}) \leq \E L(w_t) - \frac{3\gamma}{8}\E\| \nabla L ( w_t) \|^2  +\gamma \beta \frac{\sigma^2}{b} (\gamma + 2\rho^2 \beta ). 
\end{align*}
\end{lemma}

\begin{proof}
Let us define by $w_{t+1/2} = w_t + \rho \nabla L_{t+1} (w_t)$.
Using the smoothness of the function $L$ which is implied by \textbf{(A2)}, we obtain
\begin{align*}
L(w_{t+1}) \leq L(w_t) - \gamma \langle \nabla L_{t+1}( w_{t+1/2}  ) , \nabla L (w_t) \rangle + \frac{\gamma^2 \beta}{2} \| \nabla L_{t+1} (w_{t+1/2}) \|^2.
\end{align*}

We still use the binomial squares
\begin{align*}
   \| \nabla L_{t+1} ( w_{t+1/2}) \|^2  =   - \| \nabla L ( w_t) \|^2  + \| \nabla L_{t+1} ( w_{t+1/2}) - \nabla L(w_t) \|^2  +2 \langle \nabla L_{t+1}(  w_{t+1/2}  ) , \nabla L (w_t) \rangle \\
\end{align*}
and bound $L(w_{t+1})$ by 
\begin{align*}
 L(w_{t+1}) 
  &\leq   L(w_t) -  \frac{\gamma^2 \beta}{2} \| \nabla L ( w_t) \|^2  +\frac{\gamma^2 \beta}{2}  \| \nabla L_{t+1} ( w_{t+1/2}) - \nabla L(w_t) \|^2  
 - \gamma (1 - \gamma \beta )  \langle \nabla L_{t+1}(  w_{t+1/2}  ) , \nabla L (w_t) \rangle  \\
 &\leq   L(w_t) -  \frac{\gamma^2 \beta}{2} \| \nabla L ( w_t) \|^2  +{\gamma^2 \beta}  \| \nabla L_{t+1} ( w_{t+1/2}) - \nabla L_{t+1}(w_t) \|^2  +{\gamma^2 \beta}  \| \nabla L_{t+1} ( w_{t}) - \nabla L(w_t) \|^2  \\
 & \quad - \gamma (1 - \gamma \beta )  \langle \nabla L_{t+1}(  w_{t+1/2}  ) , \nabla L (w_t) \rangle  \\
  &\leq   L(w_t) -  \frac{\gamma^2 \beta}{2} \| \nabla L ( w_t) \|^2  +{\gamma^2 \beta { \beta}^2}  \|  w_{t+1/2} - w_t \|^2  +{\gamma^2 \beta }  \| \nabla L_{t+1} ( w_{t}) - \nabla L(w_t) \|^2  \\
 & \quad - \gamma (1 - \gamma \beta )  \langle \nabla L_{t+1}(  w_{t+1/2}  ) , \nabla L (w_t) \rangle  \\
  &=   L(w_t) -  \frac{\gamma^2 \beta}{2} \| \nabla L ( w_t) \|^2  +{\gamma^2 { \beta}^3 \rho^2}  \| \nabla L_{t+1}(w_t)   \|^2  +{\gamma^2 \beta }  \| \nabla L_{t+1} ( w_{t}) - \nabla L(w_t) \|^2  \\
 & \quad - \gamma (1 - \gamma \beta )  \langle \nabla L_{t+1}(  w_{t+1/2}  ) , \nabla L (w_t) \rangle  \\
  &=   L(w_t) -  \frac{\gamma^2 \beta}{2}(1-4 { \beta}^2 \rho^2) \| \nabla L ( w_t) \|^2  +\gamma^2 \beta (1+2{ \beta}^2 \rho^2) \| \nabla L_{t+1}(w_t) - \nabla L(w_t)   \|^2 \\
 & \quad - \gamma (1 - \gamma \beta )  \langle \nabla L_{t+1}(  w_{t+1/2}  ) , \nabla L (w_t) \rangle  
\end{align*}

Taking the expectation and using  Lemma~\ref{lem:gradsamstos}, we obtain 
\begin{align*}
 \E L(w_{t+1}) 
  &\leq   \E L(w_t) -  \frac{\gamma^2 \beta}{2}(1-4 { \beta}^2 \rho^2) \E \| \nabla L ( w_t) \|^2  +\gamma^2 \beta (1+2{ \beta}^2 \rho^2) \E \| \nabla L_{t+1}(w_t) - \nabla L(w_t)   \|^2 \\
 & \quad - \gamma (1 - \gamma \beta )  \E \langle \nabla L_{t+1}(  w_{t+1/2}  ) , \nabla L (w_t) \rangle \\ 
&\leq  \E L(w_t) -  \frac{\gamma^2 \beta}{2}(1-4 { \beta}^2 \rho^2) \E \| \nabla L ( w_t) \|^2  +\gamma^2 \beta (1+2{ \beta}^2 \rho^2) \sigma^2/b \\
 & \quad - \gamma (1 - \gamma \beta )  (1/2-\beta \rho )  \E \| \nabla L ( w_t) \|^2 + \gamma (1 - \gamma \beta ) \frac{\rho^2 \sigma^2 { \beta}^2}{2b} \\
 &\leq  \E L(w_t) -  \frac{\gamma^2 \beta}{2}(1-4 { \beta}^2 \rho^2) \E \| \nabla L ( w_t) \|^2  +\gamma^2 \beta (1+2{ \beta}^2 \rho^2) \sigma^2/b \\
 & \quad - \frac{\gamma}{2} (1-2 \beta \rho ( 1-\gamma  (\beta-2 \rho { \beta}^2)))  \E \| \nabla L ( w_t) \|^2 +\gamma \sigma^2/b [\gamma \beta +\frac{\rho^2 { \beta}^2}{2} (1+3\gamma \beta)].  
\end{align*}
\end{proof}

Using Lemma~\ref{lem:smoothdecreasesn} we directly obtain the main convergence result for $m$-SAM.
\begin{theorem}
    \label{theorem:samsto_detailed}
    Assume \textbf{(A1-2)}. For $\gamma\leq \frac{1}{\beta}$ and $\rho \leq \frac{1}{4\beta}$, the iterates~(\ref{eq:sambatch}) satisfy:
    \begin{align*}
        \frac{1}{T} \E \left[ \sum_{t=0}^{T-1} \| \nabla L(w_t) \|^2 \right]  \leq  \frac{8}{3T\gamma} \left(L(w_0) - \E L(w_T)\right) + \frac{8 \sigma^2 \beta (\gamma +\rho^2 \beta )}{3b}. 
    \end{align*}
    In addition, under \textbf{(A3)}, with step sizes $\gamma_t=\min\{ \frac{8t+4}{3\mu(t+1)^2}, \frac{1}{2\beta}\}$ and $\rho_t=\sqrt{\gamma_t/\beta}$:
    \begin{align*}
        \E [L(w_T)] -L_* \leq \frac{3\beta^2( L(w_0) -L_*)}{\mu^2T^2} + \frac{22\beta\sigma^2}{\mu^2b T}.
    \end{align*}
\end{theorem}

\begin{proof}
The first bound directly comes from Lemma~\ref{lem:smoothdecreasesn}. The second bound is similar to the proof of Theorem 3.2 of \citet{gower}, only the constants are changing.
\end{proof}

Finally, we note that Theorem~\ref{theorem:samsto} is a direct consequence of Theorem~\ref{theorem:samsto_detailed} with $\gamma_t = \frac{1}{\sqrt{T}\beta}$, $\rho_t = \frac{1}{T^{1/4}\beta}$ and slightly simplified constants.

\section{Experimental Details}
\label{sec:app_exp_details}

\myparagraph{Training details for deep networks.}
In all experiments, we train deep networks using SGD with step size $0.1$, momentum $0.9$, and $\l_2$-regularization parameter $\lambda=0.0005$. We perform experiments on CIFAR-10 and CIFAR-100 \citep{krizhevsky2009learning} where for all experiments we apply basic data augmentations: random image crops and mirroring. We use batch size $128$ for most experiments except when it is mentioned otherwise.
We use a pre-activation ResNet-18 \citep{he2016identity} for CIFAR-10 and ResNet-34 on CIFAR-100 with a width factor 64 and piece-wise constant learning rates (with a 10-times decay at 50\% and 75\% epochs).
We train all models for 200 epochs except those in Sec.~\ref{subsec:empirical_study_implicit_bias} and Sec.~\ref{subsec:opt_deep_networks} for which we use 1000 epochs. We use batch normalization for most experiments, except when it is explicitly mentioned otherwise as, for example, in the experiments where we aim to compute sharpness and for this we use networks with group normalization. 

For all experiments involving SAM, we select the best perturbation radius $\rho$ based on a grid search over $\rho \in \{0.025, 0.05, 0.1, 0.2, 0.3, 0.4\}$. In most cases, the optimal $\rho$ is equal to $0.1$ while in the ERM~$\rightarrow$~SAM experiment, it is equal to $\rho=0.4$ for CIFAR-10 and $\rho=0.2$ for CIFAR-100. We note that using a higher $\rho$ in this case is coherent with the experiments on diagonal linear networks which also required a higher $\rho$. For all experiments with SAM, we use a single GPU, so we do not implicitly rely on lower $m$-sharpness in $m$-SAM. The only exception where $m$ is smaller than the batch size is the experiments shown in Fig.~\ref{fig:test_err_sharpness_vs_rho_diff_m} and Fig.~\ref{fig:test_err_sharpness_vs_rho_many_diff_m}. Regarding $n$-SAM in Fig.~\ref{fig:gen_bound_is_not_predictive}, we implement it by doing the ascent step on a \textit{different} batch compared to the descent step, i.e., as described in our convergence analysis part in Eq.~(\ref{eq:samstod}).

\myparagraph{Sharpness computation.}
We compute $m$-sharpness on 1024 training points (i.e., by averaging over $\lceil 1024/m \rceil$) of CIFAR-10 or CIFAR-100 using $100$ iterations of projected gradient ascent using a step size $\alpha = 0.1 \cdot \rho$. For each iteration, we normalize the updates by the $\ell_2$ gradient norm.

\myparagraph{Confidence intervals on plots.}
Many experimental results are replicated over different random seeds used for training. We show the results using the mean and 95\% bootstrap confidence intervals which is the standard way to show such results in the \texttt{seaborn} library \citet{Waskom2021}.

\myparagraph{Code and computing infrastructure.}
The code of our experiments is publicly available.\footnote{\url{https://github.com/tml-epfl/understanding-sam}} We perform all our experiments with deep networks on a single NVIDIA V100 GPU with 32GB of memory. Since most of our experiments involved a grid search over the perturbation radius $\rho$ and replication over multiple random seeds, we could not do the same at the ImageNet scale due to our limited computational resources.

\section{Additional Deep Learning Experiments}
\label{sec:app_additional_experiments}
In this section, we show additional experimental results complementary to those presented in the main part. In particular, we provide multiple ablation study related to the role of $m$ in $m$-SAM, batch size, and model width. We also provide additional experiments on the evolution of sharpness over training using training time and test time batch normalization, training loss of ERM vs. SAM models, and the performance under label noise for standard and unnormalized SAM.

\subsection{The Effect of $\bm{m}$ in $\bm{m}$-SAM}
\label{subsec:app_effect_of_m}
We show the results of SAM for different $m$ in $m$-SAM (with a fixed batch size 256) in Fig.~\ref{fig:test_err_sharpness_vs_rho_many_diff_m}. We note that in this experiment, we used group normalization instead of batch normalization like, for example, in Fig.~\ref{fig:gen_bound_is_not_predictive}, so the exact test error values should not be compared between these two figures.
We observe from Fig.~\ref{fig:test_err_sharpness_vs_rho_many_diff_m}, that the generalization improvement is larger for smaller $m$ and it is continuous in $m$. We also note that a similar experiment has been done in the original SAM paper \citep{foret2021sharpnessaware}. Here, we additionally verified this finding on an additional dataset (CIFAR-100) and for networks trained without batch normalization (which may have had an extra regularization effect as we discussed in Sec.~\ref{subsec:two_natural_hypotheses}).
\begin{figure}[htbp]
    \centering
    \begin{subfigure}[t]{.37\textwidth}
        \caption{\hspace{8mm}\textbf{ResNet-18 on CIFAR-10}}
        \vspace{-2mm}
        \includegraphics[width=1.0\columnwidth]{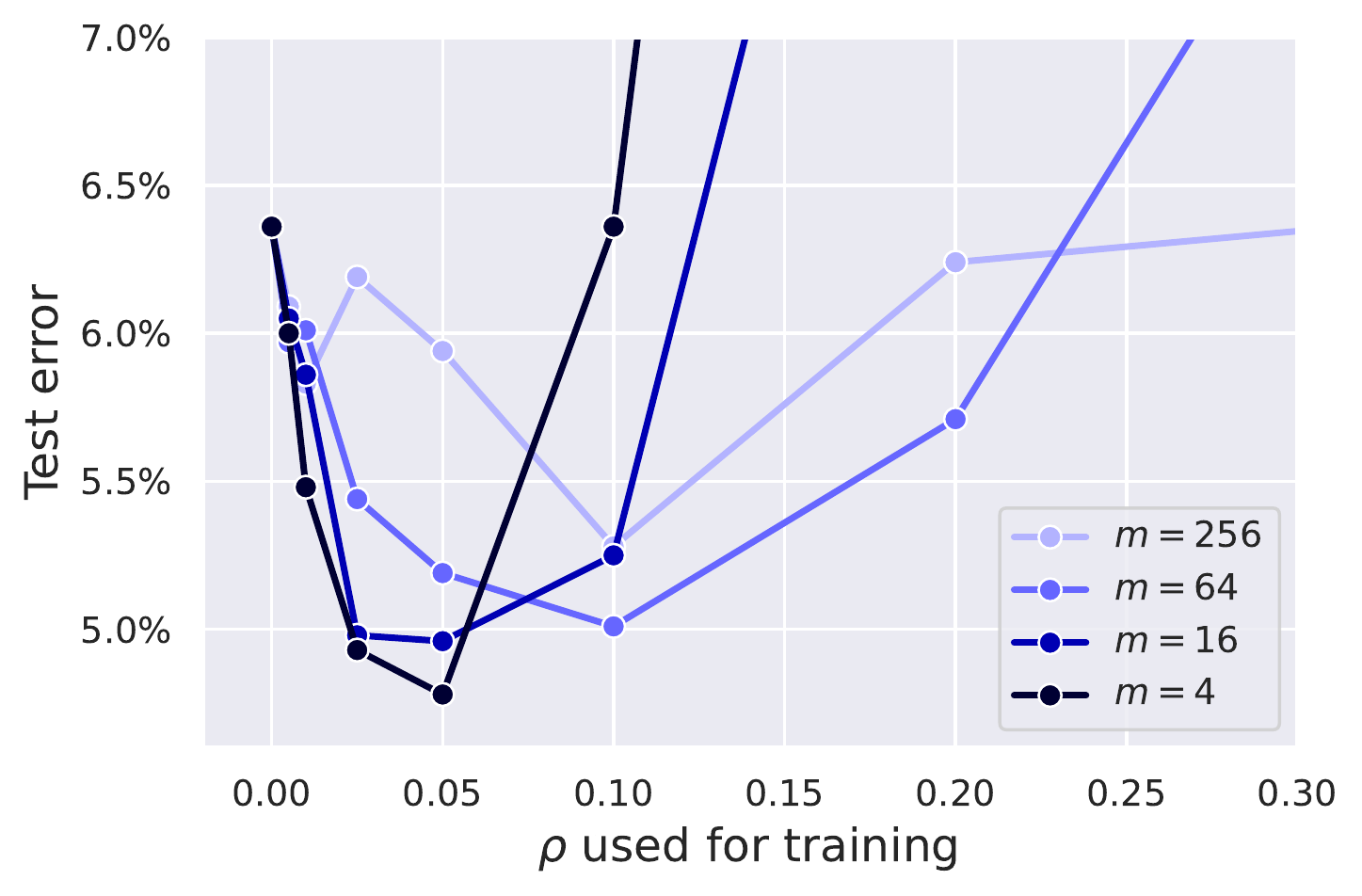}
    \end{subfigure}
    \quad \quad
    \begin{subfigure}[t]{.37\textwidth}
        \caption{\hspace{8mm}\textbf{ResNet-34 on CIFAR-100}}
        \vspace{-2mm}
        \includegraphics[width=1.0\columnwidth]{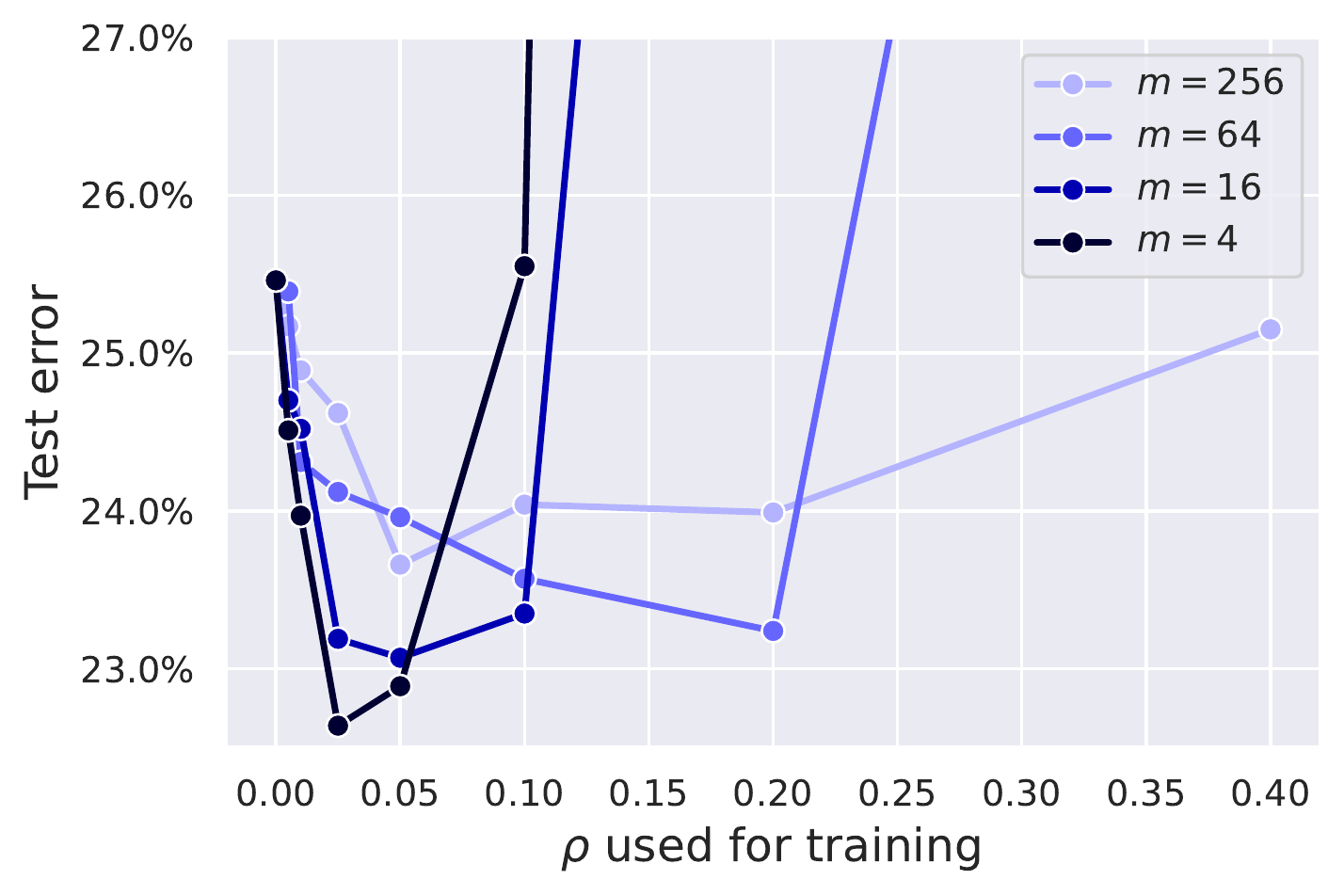}
    \end{subfigure}
    \vspace{-1mm}
    \caption{Test error of models trained with group normalization and \textbf{different $\bm{m}$} in $m$-SAM using batch size $256$.}
    \label{fig:test_err_sharpness_vs_rho_many_diff_m}
\end{figure}

\subsection{The Effect of the Batch Size on SAM}
\label{subsec:app_effect_of_batch_size}
We show the results of SAM for different batch sizes in Fig.~\ref{fig:test_err_sharpness_vs_rho_many_diff_bs} where we use $m$ equal to the batch size. Note that a too high $m$ leads to marginal improvements in generalization ($\approx0.2\%$) and is not able to bridge the gap between large-batch (1024) and small-batch (256 or 128) SGD.
\begin{figure}[htbp]
    \centering
    \begin{subfigure}[t]{.37\textwidth}
        \caption{\hspace{8mm}\textbf{ResNet-18 on CIFAR-10}}
        \vspace{-2mm}
        \includegraphics[width=1.0\columnwidth]{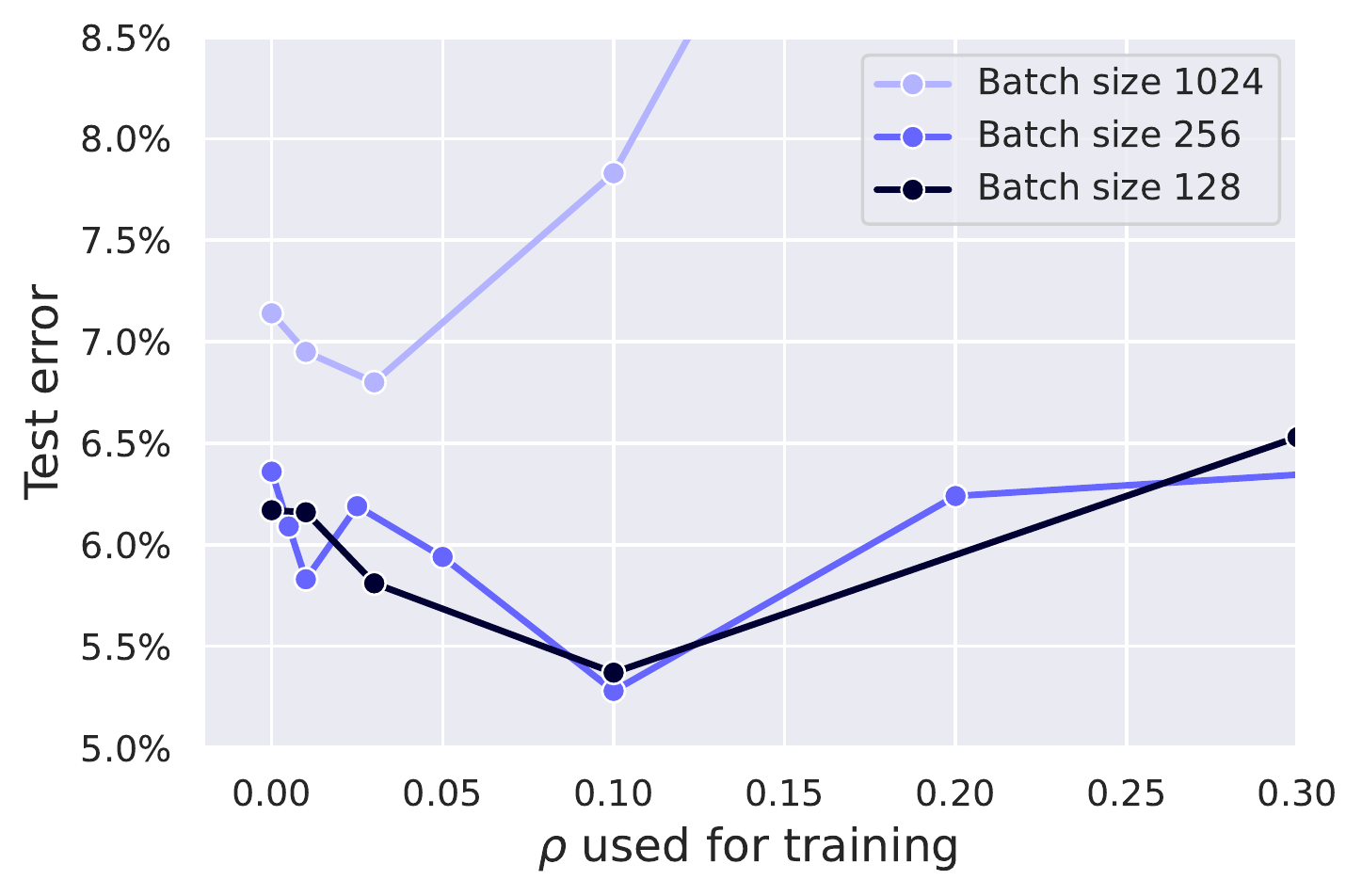}
    \end{subfigure}
    \quad \quad
    \begin{subfigure}[t]{.37\textwidth}
        \caption{\hspace{8mm}\textbf{ResNet-34 on CIFAR-100}}
        \vspace{-2mm}
        \includegraphics[width=1.0\columnwidth]{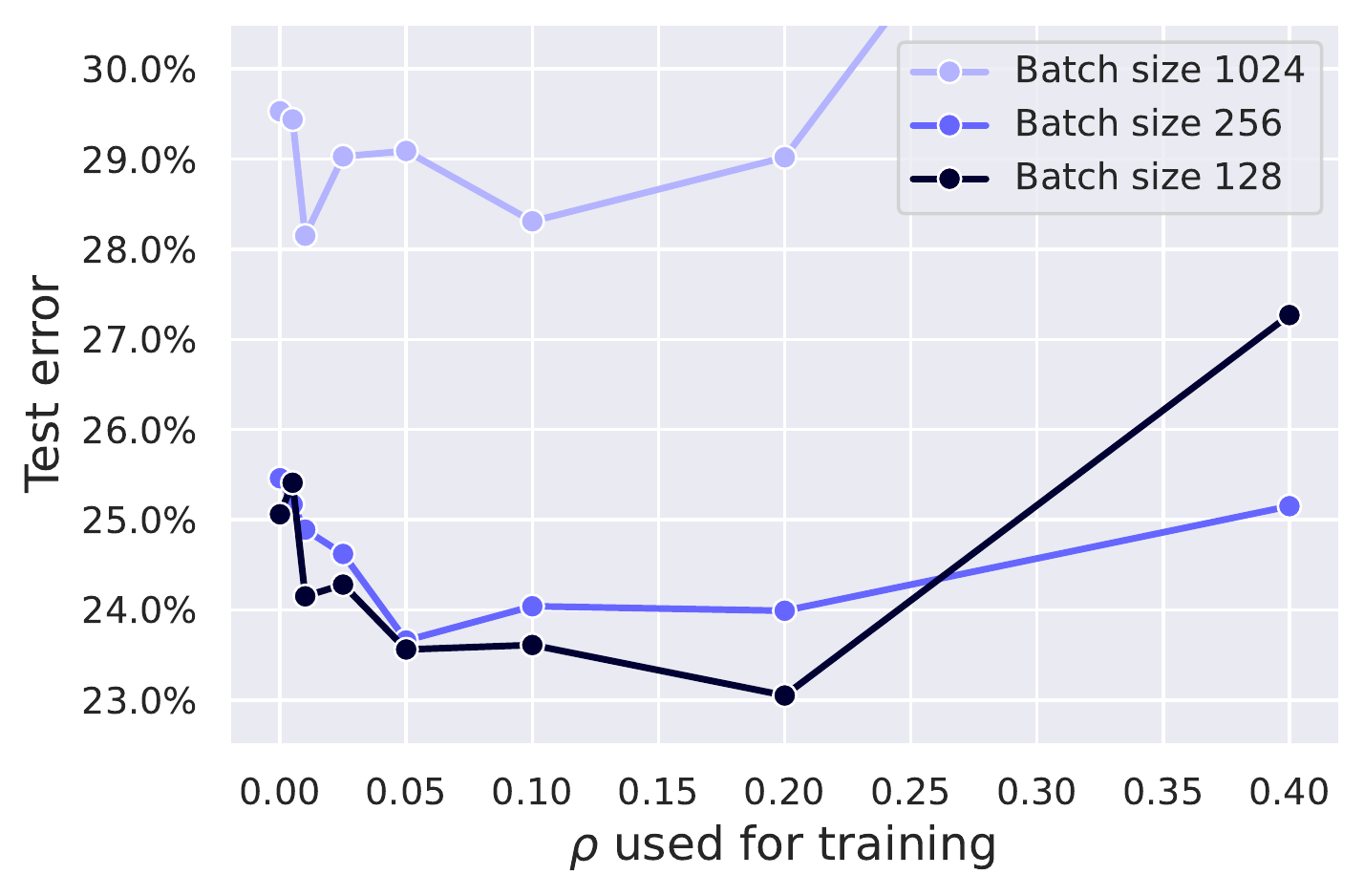}
    \end{subfigure}
    \vspace{-1mm}
    \caption{Test error of models trained with group normalization and \textbf{different batch sizes} for the same number of epochs (200). Note that for all models, we use $m$ in $m$-SAM equal to the batch size.}
    \label{fig:test_err_sharpness_vs_rho_many_diff_bs}
\end{figure}

\subsection{The Effect of the Model Width on SAM}
\label{subsec:app_effect_of_model_width}
We show in Fig.~\ref{fig:test_err_delta_diff_width} test error improvements of SAM over ERM for different model width factors. For comparison, in all other experiments we use model width factor 64. As expected, there is little improvement (or even no improvement as on CIFAR-10) from SAM for small networks where extra regularization is not needed. However, interestingly, the generalization improvement is the largest not for the widest models, but rather for intermediate model widths, such as model width 16. 
\begin{figure}[htbp]
    \centering
    \begin{subfigure}[t]{.37\textwidth}
        \caption{\hspace{8mm}\textbf{ResNet-18 on CIFAR-10}}
        \vspace{-2mm}
        \includegraphics[width=1.0\columnwidth]{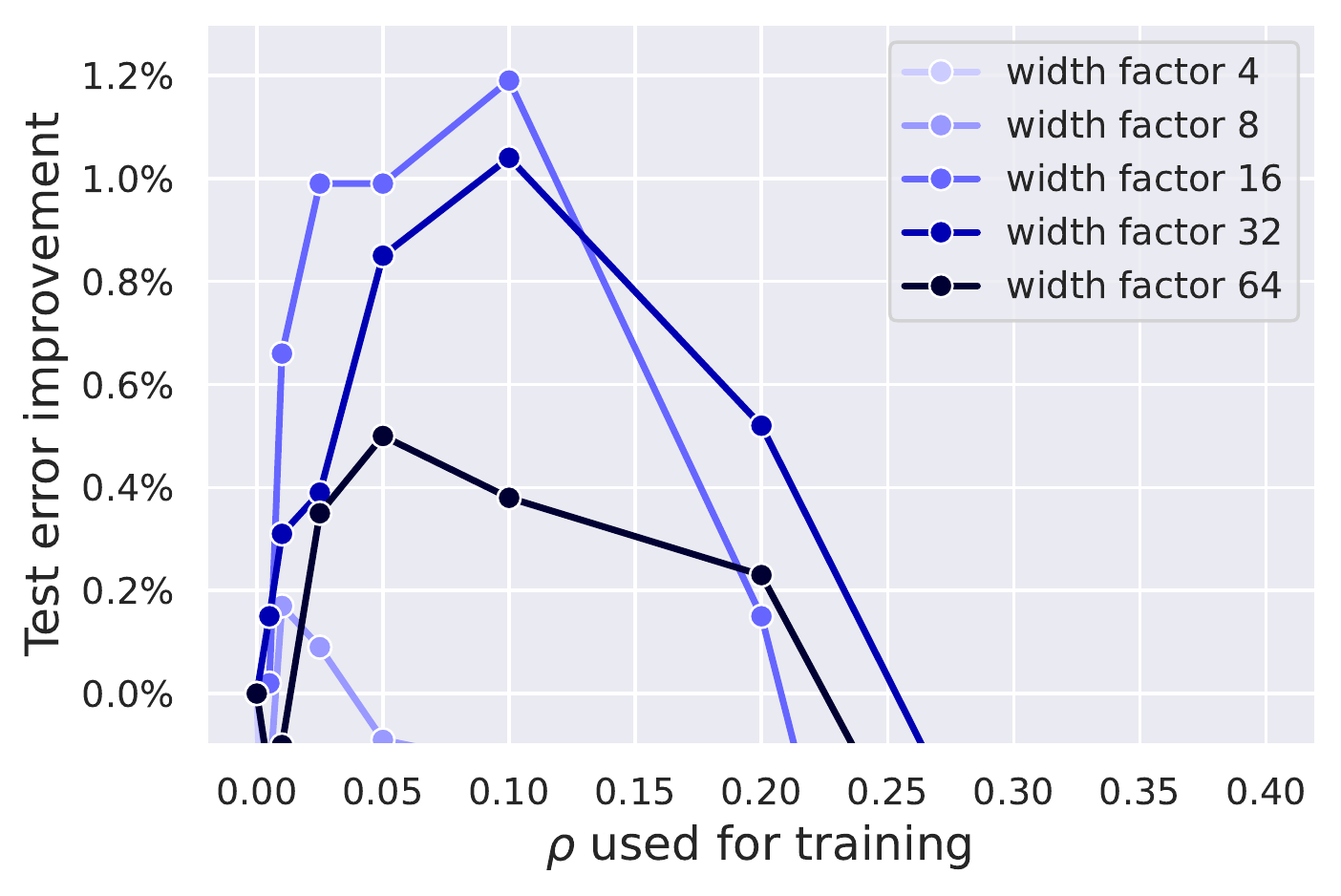}
    \end{subfigure}
    \quad \quad
    \begin{subfigure}[t]{.37\textwidth}
        \caption{\hspace{8mm}\textbf{ResNet-34 on CIFAR-100}}
        \vspace{-2mm}
        \includegraphics[width=1.0\columnwidth]{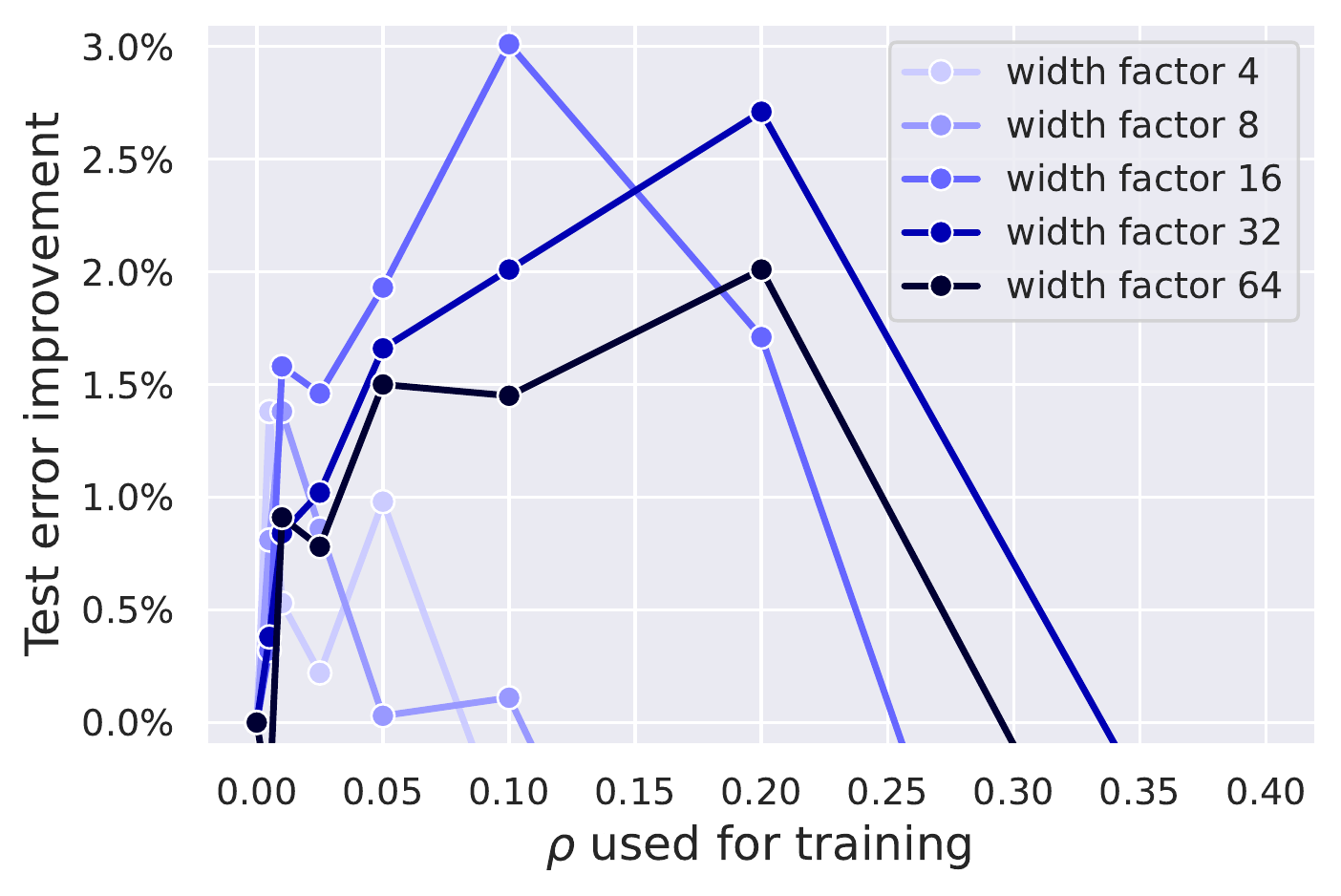}
    \end{subfigure}
    \caption{Test error improvements of SAM over ERM for \textbf{different model width} factors.}
    \label{fig:test_err_delta_diff_width}
\end{figure}

\subsection{Sharpness for Models with Batch Normalization}
\label{subsec:app_sharpness_bn_modes}
The main problem of measuring sharpness for networks with BatchNorm is the discrepancy between training and test-time behaviour. Fig.~\ref{fig:bn_max_loss} illustrates this issue: the maximum loss computed over radius $\rho$ is substantially different depending on whether we use training-time vs. test-time BatchNorm. This is an important discrepancy since the training-time BatchNorm is effectively used by SAM while the test-time BatchNorm is used by default for post-hoc sharpness computation. To avoid this discrepancy, we presented the results in the main part only on models trained with GroupNorm which does not have this problem.
\begin{figure}[htbp]
    \centering
    \footnotesize
    \hspace{8mm}\textbf{ResNet-18 on CIFAR-10}\\
    \includegraphics[width=.38\columnwidth]{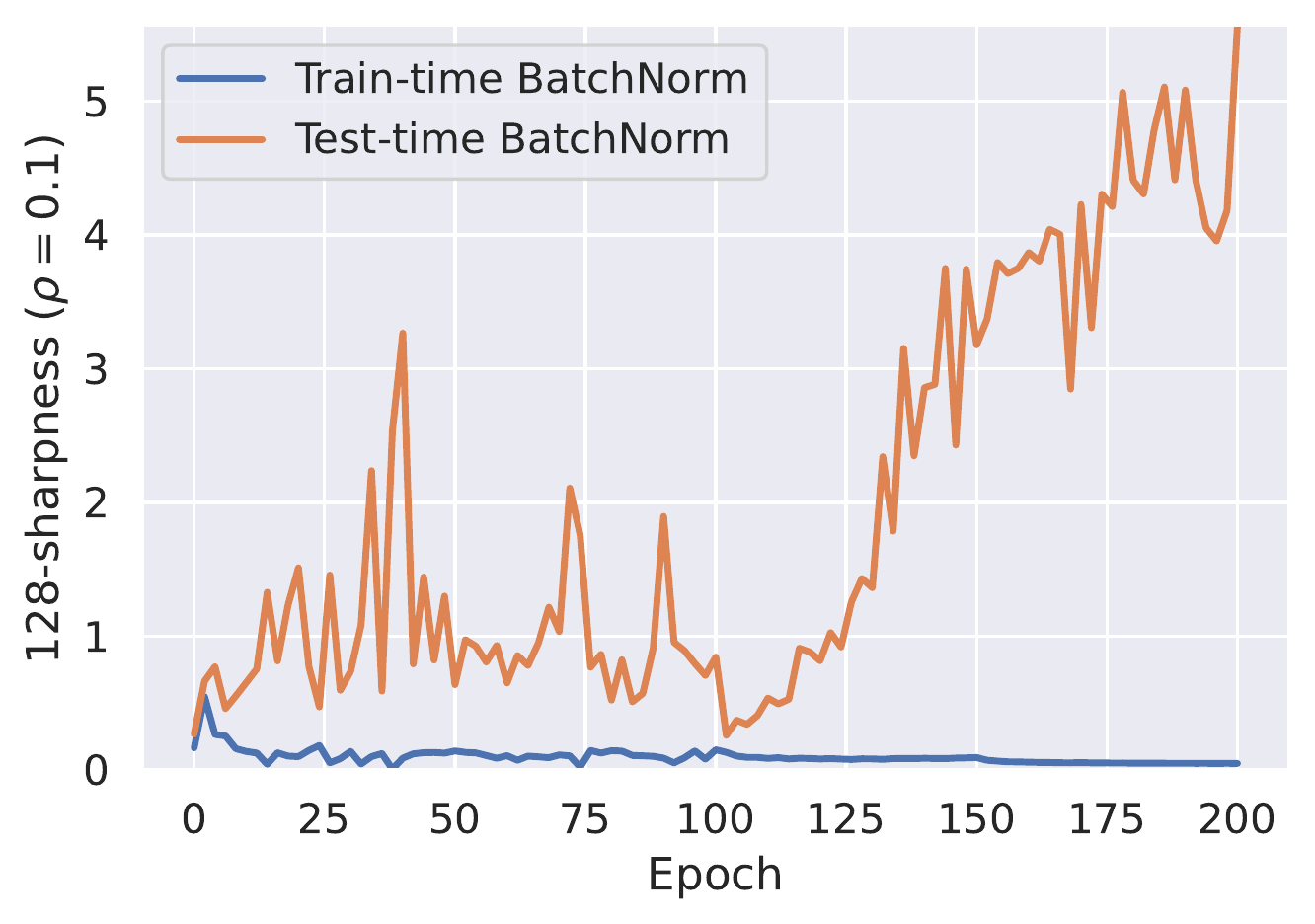}
    \caption{128-sharpness ($\rho=0.1$) over training for a network with batch normalization when measured with the training-time and test-time batch normalization. The model is trained with SAM using $\rho=0.1$.}
    \label{fig:bn_max_loss}
\end{figure}

\subsection{Training Loss for ERM vs. SAM Models}
\label{subsec:app_training_loss_erm_sam}
Fig.~\ref{fig:err_sam_vs_erm} in the main part shows that both training and test errors have a slight increasing trend after the first learning rate decay at 500 epochs. As a sanity check, in Fig.~\ref{fig:obj_sam_vs_erm}, we plot the total objective value (including the $\ell_2$ regularization term) which shows a consistent decreasing trend. Thus, we conclude that the increasing training error is not some anomaly connected to a failure of optimizing the training objective.
\begin{figure}[htbp]
    \centering
    \begin{subfigure}[t]{.37\textwidth}
        \caption{\hspace{8mm}\textbf{ResNet-18 on CIFAR-10}}
        \vspace{-2mm}
        \includegraphics[width=1.0\columnwidth]{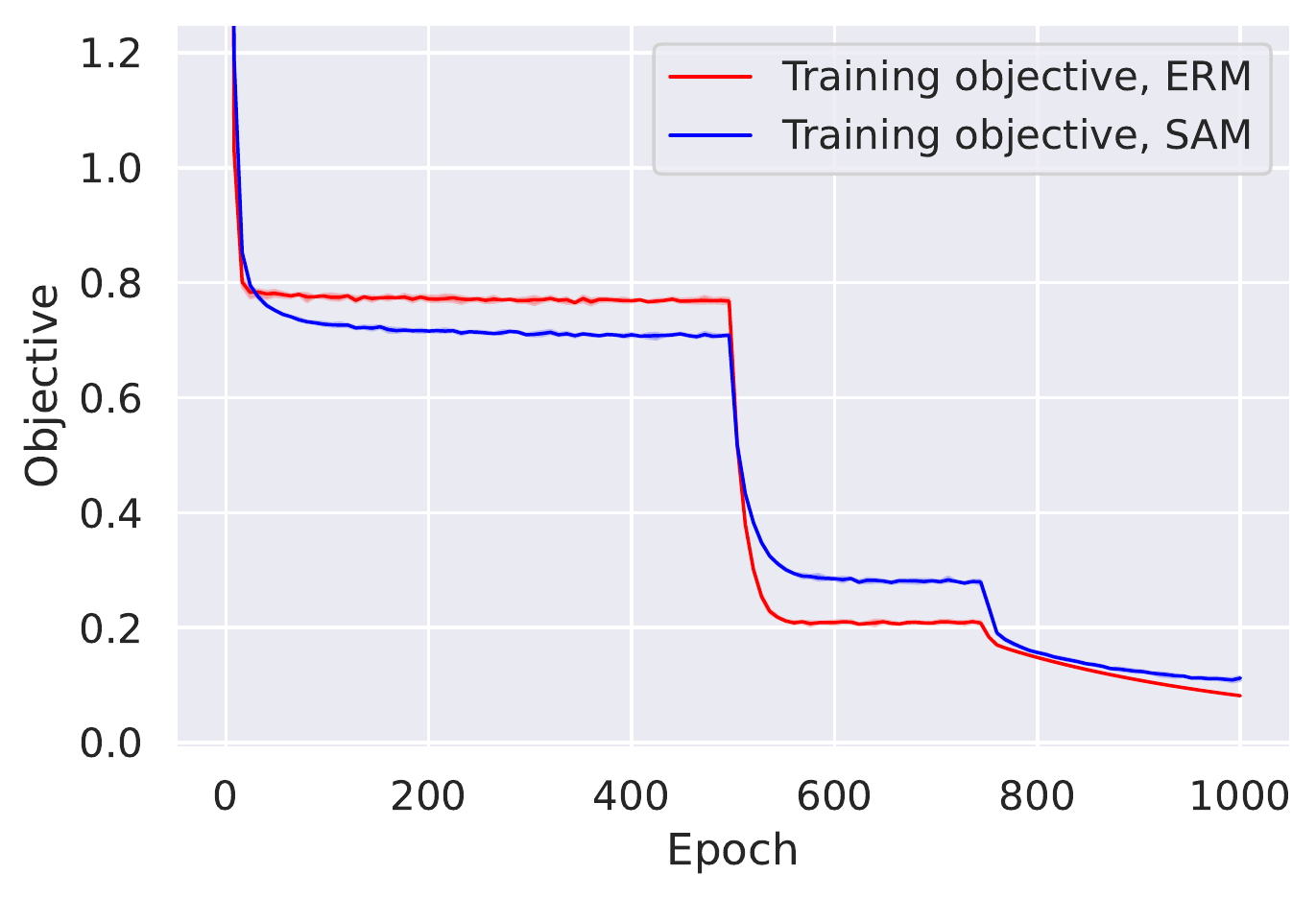}
    \end{subfigure}
    \quad \quad
    \begin{subfigure}[t]{.37\textwidth}
        \caption{\hspace{8mm}\textbf{ResNet-34 on CIFAR-100}}
        \vspace{-2mm}
        \includegraphics[width=1.0\columnwidth]{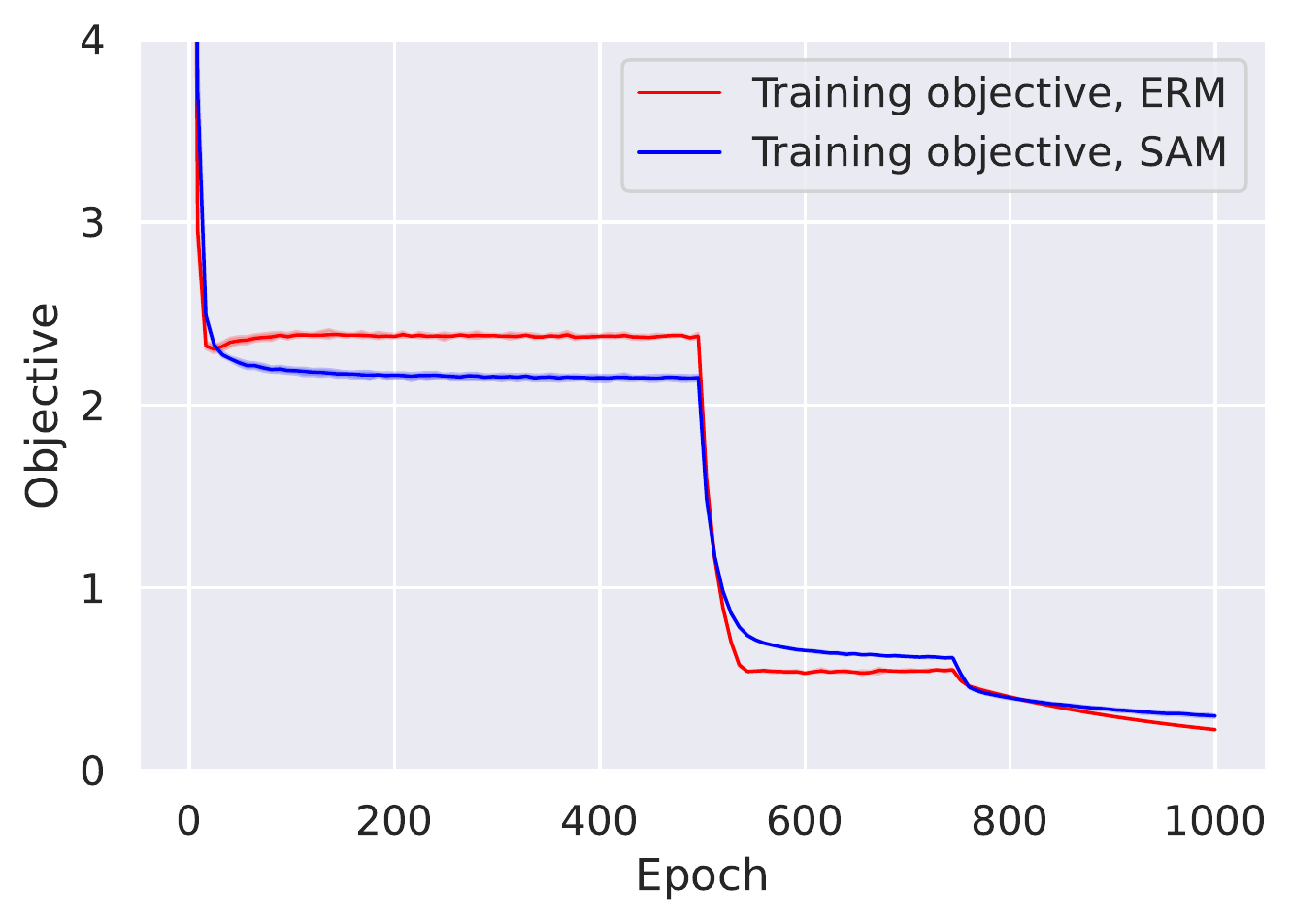}
    \end{subfigure}
    \caption{Training objective of ERM vs. SAM over epochs. For both models, we observe a clear decreasing trend.}
    \label{fig:obj_sam_vs_erm}
\end{figure}

\subsection{SAM with a Decreasing Perturbation Radius $\rho$}
\label{subsec:app_decreasing_rho}
In Fig.~\ref{fig:test_err_over_decreasing_rho}, we plot the test error over different $\rho_t$ where we decay the $\rho_t$ using the same schedule as for the outer learning rate $\gamma_t$. We denote this as \textit{SAM with decreasing $\rho$} contrary to the standard SAM for which $\rho$ is constant throughout training. We note that in both cases, we use the $\ell_2$-normalized updates as in the original SAM. The results suggest that decreasing the perturbation radius $\rho_t$ over epochs is detrimental to generalization. This observation is relevant in the context of the convergence analysis that suggests that SAM converges even if $\rho_t$ is significantly larger than the outer step size $\gamma_t$ which is the case when we decay $\gamma_t$ over epochs while keeping $\rho_t$ constant.
\begin{figure}[htbp]
    \centering
    \begin{subfigure}[t]{.37\textwidth}
        \caption{\hspace{8mm}\textbf{ResNet-18 on CIFAR-10}}
        \vspace{-2mm}
        \includegraphics[width=1.0\columnwidth]{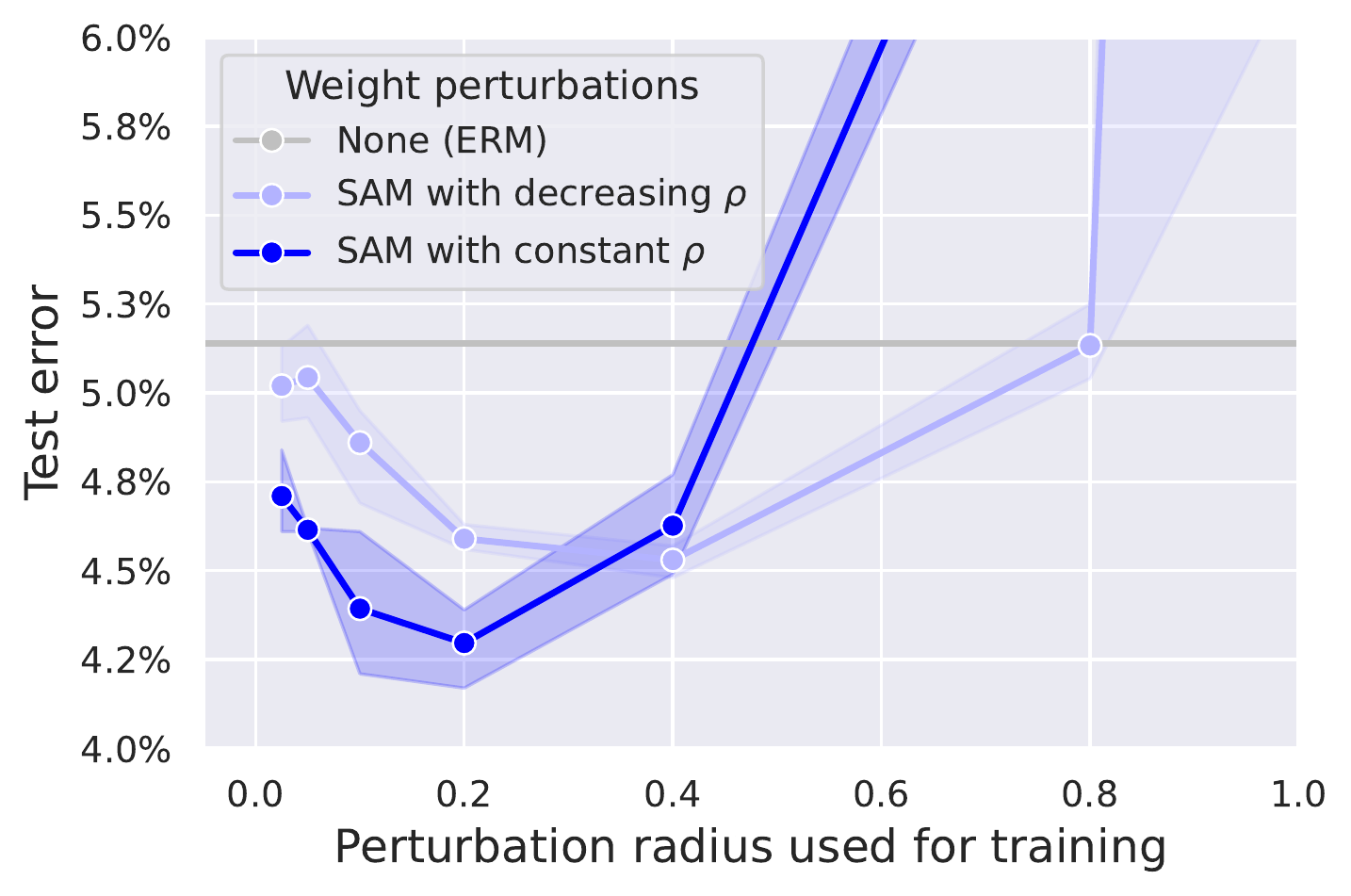}
    \end{subfigure}
    \quad \quad
    \begin{subfigure}[t]{.37\textwidth}
        \caption{\hspace{8mm}\textbf{ResNet-34 on CIFAR-100}}
        \vspace{-2mm}
        \includegraphics[width=1.0\columnwidth]{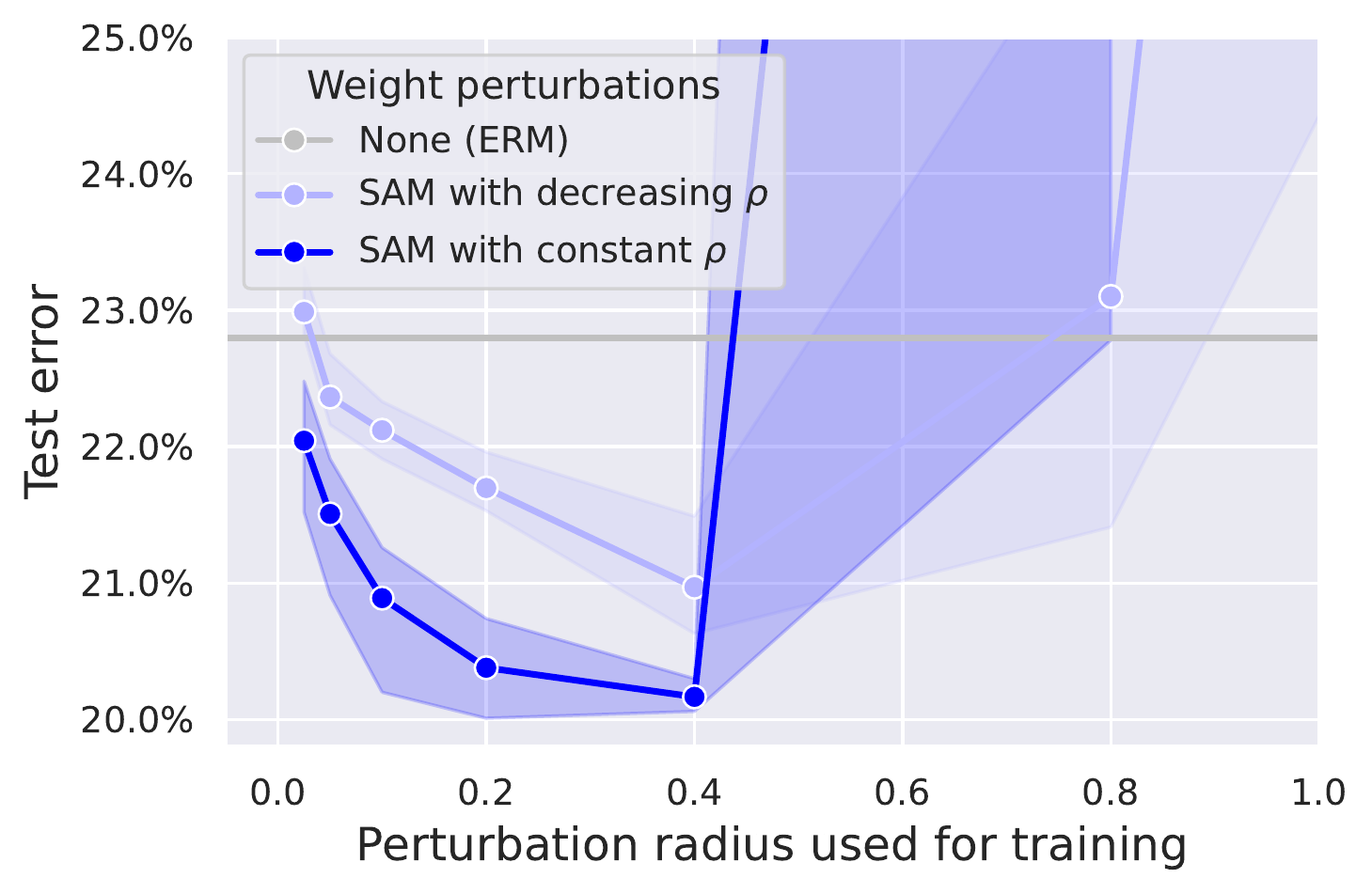}
    \end{subfigure}
    \caption{Test error of SAM with a constant perturbation radius $\rho$ (i.e., standard SAM) compared to SAM with decreasing perturbation radii $\rho_t$. The decrease of $\rho_t$ follows the same piecewise constant schedule as the learning rate $\gamma_t$. We note that in both cases, we use the $\ell_2$-normalized updates as in the original SAM.}
    \label{fig:test_err_over_decreasing_rho}
\end{figure}

\subsection{Experiments with Noisy Labels}
\label{subsec:app_const_rho_sam_ln}
In Fig.~\ref{fig:label_noise_plots_unnorm_sam}, we show experiments with CIFAR-10 and CIFAR-100 with 60\% of noisy labels for SAM with a fixed inner step size $\rho$ that does not include gradient normalization (denoted as \textit{unnormalized SAM}). We did a prior grid search to determine the best fixed $\rho$ for this case which we show in the figure. We can observe that the best test error taken over epochs almost exactly matches that of the standard SAM.
\begin{figure}[htbp]
    \centering
    \begin{subfigure}[t]{.37\textwidth}
        \caption{\hspace{7mm}\textbf{ResNet-18 on CIFAR-10}}
        \vspace{-2mm}
        \includegraphics[width=1.0\columnwidth]{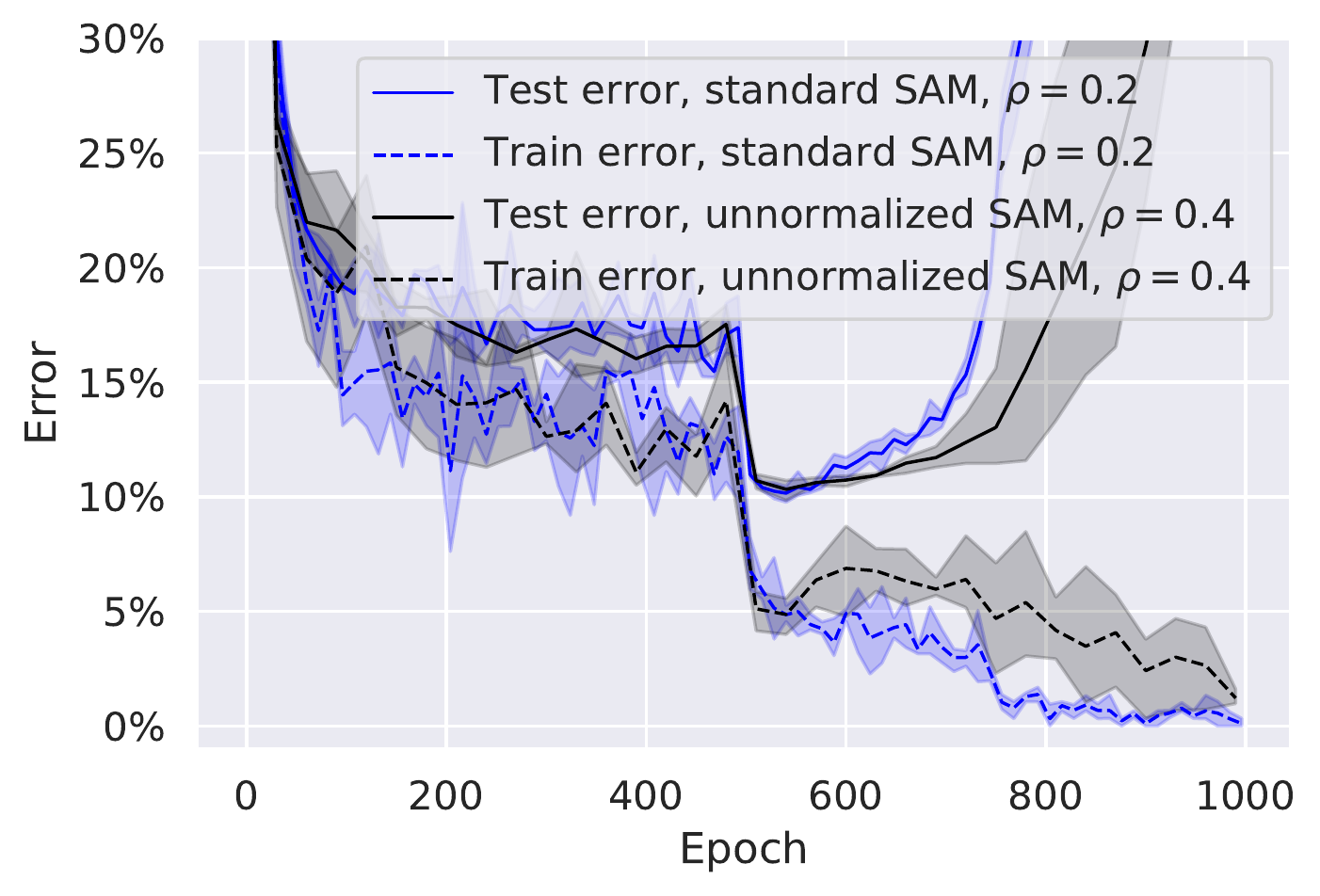}
    \end{subfigure}
    \quad \quad
    \begin{subfigure}[t]{.37\textwidth}
        \caption{\hspace{7mm}\textbf{ResNet-34 on CIFAR-100}}
        \vspace{-2mm}
        \includegraphics[width=1.0\columnwidth]{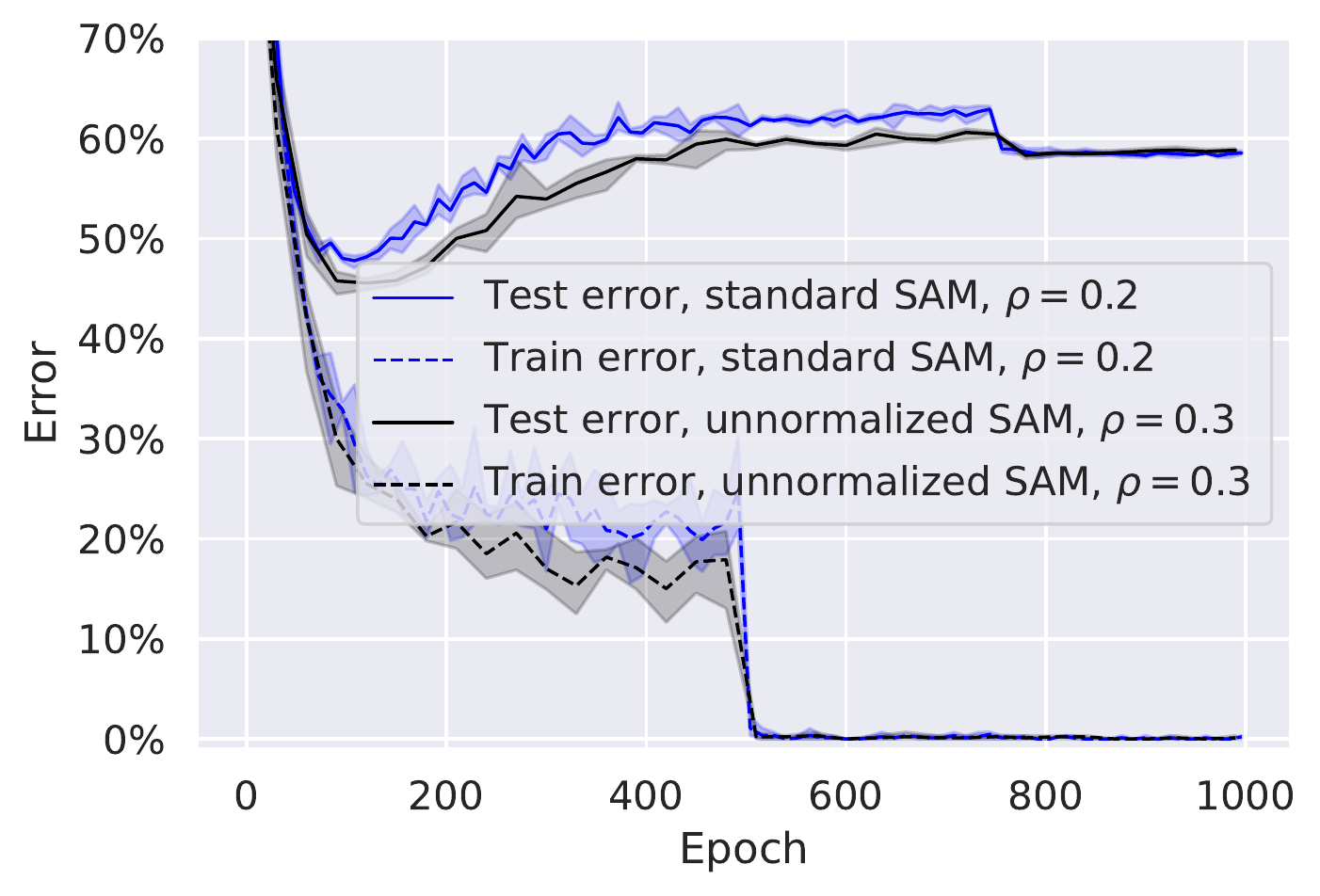}
    \end{subfigure}
    \caption{Plots over training for a ResNet-18 trained on CIFAR-10 with 60\% label noise for SAM with and without gradient normalization.}
    \label{fig:label_noise_plots_unnorm_sam}
\end{figure}

\end{document}